\documentclass{article}

\PassOptionsToPackage{numbers, compress}{natbib}

\usepackage[main, final]{neurips_2026}

\usepackage[utf8]{inputenc} 
\usepackage[T1]{fontenc}    
\usepackage{hyperref}       
\usepackage{url}            
\usepackage{booktabs}       
\usepackage{amsfonts}       
\usepackage{nicefrac}       
\usepackage[english]{babel}
\usepackage{microtype}      
\usepackage{xcolor}         

\usepackage{float}
\usepackage{placeins}

\usepackage{subcaption}
\usepackage[table]{xcolor}
\usepackage{hyperref}
\usepackage{tabularx,array,colortbl}
\usepackage{algorithm}
\usepackage{algorithmic}
\usepackage{multirow}
\usepackage{amsmath}
\usepackage{amssymb}
\usepackage{mathtools}
\usepackage{amsthm}
\usepackage[capitalize,noabbrev]{cleveref}
\usepackage{hyperref}

\theoremstyle{plain}
\newtheorem{theorem}{Theorem}[section]
\newtheorem{proposition}[theorem]{Proposition}
\newtheorem{lemma}[theorem]{Lemma}

\theoremstyle{definition}

\newtheorem{assumption}[theorem]{Assumption}
\theoremstyle{remark}

\usepackage[textsize=tiny]{todonotes}
\newcolumntype{Y}{>{\centering\arraybackslash}X}
\newcommand{\ipF}[2]{\langle #1, #2 \rangle_F}
\newcommand{\normF}[1]{\left\|#1\right\|_F}
\newenvironment{fittable}
{\begin{adjustbox}{max width=\textwidth}}
{\end{adjustbox}}

\usepackage{makecell}
\usepackage{adjustbox}

\title{LOCU: L\"owdin-Orthogonalized Constraint Updates for Multi-Constraint Policy Optimization}

\author{%
  Joonyoung Lim \\
  Department of Information Convergence Engineering \\
  Pusan National University \\
  \texttt{ingyuh1015@pusan.ac.kr}
  \And
  Younghwan Yoo\thanks{Corresponding author.} \\
  School of Computer Science and Engineering \\
  Pusan National University \\
  \texttt{ymomo@pusan.ac.kr}
}

\begin{document}

\maketitle

\begin{abstract}
We investigate multi-constraint policy optimization in constrained Markov decision processes (CMDPs), where interactions among constraint gradients often yield ill-conditioned update directions, causing numerically fragile steps and unintended cancellation or redundant overlap in constraint corrections. To address this, we propose LOCU, which mitigates numerical instability arising from constraint interactions by applying symmetric L\"owdin orthogonalization to the natural gradients of the reward and constraints under the Fisher–Rao metric. Combined with Fisher-geometric screening and near-collinear compression, LOCU yields a compact reduced space in which the trust region becomes a Euclidean ball and constraints are handled symmetrically through a low-dimensional active-set solve. For infeasible iterates, LOCU constructs a Pareto-descent direction which preserves the local feasibility of satisfied constraints within the same reduced formulation. Experiments show that LOCU remains stable under high constraint coupling, narrow feasible regions, and infeasible starts, and generalizes across different cost critic architectures. Code is available at: \href{https://github.com/ingyuh1015/Lowdin-Orthogonalized-Constraint-Updates}{\texttt{GitHub Repository}}.
\end{abstract}

\section{Introduction}
\label{sec:intro}

Reinforcement learning has achieved notable success in sequential decision making, but deployment in safety-critical domains requires policies to additionally satisfy explicit operational constraints rather than maximize reward alone \cite{schulman2015trust,schulman2017proximal,achiam2017constrained,altman2021constrained}. Constrained Markov decision processes (CMDPs) formalize this requirement by optimizing return subject to cost limits. In practical scenario, agents could face several constraints simultaneously, such as velocity, collision, or resource limit in autonomous driving or robotic control. Existing safe policy-optimization methods commonly proposed and evaluated in single-constraint regimes, and their behavior can degrade when multiple constraints interact.

The core difficulty with multi-constraint is not merely computational overhead, but the \emph{interaction geometry} among constraints~\cite{nocedal2006numerical}. Existing approaches expose one or more of following failures (i) unstable updates under near-dependent constraint gradients or severe scale disparity, (ii) weakened effective updates under constraint interaction, caused by cancellation, dominance by a single direction, or projection-order dependence, and (iii) inefficient feasibility recovery after constraint violations, as noted in prior work~\cite{kim2023trust}.The problem worsens under many or near-parallel constraints, making stability depend on hyperparameter tuning rather than geometry-consistent structure.

A geometric perspective clarifies why these issues arise and how to address them.
Policy optimization can be analyzed on a Riemannian manifold whose metric arises as the second-order local approximation of the KL divergence, known as the Fisher–Rao metric~\cite{ay2017information,amari2016information}. Reward and constraint objectives define tangent directions in natural-gradient coordinates~\cite{kakade2001natural}, and each linearized constraint induces a halfspace normal to its constraint gradient. With multiple constraints, feasible updates lie in the intersection of these halfspaces and the local Fisher-metric ellipsoid~\cite{achiam2017constrained,yang2020projection}. Redundant or ill-conditioned constraint directions under this metric make both first- and second-order updates fragile and conservative, an effect amplified by finite-sample noise or critic error.


Building on this geometric insight, we introduce the L\"owdin Orthogonalized Constraint Update (LOCU), a second-order framework for multi-constraint RL that mitigates constraint coupling in Fisher-consistent coordinates. LOCU constructs a Fisher-orthonormal basis from the natural gradients of the reward and constraints via symmetric L\"owdin orthogonalization~\cite{lowdin1950non}. This transformation maps the Fisher-metric ellipsoid into a spherical trust region in a reduced Euclidean space, making constraint interactions explicit and numerically stable. For large constraint sets, LOCU employs Fisher-geometric screening to prune non-binding or redundant constraints. Furthermore, it handles infeasibility with a geometry-aligned recovery step that seeks to reduce violations while preserving satisfied constraints in linearized local model. The update operates purely at the policy-optimization layer, ensuring compatibility with various critic architectures.

We conduct experiments with multi-constrained tasks of ~\cite{Gronauer2022BulletSafetyGym,kim2024spectral} applying the constraint configurations introduced by~\cite{yao2024gradient}.
Evaluations under narrow feasible regions and strong constraint coupling, LOCU satisfies all constraints while limiting return drops to around 21\% in both settings, compared to 26–49\% and 13–32\% drops with frequent violations in baselines, performing consistently from infeasible starts and generalizing across different cost critic architectures. These results demonstrate the effectiveness of the proposed method in maintaining stable policy optimization by symmetrically handling constraint interactions under the Fisher–Rao geometry.

\section{Related Work}
\label{sec:related}

Existing constrained RL methods differ in how they incorporate multiple constraints into policy optimization and maintain feasibility during actor updates. These design choices become tightly coupled with the representation of constraint interactions. When constraint interactions are not handled in a geometry-consistent manner under the Fisher--Rao metric, stable learning often relies heavily on environment-specific tuning of numerous algorithmic choices and hyperparameters.

\paragraph{Primal-dual Methods.}
Lagrangian-based methods adopt a primal-dual formulation where constraints are managed via Lagrange multipliers co-optimized with policy parameters~\cite{stooke2020responsive}. However, with multi-constraint, these methods aggregate costs from multi constraint into a scalar, obscuring the directional structure of individual constraint gradients. Depending on alignment, constraint gradients may cancel and suppress effective update, or reinforce and destabilize multiplier dynamics, causing oscillations and training instability. As the number of constraints increases, performance becomes increasingly sensitive to multiplier initialization and tuning can be ``as hard as solving the RL problem itself''~\cite{paternain2019constrained}.

Recent work has attempted to address gradient-level interference in multi-constraint settings by analyzing directional interactions and applying shaping rules to improve update consistency~\cite{yao2024gradient}. While this mitigates some interference issues within the Lagrangian framework, its reliance on Euclidean angle-based similarity may not faithfully reflect the true correlations under the Fisher--Rao geometry~\cite{kakade2001natural}. Furthermore, it does not alleviate the burden of tuning an expanding set of hyperparameters as constraints grow.

A separate line of works handle dual variables by solving a quadratic program (QP) or projection defined by linearized constraints within a KL-based trust region~\cite{achiam2017constrained}. In multi-constraint regimes, feasibility recovery is either not explicitly addressed once iterates leave the local feasible region ~\cite{achiam2017constrained}, or handled via repeated constraint-wise projections that shrink the effective reward step and introduce imbalances as the number of constraints grows~\cite{yang2020projection}. Moreover, the active set is unknown a priori, these updates rely on generic QP solvers. Under linear dependence or near-collinearity among constraint gradients, the resulting QP can be ill-conditioned, leading to solver failures or high-variance solutions where it is difficult to distinguish intrinsic constraint degeneracy from solver-induced numerical errors.

Recent works in risk-constraint RL proposed gradient integration mechanisms to manage infeasibility start and recovery in multi-constraint settings and further extend these ideas toward global convergence guarantees~\cite{kim2023trust,kim2024spectral}. Despite these advances, general applicability across constrained RL is largely limited to gradient integration, while the absence of a geometry-consistent formulation and continued dependence on generic solvers hinder scalability and stability as the number of constraints and interaction grows.

\paragraph{Primal Methods.}
Primal methods can be naturally extended to multi-constraint settings by switching between reward optimization and constraint minimization~\cite{xu2021crpo}. However, as the number of constraints increases, handling one constraint at a time leads to slower convergence and inconsistent updates. Furthermore, substantial correlations between adaptive weights and stochastic gradients are known to bias the update step and cause convergence failure~\cite{zhou2022convergence}. To address such instabilities, recent work proposes a soft-switching strategy~\cite{gu2024balance}, which reduces oscillations and accelerates convergence by removing conflicting components between reward and constraint gradients. While it improves stability in single-constraint scenarios, effectiveness in multi-constraint settings remains unverified.

For a geometric discussion of multi-constraint update pathologies in representative primal--dual and projection-based methods, see Appendix~\ref{sec:appendix:projection-cases}.

\section{Preliminaries}
\label{sec:prelim}

\subsection{Constrained Markov Decision Processes}
\label{sec:cmdp}

A CMDP is defined by
$(\mathcal{S},\mathcal{A},P,r,\{c_k\}_{k=1}^K,\rho,\gamma)$, where
$\mathcal{S}$ and $\mathcal{A}$ are the state and action spaces,
$P(\cdot\mid s,a)$ is the transition probability, $r:\mathcal{S}\times\mathcal{A}\to\mathbb{R}$
is the reward function, and $c_k:\mathcal{S}\times\mathcal{A}\to\mathbb{R}_{\ge 0}$ are
$K$ cost functions that quantify constraint violations~\cite{altman2021constrained}.
The initial state is $s_0\sim\rho$ and the discount factor is $\gamma\in(0,1)$.
A stochastic policy $\pi(a\mid s)$ induces a trajectory distribution
$\mathbb{P}_{\pi}$ over $(s_0,a_0,s_1,\dots)$. Given a policy $\pi$, its discounted state--action distribution is
\begin{equation}
  d^\pi(s,a)
  :=(1-\gamma)\sum_{t=0}^\infty \gamma^t\,\mathbb{P}_\pi(s_t=s,a_t=a).
\end{equation}

Using this distribution, define the discounted return
$J_R(\pi):=\mathbb{E}_{s_0\sim\rho}[\sum_{t=0}^\infty \gamma^t r(s_t,a_t)]$,
and for each cost $k$ define
$J_{C_k}(\pi):=\mathbb{E}_{s_0\sim\rho}[\sum_{t=0}^\infty \gamma^t c_k(s_t,a_t)]$.
The CMDP objective is
\begin{equation}
  \label{eq:cmdp-objective}
  \max_{\pi}\; J_R(\pi)
  \quad \text{s.t.}\quad
  J_{C_k}(\pi)\le b_k,
  \forall k.
\end{equation}

\begin{assumption}[Slater condition]
  \label{assump:slater-cmdp}
  There exists a policy $\bar{\pi}$ in the considered policy class such that
  \begin{equation}
    \label{eq:slater-cmdp}
    J_{C_k}(\bar{\pi}) < b_k,
    \qquad \forall k\in\{1,\dots,K\}.
  \end{equation}
\end{assumption}
This feasibility condition is standard in constrained policy optimization~\cite{achiam2017constrained}.

\subsection{Policy Gradients for Reward and Constraints}
\label{sec:policy-gradients}

Consider a differentiable policy $\pi_\theta$ parameterized by $\theta\in\mathbb{R}^n$.
Let $A_R^\pi(s,a):=Q_R^\pi(s,a)-V_R^\pi(s)$. The policy gradient for the reward objective is
\begin{equation}
  \label{eq:reward-grad}
  \nabla_\theta J_R(\pi_\theta)
  = \tfrac{1}{1-\gamma}\,
  \mathbb{E}_{(s,a)\sim d^{\pi_\theta}}
  \!\left[
    \nabla_\theta \log \pi_\theta(a\mid s)\,
    A_R^{\pi_\theta}(s,a)
  \right].
\end{equation}

Similarly, each cost objective admits a policy gradient obtained by replacing $r$ with $c_k$.
We denote
$g_r:=\nabla_\theta J_R(\pi_\theta)$ and $g_{c_k}:=\nabla_\theta J_{C_k}(\pi_\theta)$.

\subsection{Trust-Region Geometry and Natural Gradient}
\label{sec:fisher-geometry}

Trust-region updates are measured by the average KL divergence under the discounted state visitation distribution,
\begin{equation}
  \label{eq:avg-kl}
  \bar D_{\mathrm{KL}}(\pi_\theta,\pi_{\theta'})
  :=
  \mathbb{E}_{s\sim d^{\pi_\theta}}
  \!\left[
    \mathrm{KL}\!\left(\pi_\theta(\cdot\mid s)\,\|\,\pi_{\theta'}(\cdot\mid s)\right)
  \right].
\end{equation}
For a small displacement $\Delta\theta:=\theta'-\theta$, a second-order Taylor expansion gives
\begin{equation}
  \label{eq:kl-fisher}
  \bar D_{\mathrm{KL}}(\pi_\theta,\pi_{\theta+\Delta\theta})
  =
  \frac{1}{2}\Delta\theta^\top F\,\Delta\theta
  + o(\|\Delta\theta\|^2),
\end{equation}
where $F$ is the Fisher information matrix (FIM),
\begin{equation}
  \label{eq:fisher}
  F
  :=
  \mathbb{E}_{(s,a)\sim d^{\pi_\theta}}
  \!\left[
    \nabla_\theta \log \pi_\theta(a\mid s)\,
    \nabla_\theta \log \pi_\theta(a\mid s)^\top
  \right].
\end{equation}
Equation~\eqref{eq:kl-fisher} induces the Fisher--Rao geometry, where the FIM $F$ defines the inner product
$\langle u,v\rangle_F:=u^\top F v$ and the corresponding norm $\normF{u}:=\sqrt{u^\top F u}$.
Since the FIM $F$ is always positive semi-definite, we assume it to be positive definite ($F\succ 0$) in what follows
so that $F^{-1}$ and the Fisher inner product are well-defined (in practice, this can be enforced by diagonal damping).

Under this geometry, the natural gradient is defined as the steepest-ascent direction with respect to the Fisher--Rao metric:
\begin{equation}
  \label{eq:natural-grad}
  \tilde{g} := F^{-1} g.
\end{equation}
We denote natural gradients for the reward and costs as $\tilde{g}_r:=F^{-1}g_r$ and $\tilde{g}_{c_k}:=F^{-1}g_{c_k}$.

\subsection{Local Multi-Constraint Trust-Region Subproblem}
\label{sec:local-subproblem}
Under a KL-based trust region, Eq.~\eqref{eq:cmdp-objective} is locally approximated by first-order expansions
of $J_R$ and $J_{C_k}$ at $\theta$ together with the Fisher quadratic approximation \eqref{eq:kl-fisher}.
Define the current constraint slacks as $\Delta_k := b_k - J_{C_k}(\pi_\theta)$.
The local trust-region CMDP step is
\begin{equation}
  \label{eq:local-trust-region-cmdp}
  \begin{aligned}
    \max_{\Delta\theta\in\mathbb{R}^n}\quad
    & g_r^\top \Delta\theta\\
    \text{s.t.}\quad
    & g_{c_k}^\top \Delta\theta \le \Delta_k,\quad k=1,\dots,K,\\
    & \tfrac12 \normF{\Delta\theta}^{2} \le \delta.\\
  \end{aligned}
\end{equation}

Problem~\eqref{eq:local-trust-region-cmdp} is a convex program.
When the current iterate is strictly feasible, i.e., $\Delta_k>0$ for all $k$ and $\delta>0$,
the point $\Delta\theta=0$ lies in the interior of the feasible set.
In this case, strong duality holds and the optimum is characterized by the Karush--Kuhn--Tucker (KKT) conditions
and the associated dual variables~\cite{boyd2004convex}.

\begin{lemma}
  \label{lem:solution-subspace}
  Problem \eqref{eq:local-trust-region-cmdp} admits an optimal solution
  $\Delta\theta^\star$ such that
  $\Delta\theta^\star \in \mathrm{span}\{\tilde g_r,\tilde g_{c_1},\dots,\tilde g_{c_K}\}$.
\end{lemma}

\begin{proof}\noindent
  See Appendix~\ref{sec:appendix:proof-solution-subspace}.
\end{proof}

\subsection{Symmetric L\"owdin orthogonalization}
\label{sec:prelim:lowdin}

When working with multiple update directions, it is often useful to replace a correlated basis by an orthonormal one in a way that is symmetric and insensitive to the ordering of directions. A classical construction is the symmetric
orthogonalization of L\"owdin~\cite{lowdin1950non}.

Let $S=[s_1,\ldots,s_p]\in\mathbb{R}^{n\times p}$ be a matrix whose columns are linearly independent directions.
Define the Gram matrix \(G:=S^\top S\in\mathbb{R}^{p\times p}\).
Since $\mathrm{rank}(S)=p$, we have $G\succ 0$ and the principal inverse square root $G^{-1/2}$ is well-defined. The symmetric L\"owdin orthogonalization constructs
\begin{equation}
  E:=S\,G^{-1/2}\in\mathbb{R}^{n\times p},
  \qquad
  E^\top E
  =
  G^{-1/2}S^\top S\,G^{-1/2}
  =
  I.
  \label{eq:prelim:lowdin-basis}
\end{equation}

so the columns of $E$ are orthonormal and $\mathrm{span}(E)=\mathrm{span}(S)$.

Unlike Gram--Schmidt, the L\"owdin construction is symmetric: it depends only on $G=S^\top S$ and does not privilege any
particular column order. This order-invariant structure is the property used by LOCU when constructing reduced
coordinates from reward and constraint directions.
\section{Proposed Method}
\label{sec:method}
We focus on the local multi-constraint trust-region subproblem in \eqref{eq:local-trust-region-cmdp}.
By Lemma~\ref{lem:solution-subspace}, it suffices to search over the subspace spanned by the reward and constraint natural-gradient directions.
We construct Fisher-orthonormal reduced coordinates for this subspace and solve an equivalent low-dimensional convex program; the remainder of this section provides the coordinate construction, the scalable handling of many constraints, and the feasibility-recovery variant. Practical implementation details are provided in Appendix~\ref{sec:appendix:practical_impl}.

\subsection{Fisher-orthonormal reduced coordinates via L\"owdin orthogonalization}
\label{sec:method:lowdin}

Motivated by the symmetric L\"owdin orthogonalization introduced in Section~\ref{sec:prelim:lowdin},
we apply the same construction to the reward and constraint natural-gradient directions under the Fisher--Rao metric.

Let $\tilde g_r := F^{-1}g_r$ and $\tilde g_{c_k}:=F^{-1}g_{c_k}$ denote the reward and cost natural gradients.
Fix any set of $m$ representative constraint indices $\{i_1,\ldots,i_m\}\subseteq\{1,\ldots,K\}$,
and form the direction matrix
\begin{equation}
  \label{eq:method:S}
  S :=
  \begin{bmatrix}\tilde g_r & \tilde g_{c_{i_1}} & \cdots & \tilde g_{c_{i_m}}
  \end{bmatrix}
  \in \mathbb{R}^{n\times(m+1)}.
\end{equation}
Define the Gram matrix under the Fisher inner product, which we call the \emph{Fisher Gram matrix} for convenience:
\begin{equation}
  \label{eq:method:O}
  O := S^\top F S \in \mathbb{R}^{(m+1)\times(m+1)} .
\end{equation}
We construct the symmetric L\"owdin basis
\begin{equation}
  \label{eq:method:E}
  E := S O^{-1/2} \in \mathbb{R}^{n\times(m+1)}.
\end{equation}

\begin{lemma}
  \label{lem:lowdin-orth}
  If $O\succ 0$ (or if $O$ is regularized to be positive definite), then
  \begin{equation}
    \label{eq:method:EFEsI}
    E^\top F E = I.
  \end{equation}
\end{lemma}
\noindent Proof is in Appendix~\ref{sec:appendix:lowdin-proof}.

Lemma~\ref{lem:lowdin-orth} implies that the columns of $E$ form an orthonormal basis under the Fisher inner product. We restrict the update to $\mathrm{span}(S)$ and write $\Delta\theta = Ew$ with $w\in\mathbb{R}^{m+1}$. Let $R:=\sqrt{2\delta}$. Since $E^\top F E=I$ by \eqref{eq:method:EFEsI}, the quadratic KL trust region reduces to
\begin{equation}
  \label{eq:method:trustball}
  \tfrac12 \Delta\theta^\top F \Delta\theta
  =
  \tfrac12 w^\top w
  =
  \tfrac12 \|w\|_2^2,
\end{equation}
and hence becomes the Euclidean ball $\|w\|_2\le R$ in the reduced coordinates.

Define the reduced constraint normals and the reduced-space objective coefficients in the L\"owdin basis by
\begin{equation}
  \label{eq:method:reduced-u}
  u := E^\top g_r,
  \qquad
  C_k^\top := E^\top g_{c_k}\in\mathbb{R}^{m+1}.
\end{equation}
Here, we use the symbol $u$ as the objective coefficient vector of the reduced program.
In the feasible phase, $u=E^\top g_r$ as above; in the feasibility-recovery phase
(Section~\ref{sec:method:recovery}), we simply replace $u$ by the recovery direction constructed directly in the same
reduced space.

Let $\mathcal{K}\subseteq\{1,\dots,K\}$ be the set of constraints enforced in the current update,
and collect the corresponding rows into
$C\in\mathbb{R}^{|\mathcal{K}|\times(m+1)}$ and $\Delta\in\mathbb{R}^{|\mathcal{K}|}$ with
$\Delta_k=b_k-J_{C_k}(\pi_\theta)$.
Then the reduced program is
\begin{equation}
  \label{eq:method:reduced-qp}
  \max_{w\in\mathbb{R}^{m+1}}\; u^\top w
  \quad \text{s.t.}\quad
  \|w\|_2 \le R,\;\; Cw \le \Delta.
\end{equation}

\begin{proposition}
  \label{prop:coeff-gram}
  Let $e_j$ be the $j$th canonical basis vector in $\mathbb{R}^{m+1}$, and let the columns of $S$ be ordered as in
  \eqref{eq:method:S}. Then
  \begin{equation}
    \label{eq:method:coeff-gram}
    E^\top g_r = O^{1/2} e_1,
    \quad
    E^\top g_{c_{i_j}} = O^{1/2} e_{j+1},
    \quad j=1,\dots,m.
  \end{equation}
\end{proposition}
\noindent Proof is in Appendix~\ref{sec:appendix:proof-coeff-gram}.

\begin{theorem}
  \label{thm:reduced-equivalence}
  Assume $O\succ 0$ (or $O$ is regularized so that $O\succ 0$).
  Fix any enforced constraint index set $\mathcal{K}$ and any representative set $\{i_1,\ldots,i_m\}$ used to form
  $S,O,E$ as in \eqref{eq:method:S}--\eqref{eq:method:E}.
  Then the change of variables $\Delta\theta=Ew$ gives a one-to-one correspondence between
  $\mathrm{span}(S)$ and $\mathbb{R}^{m+1}$ and transforms the trust-region CMDP subproblem
  \eqref{eq:local-trust-region-cmdp} restricted to $\Delta\theta\in\mathrm{span}(S)$ and constraints $k\in\mathcal{K}$
  into the reduced program \eqref{eq:method:reduced-qp}.
  In particular, any optimizer $w^\star$ of \eqref{eq:method:reduced-qp} yields a valid update
  $\Delta\theta^\star=Ew^\star$ in the original parameter space.
\end{theorem}
\noindent Proof is in Appendix~\ref{sec:appendix:proof-reduced-equivalence}.

Unlike Gram–Schmidt orthogonalization, which propagates and accumulates perturbations through a sequential process, Löwdin orthogonalization in Proposition~\ref{prop:coeff-gram} distributes perturbations symmetrically across the eigenspace. We analyze these perturbation structures in Appendix~\ref{sec:appendix:lowdin-gs-noise}.

\subsection{Constraint screening and basis compression}
\label{sec:method:screening}
However, directly working with all constraint directions can be expensive at scale and can be numerically unstable when constraint directions are strongly aligned in Fisher--Rao geometry. Therefore, LOCU screens locally non-binding/degenerate constraints and compresses near-collinear directions for basis construction, while still enforcing every retained constraint in the reduced optimization.

With Eqs.~\eqref{eq:method:S} and~\eqref{eq:method:O}, we start from full direction with reward and K constraints $S_{\mathrm{full}}\in\mathbb{R}^{n\times(K+1)}$ and its Fisher Gram matrix
$O_{\mathrm{full}}\in\mathbb{R}^{(K+1)\times(K+1)}$. For each constraint $k$, define the Fisher norm and normalized slack
\begin{equation}
  \label{eq:method:screen-metrics}
  \nu_k := \sqrt{(O_{\mathrm{full}})_{k+1,k+1}},
  \qquad
  \mu_k := \Delta_k/\nu_k,
\end{equation}
and for each pair $(i,j)$ define the Fisher cosine similarity among constraint natural gradients
\begin{equation}
  \label{eq:method:screen-cos}
  \varrho_{ij}
  :=
  \frac{(O_{\mathrm{full}})_{i+1,j+1}}{\nu_i\nu_j}.
\end{equation}
Equivalently, with Fisher-unit directions $u_k:=\tilde g_{c_k}/\nu_k$ (so that $\|u_k\|_F=1$), we have
$\varrho_{ij}=\langle u_i,u_j\rangle_F$ and the $k$th linearized constraint is
$\langle u_k,\Delta\theta\rangle_F \le \mu_k$.

\subsubsection{Constraint screening}
We perform worst-case Fisher screening and exclude non-binding and near-degenerate constraints.

\begin{lemma}
  \label{lem:bound-gram}
  For any $\Delta\theta$ satisfying $\|\Delta\theta\|_F \le R$,
  \begin{equation}
    \label{eq:method:bound-gram}
    g_{c_k}^\top \Delta\theta \le \nu_k R.
  \end{equation}
\end{lemma}
\noindent Proof is in Appendix~\ref{sec:appendix:proof-bound-gram}.

\begin{proposition}
  \label{prop:nonbinding-gram}
  If $\Delta_k \ge \nu_k R$, then constraint $k$ cannot be active inside the trust region
  $\tfrac12\Delta\theta^\top F\Delta\theta\le\delta$ and can be excluded from $\mathcal{K}$ without changing feasibility
  within that trust region.
\end{proposition}
\noindent Proof is in Appendix~\ref{sec:appendix:proof-nonbinding-gram}.

In addition, constraints with $\nu_k \le \varepsilon_{\mathrm{nz}}$ are excluded from $\mathcal{K}$ to avoid unstable normalization in \eqref{eq:method:screen-metrics}, since $\tilde g_{c_k}$ is nearly degenerate in Fisher--Rao geometry.

\subsubsection{Basis clustering}

We compress near-collinear constraint directions only to stabilize basis construction; all constraints in $\mathcal{K}$ remain enforced in the reduced program.

We first build an undirected graph on $\mathcal{K}$ by connecting each pair $(i,j)$ that satisfies $|\varrho_{ij}| \ge 1-\varepsilon_{\mathrm{dup}}$.
Let $\{\mathcal{G}_1,\ldots,\mathcal{G}_m\}$ be the connected components of this graph.
Then, for each group $\mathcal{G}\ell$, we select a representative index by $i\ell \in \arg\min_{k\in\mathcal{G}_\ell} \mu_k$.
If all $\mu_k$ in the group are nonnegative, this rule selects the tightest constraint in Fisher-normalized units.
If some $\mu_k$ are negative, it selects the most violated constraint in the same units.
The reduced basis matrix $S$ in \eqref{eq:method:S} is then formed using $\{i_1,\ldots,i_m\}$.

\begin{lemma}
  \label{lem:near-collinear}
  Let $u_i,u_j$ be Fisher-unit vectors with $\langle u_i,u_j\rangle_F=\rho\in[-1,1]$.
  Then for any $\Delta\theta$ with $\|\Delta\theta\|_F\le R$,
  \begin{equation}
    \label{eq:method:near-collinear-bound}
    \min_{s\in\{+1,-1\}}
    \big|\langle u_i,\Delta\theta\rangle_F - s\,\langle u_j,\Delta\theta\rangle_F\big|
    \le
    \sqrt{2(1-|\rho|)}\,R.
  \end{equation}
  In particular, if $|\rho|\ge 1-\varepsilon_{\mathrm{dup}}$, the discrepancy is at most $\sqrt{2\varepsilon_{\mathrm{dup}}}\,R$.
  One may take $s=+1$ when $\rho\ge 0$ and $s=-1$ when $\rho<0$.
\end{lemma}

\noindent Proof is in Appendix~\ref{sec:appendix:proof-near-collinear}.

Lemma~\ref{lem:near-collinear} formalizes why compressing near-collinear directions in the basis construction has negligible local effect:
within the trust region, constraints in the same group respond similarly in Fisher-normalized units up to a sign.
Importantly, all constraints in $\mathcal{K}$ remain explicitly enforced in \eqref{eq:method:reduced-qp} via the rows of $C$.

\subsection{Solving the reduced program}
\label{sec:method:solve}

We solve \eqref{eq:method:reduced-qp} in the reduced coordinates.
In the feasible case (all slacks in $\Delta$ are nonnegative), we set $u=E^\top g_r$.
If the halfspaces are inactive, the maximizer over the ball has the closed form
\begin{equation}
  \label{eq:method:w0}
  w_0 = R \frac{u}{\|u\|_2},
\end{equation}
whenever $\|u\|_2\neq 0$. If $Cw_0\le \Delta$, we accept $w^\star=w_0$. Otherwise, we use KKT conditions.
There exist $\lambda\in\mathbb{R}^{|\mathcal{K}|}$ and $\beta\in\mathbb{R}$ such that
\begin{equation}
  \label{eq:method:kkt-stationarity}
  -u + C^\top \lambda + 2\beta w = 0,
  \qquad
  \lambda \ge 0,
  \qquad
  \beta \ge 0,
\end{equation}
together with primal feasibility, complementary slackness, and $\|w\|_2\le R$.
For an active set $\mathcal{J}\subseteq\{1,\dots,|\mathcal{K}|\}$ with $\beta>0$, stationarity yields
\begin{equation}
  \label{eq:method:w-beta}
  w = \frac{1}{2\beta}\bigl(u - C_{\mathcal{J}}^\top \lambda_{\mathcal{J}}\bigr),
\end{equation}
and enforcing $C_{\mathcal{J}}w=\Delta_{\mathcal{J}}$ gives the reduced linear system
\begin{equation}
  \label{eq:method:lambda-beta}
  C_{\mathcal{J}} C_{\mathcal{J}}^\top \lambda_{\mathcal{J}}
  =
  C_{\mathcal{J}} u - 2\beta \Delta_{\mathcal{J}}.
\end{equation}
All matrices in \eqref{eq:method:w-beta}--\eqref{eq:method:lambda-beta} have dimension at most $(m+1)$ and are
independent of the parameter dimension $n$.
The active-set routine and the one-dimensional search for $\beta$ are given in Appendix~\ref{sec:appendix:kkt}.

\subsection{Feasibility recovery via Pareto descent in the orthonormal space}
\label{sec:method:recovery}

When some constraints are violated ($\Delta_k<0$ for at least one $k\in\mathcal{K}$), we use a recovery objective
that prioritizes feasibility improvement in the linearized model while preventing already satisfied constraints
from deteriorating in the same local step.

Let $\mathcal{V} := \{ i \in \{1,\dots,|\mathcal{K}|\} \mid \Delta_i < 0 \}$ and $\mathcal{H} := \{ i \in \{1,\dots,|\mathcal{K}|\} \mid \Delta_i \ge 0 \}$, with indices defined according to the row ordering of $C$ induced by $\mathcal{K}$. 
We compute the simplex weights by solving the following problem:
\begin{equation}\label{eq:method:pareto-qp}
\min_{\alpha\in\mathbb{R}^{|\mathcal{V}|}}\ \|C_{\mathcal{V}}^\top\alpha\|_2^2\quad\text{s.t.}\quad \alpha\ge 0,\ \mathbf{1}^\top\alpha=1.
\end{equation}
and define the Pareto descent direction
\begin{equation}
  \label{eq:method:pareto-dir}
  \bar g_p := - C_{\mathcal{V}}^\top \alpha.
\end{equation}

\begin{lemma}
  \label{lem:pareto-descent}
  If $\bar g_p \neq 0$ defined by \eqref{eq:method:pareto-qp}--\eqref{eq:method:pareto-dir}, then for every $i\in\mathcal{V}$,
  \begin{equation}
    \label{eq:method:pareto-descent-ineq}
    C_i^\top \bar g_p < 0.
  \end{equation}
\end{lemma}
\noindent Proof is in Appendix~\ref{sec:appendix:proof-pareto-descent}.

We set the recovery objective to $u=\bar g_p$ and solve

\begin{equation}
\label{eq:method:recovery-qp}
\max_{w\in\mathbb{R}^{m+1}}\ \bar g_p^\top w
\quad\text{s.t.}\quad
\|w\|_2 \le R,\quad C_{\mathcal{H}}w \le \Delta_{\mathcal{H}}.
\end{equation}

The point $w=0$ is feasible since $\Delta_{\mathcal{H}}\ge 0$, so \eqref{eq:method:recovery-qp} cannot worsen
already satisfied constraints in the linearized local model.

\paragraph{Computational complexity.}
The proposed method starts by computing natural gradients for the reward and constraints.
Each gradient computation requires one backpropagation over the entire batch.
Computing a natural gradient with conjugate gradient (CG) further requires $T_{\mathrm{cg}}$ iterations per gradient, each consisting of one or two backpropagations depending on the implementation.
Therefore, the computational cost can be written as $\mathcal{O}\big((K+1) \cdot T_{\mathrm{cg}} \cdot \texttt{BackProp}\big)$.

After obtaining the natural gradients, the Fisher Gram matrix is constructed from $K^2$ pairwise inner products between them, resulting in a cost of $\mathcal{O}(K^2 \cdot \texttt{InnerProd})$.
Screening and clustering are performed at a cost of $\mathcal{O}(K + |\mathcal{K}|^2)$, followed by Löwdin orthogonalization of the $m$ representative directions at a cost of $\mathcal{O}(m^2 \cdot \texttt{InnerProd} + m^3)$.
Solving the reduced program with the active-set method requires $\mathcal{O}(T_{\mathrm{active}} \cdot m^3)$ operations, where $T_{\mathrm{active}}$ is the number of active-set iterations, while an additional simplex QP over the violated constraints is solved in the recovery case. Since $m \le |\mathcal{K}| \le K$, the computational cost of LOCU procedure is bounded by $\mathcal{O}(K^2 \cdot \texttt{InnerProd} + T_{\mathrm{active}} \cdot K^3)$, which simplifies to $\mathcal{O}(K^2 \cdot \texttt{InnerProd})$ since the second term is negligible relative to the parameter-dimensional inner products.

Both back-propagation and inner product scale linearly with the size of policy parameters $n$.
However, back-propagation involves forward and backward passes over the entire batch, whereas inner product requires only element-wise multiplication and summation. Consequently, the overall computational cost is dominated by natural-gradient computation, which is also required by other second-order multi-constraint methods, leaving LOCU with only marginal additional overhead. Further discussion related to computational cost is provided in Appendix~\ref{app:computational cost}.

\section{Evaluation and Results}\label{sec:experiments}
\subsection{Tasks and Baselines}
\paragraph{Tasks} We evaluated on two settings.
First, we use Circle tasks from Bullet-Safety-Gym~\cite{Gronauer2022BulletSafetyGym} with Ball, Car, and Drone agents, denoted as BC, CC, and DC, under the multi-constraint setting of~\cite{yao2024gradient}.
The task is to follow a circular trajectory, where costs are incurred for leaving the boundary region, violating upper and lower velocity limits.
We vary constraint thresholds and extend with redundancy variants~\cite{yao2024gradient} to control feasible-set geometry and induce near-collinear constraint directions.

Second, we evaluate legged locomotion tasks introduced by~\cite{kim2024spectral} on the bipedal and quadrupedal robots Cassie (Agility Robotics) and Laikago (Unitree), respectively. The objective is to track a commanded velocity while enforcing three constraints on body balance, height, and foot contact timing. To further test heterogeneous constraint handling, we add a goal-navigation variant that retains the three locomotion-stability constraints and introduces three qualitatively different constraints on hazard contact, speed, and energy which was originally on the reward side. This variant requires the agent to reach a goal while avoiding hazards and satisfying operational limits. Both versions are initialized from infeasible states with simultaneous constraint violations under high-dimensional observations, making these tasks suitable for assessing feasibility recovery and compatibility with alternative critic types. For locomotion tasks, all methods use the same quantile-based cost critic.

\begin{figure*}[t]
  \centering
  \setlength{\abovecaptionskip}{2pt}
  \setlength{\belowcaptionskip}{0pt}
  \includegraphics[width=0.70\textwidth]{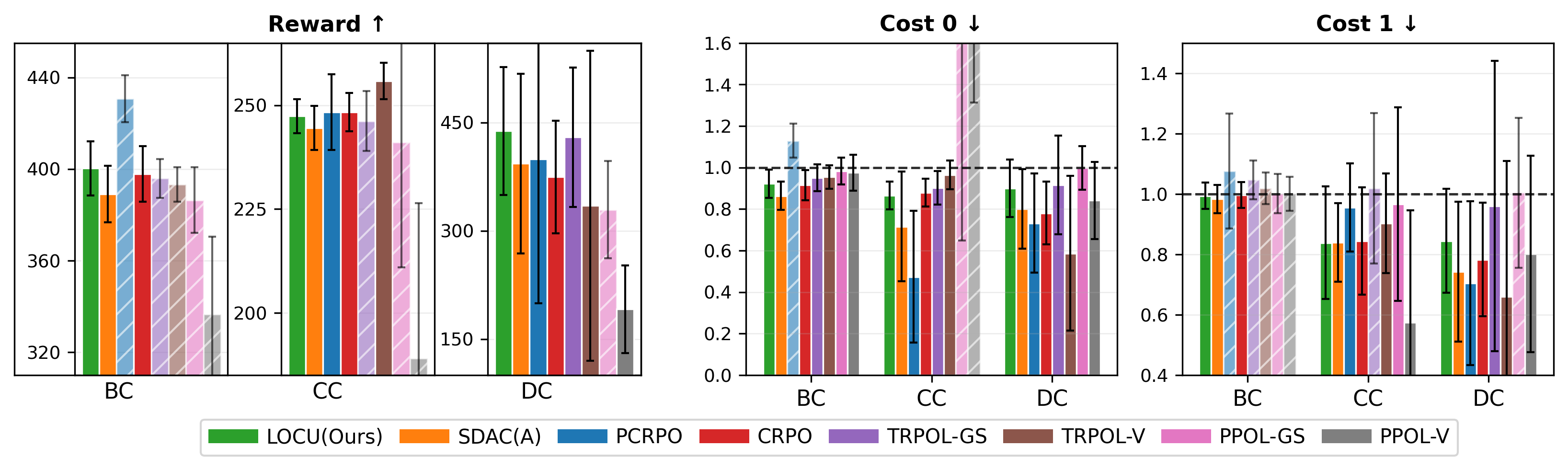}
  \caption{
    Final evaluation on the v2 Circle tasks, with high-velocity and boundary constraints.
  }
  \label{fig:results:v2}
\end{figure*}

\begin{figure*}[t]
  \centering
  \setlength{\abovecaptionskip}{2pt}
  \setlength{\belowcaptionskip}{0pt}
  \includegraphics[width=0.95\textwidth]{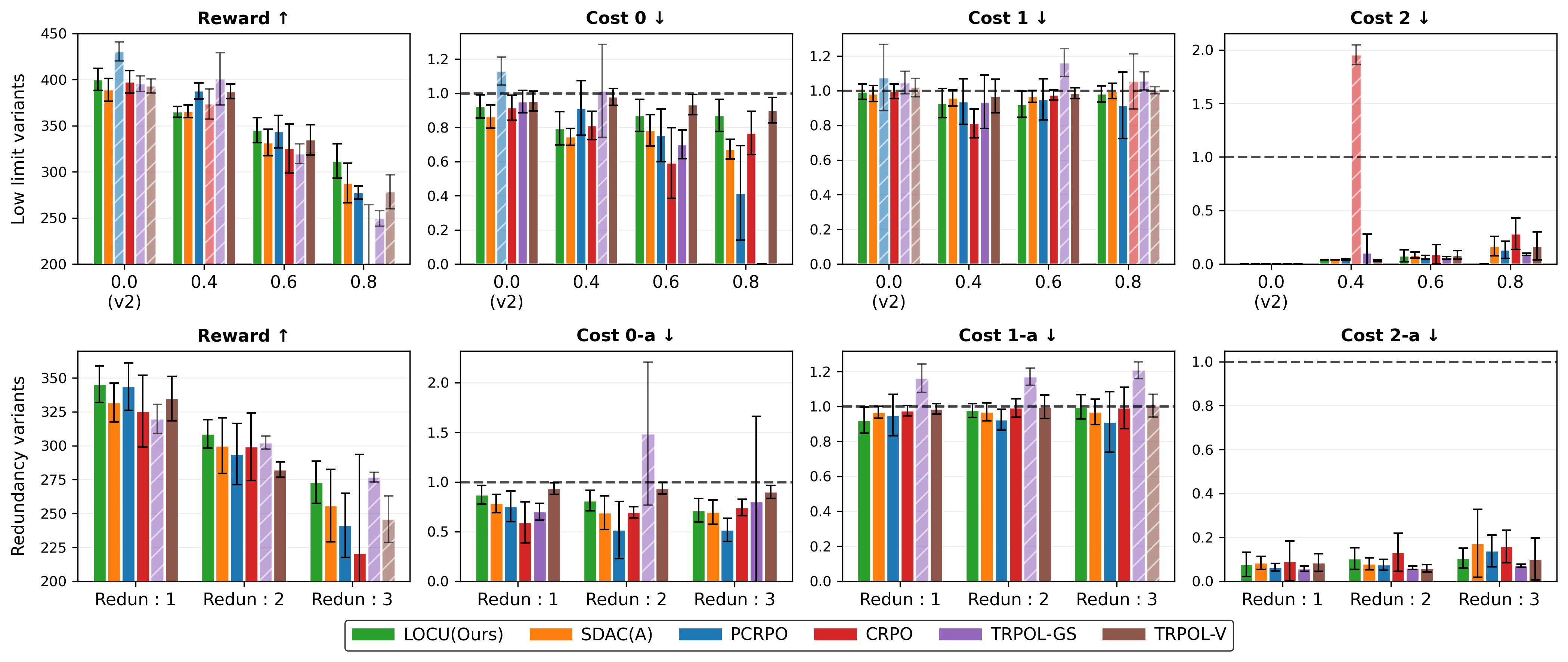}
  \caption{
    Final evaluation results on BC-v3 variants.
    Top: the low-velocity requirement $\kappa_{\mathrm{lo}}\in\{0.0,0.4,0.6,0.8\}$.
    Bottom: redundancy variants with increasing numbers of same-type boundary and velocity constraints.
    The suffix ``--a'' marks the tightest constraint among the same type.
  }
  \label{fig:results:v3}
\end{figure*}

\begin{figure*}[t]
  \centering
  \setlength{\abovecaptionskip}{2pt}
  \setlength{\belowcaptionskip}{0pt}
  \includegraphics[width=0.99\textwidth]{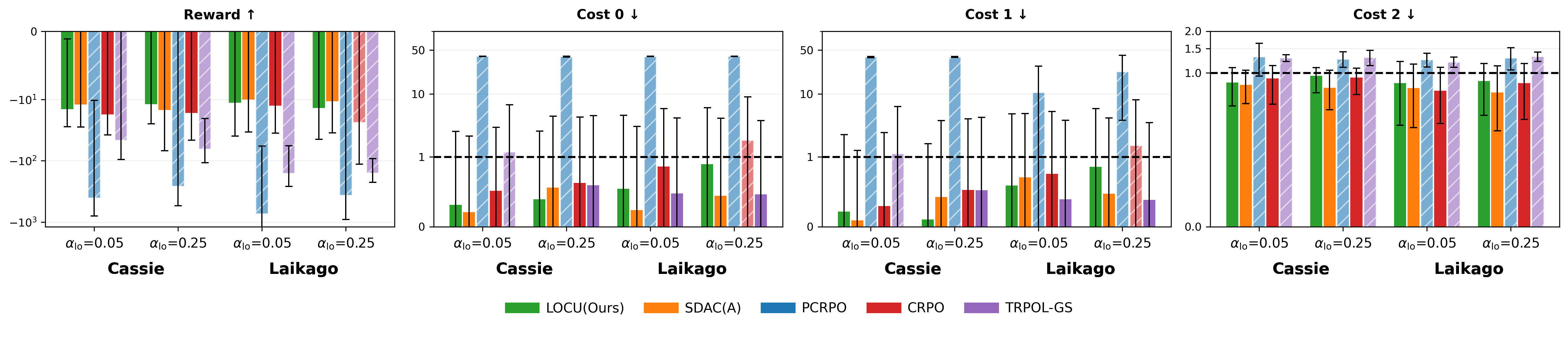}
  \caption{
    Final evaluation on the locomotion tasks Cassie and Laikago with three simultaneous constraints.
    Costs 0, 1, and 2 denote the balance, height, and foot timing constraints. Mean--std constraints use $\alpha_{\mathrm{lo}}\in\{0.05,0.25\}$.
    Cost axes are shown on a log scale.
  }
  \label{fig:results:locomotion-v0}
\end{figure*}

\begin{figure*}[t]
  \centering
  \setlength{\abovecaptionskip}{2pt}
  \setlength{\belowcaptionskip}{0pt}
  \includegraphics[width=0.90\textwidth]{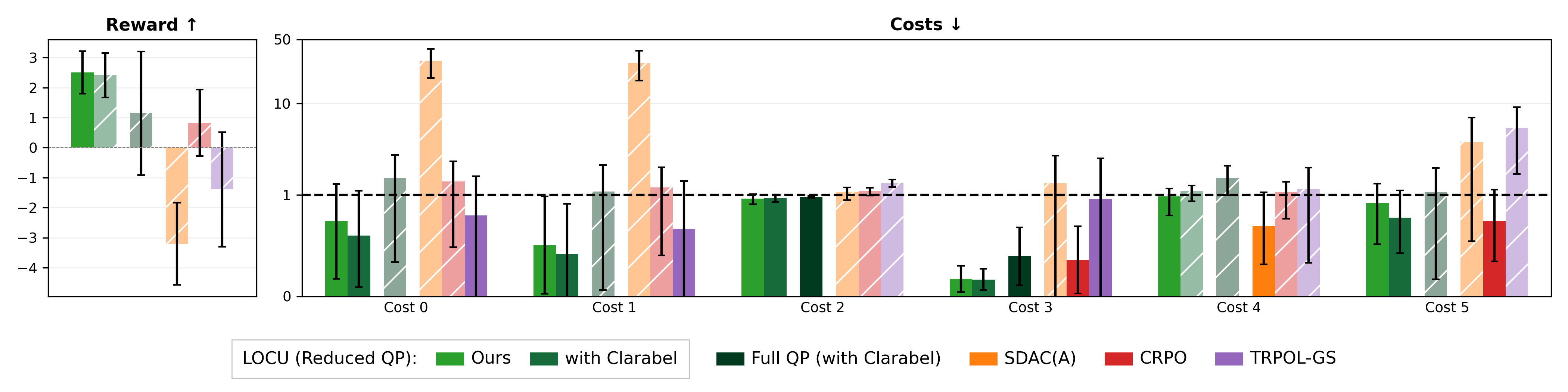}
  \caption{
    Final evaluation on the Cassie goal task with six constraints.
    Costs 0, 1, and 2 denote the balance, height, and foot timing constraints. And costs 3, 4 and 5 denote speed, hazard contact, energy constraints. Cost axes are shown on a log scale and 0.25 is used for $\alpha_{\mathrm{lo}}$.}
  \label{fig:results:cassie-goal-v0}
\end{figure*}

\paragraph{Baselines}
The baselines include the actor update of SDAC~\cite{kim2023trust}, denoted as SDAC(A), which uses a second-order trust-region QP update with gradient integration for recovery, while all other components are aligned with our setup. We also evaluate two primal methods, PCRPO~\cite{gu2024balance} and CRPO~\cite{xu2021crpo}, alongside PPO and TRPO Lagrangian with PID-style multiplier updates~\cite{stooke2020responsive}, each with and without gradient shaping (GS)~\cite{yao2024gradient}, denoted as PPOL-GS, PPOL-V, TRPOL-GS, and TRPOL-V, respectively. For the final task, we further apply Clarabel, a modern interior-point conic solver~\cite{goulart2405clarabel}, to both LOCU's reduced QP and the full-dimensional QP. 

Detailed descriptions for both task and baselines are provided in Appendix~\ref{sec:appendix:experimental-details}.

\subsection{Results}
Results are averaged over five training seeds and 100 evaluation episodes with 10 rollout seeds. For the Circle tasks, we additionally reduce evaluation bias by averaging over the five checkpoints at the end of each training run. All costs are normalized by their corresponding limits; values above 1 indicate constraint violations. In the bar plots, violating cost bars and the corresponding reward bar are rendered with reduced opacity and hatching. Error bars indicate standard deviation. Additional
results of ablation studies and full learning curves are provided in Appendix~\ref{app:additional experiments} and ~\ref{app:learning_curves}.

Fig.~\ref{fig:results:v2} summarizes v2 results on BC, CC, and DC. LOCU consistently satisfies all constraint limits across agents while achieving the highest returns on BC and DC among methods satisfying all constraints, and remaining competitive on CC. The PPOL variants exhibit instability in this multi-constraint setting, where the trained policy fails to reliably track the circular objective. Therefore, we omit PPOL-GS and -V from subsequent experiments. 

Fig.~\ref{fig:results:v3} summarizes the BC-v3 results under feasibility-tightening and redundancy variants. Feasibility tightening includes low-velocity constraints, which shrink the feasible set and reshape early learning by forcing the agent to increase speed before stable circle tracking, thereby delaying reward improvement. This effect is stronger when train starts from infeasible policies: trust-region and primal methods prioritize feasibility recovery for longer as thresholds tighten, while Lagrangian methods may preserve higher return through multiplier adaptation but show less stable constraint satisfaction when aggressively growing multipliers destabilize optimization. The redundancy variants build on the base v3 setting with a single low-velocity coefficient of \(0.6\), simultaneously shrinking the feasible set and inducing near-collinear constraint directions. Across both tightening and redundancy variants, LOCU maintains stable constraint satisfaction while limiting return degradation relative to the baselines.

Figure~\ref{fig:results:locomotion-v0} summarizes the results on Cassie and Laikago with three simultaneous constraints. LOCU and SDAC(A) achieve the best returns while consistently satisfying all constraint limits, whereas PCRPO does not reliably converge in our multi-constraint locomotion setup. Since PCRPO reduces oscillations through soft-switching, we interpret this result as indicating that additional stabilization may be required in coupled multi-constraint regimes with intermittent violations and cross-constraint scale mismatch. 

Additional results on the goal-navigation task, including two variants using Clarabel, are summarized in Figure~\ref{fig:results:cassie-goal-v0}. Here, SDAC(A) misallocates recovery weights as the QP becomes severely ill-conditioned. Unlike the original locomotion task, where the diagonal terms are comparable, body tilt and height constraints violate their limits by roughly 30 times, resulting in much larger diagonal terms than weak-gradient constraints such as energy or obstacle contact. The dual update aims to reduce all active constraints without considering the violation levels, leading to large recovery weights with little correction and negligible weights on the constraints that require correction most. Full QP variants which using the same set of natural gradients exhibit similar behavior, although solving with Clarabel partially improves recovery. Notably, solver flagged only 2 of approximately 24,900 updates across five seeds as inaccurate, indicating that the remaining degradation stems from severe ill-conditioning of the QP itself rather than from frequent solver failures. LOCU avoids this by screening out inactive or degenerate constraints, and constructing an orthogonal basis of the retained gradients. Despite some constraint violations, the Reduced QP variants which use LOCU’s orthogonal basis perform closest to LOCU in both return and constraint satisfaction. This suggests that LOCU’s orthogonal basis construction can be integrated into existing algorithms to enhance numerical stability, while its proposed active-set solver can be replaced with alternative solvers depending on computational and real-time requirements. We provide the additional robustness test on goal-object task with redundancy in Appendix~\ref{sec:appendix:robustness-test}.

\section{Limitations and future works}
LOCU provides local feasibility control but no global convergence guarantee as other second order trust region methods~\cite{achiam2017constrained}. A possible approach is to reinterpret KL-induced Fisher--Rao geometry through a Bregman-proximal mirror-descent perspective over the convex occupancy-measure CMDP, which may provide a route to primal-dual convergence analysis. Also, computing natural-gradient directions and the Fisher Gram matrix for all constraints may raise computational-efficiency concerns. However, as discussed in Appendix~\ref{app:computational cost}, parallelized constraint processing can mitigate this issue by better utilizing GPU throughput. It may also be useful to extend natural gradient calculation with curvature approximations to further reduce computation of CG iters and Fisher-related functions.

\section{Conclusion}
We introduced LOCU, a second-order framework for multi-constraint policy optimization that mitigates constraint coupling in a geometry consistent manner.
By applying symmetric L\"owdin orthogonalization to natural gradients under the Fisher--Rao metric, LOCU yields an orthonormal reduced coordinate system in which the KL trust region becomes a Euclidean ball and constraint interactions are handled in a numerically stable way.
Across multi-constraint benchmarks with tight feasible regions, redundant constraints, and infeasible starts, LOCU achieves stable constraint handling with competitive returns across the starting feasibility and critic architectures.

\newpage
\begin{ack}
This work was supported by the Institute of Information \& Communications Technology Planning \& Evaluation (IITP) grants funded by the Korea government (MSIT) through the Information Technology Research Center program (IITP-2026-RS-2023-00260098) and the Artificial Intelligence Convergence Innovation Human Resources Development program (IITP-2026-RS-2023-00254177). 

We thank Google TPU Research Cloud (TRC) for granting us access to TPUs. Additionally, J.L. thanks Nixon Son for the discussion that inspired the initial idea behind this work.
\end{ack}

\medskip

{\small
\bibliographystyle{plainnat}
\bibliography{ref}
}


\newpage
\appendix
\onecolumn

\theoremstyle{plain}
\newtheorem*{lemSolutionSubspace}{Lemma~\ref{lem:solution-subspace}}
\newtheorem*{lemLowdinOrth}{Lemma~\ref{lem:lowdin-orth}}
\newtheorem*{propCoeffGram}{Proposition~\ref{prop:coeff-gram}}
\newtheorem*{thmReducedEquivalence}{Theorem~\ref{thm:reduced-equivalence}}
\newtheorem*{lemBoundGram}{Lemma~\ref{lem:bound-gram}}
\newtheorem*{propNonbindingGram}{Proposition~\ref{prop:nonbinding-gram}}
\newtheorem*{lemNearCollinear}{Lemma~\ref{lem:near-collinear}}
\newtheorem*{lemParetoDescent}{Lemma~\ref{lem:pareto-descent}}
\section{Supplementary proofs and implementation details}
\label{sec:appendix}

This appendix provides (i) proofs omitted in the main text and (ii) additional derivations and implementation notes.
To keep the section hierarchy consistent, \emph{all omitted proofs} are grouped under
Section~\ref{sec:appendix:omitted-proofs}.
The remaining sections provide supporting material (reduced-space KKT details, screening notes,
recovery mechanics, empirical KL control, and projection-pathology analysis).

\subsection{Notation and shapes} 
\label{sec:appendix:notation}

Let $\mathcal{K}$ be the enforced constraint index set after screening, and write $\bar K:=|\mathcal{K}|$.
Let $\{\mathcal{G}_1,\ldots,\mathcal{G}_m\}$ be the near-collinear clusters and $\{i_1,\ldots,i_m\}$ their representatives.
Define
\[
  S\in\mathbb{R}^{n\times(m+1)},\quad
  O := S^\top F S\in\mathbb{R}^{(m+1)\times(m+1)},\quad
  E := S O^{-1/2}\in\mathbb{R}^{n\times(m+1)}.
\]

Throughout the appendix we work with the generic reduced program
\[
  \max_{w\in\mathbb{R}^{m+1}} \; u^\top w
  \quad \text{s.t.}\quad
  \|w\|_2 \le R,\quad Cw\le \Delta,
\]
where $u\in\mathbb{R}^{m+1}$ is the objective coefficient vector in reduced coordinates.
In LOCU, $u$ is set depending on the update mode:
\[
  u :=
  \begin{cases}
    E^\top g_r & \text{(feasible case)},\\
    \bar g_p   & \text{(infeasible case; Section~\ref{sec:method:recovery})}.
  \end{cases}
\]
For each enforced constraint index $k\in\mathcal{K}$, define the reduced normal
\[
  C_k^\top := E^\top g_{c_k}\in\mathbb{R}^{m+1}.
\]
After fixing an ordering of $\mathcal{K}$, we assemble
\[
  C\in\mathbb{R}^{\bar K\times(m+1)},
  \qquad
  \Delta\in\mathbb{R}^{\bar K},
  \qquad
  R=\sqrt{2\delta}.
\]
We use the Fisher inner product $\langle x,y\rangle_F:=x^\top Fy$ and norm $\|x\|_F:=\sqrt{x^\top Fx}$.

\subsection{Proofs omitted in the main text}
\label{sec:appendix:omitted-proofs}

\subsubsection{Lemma~\ref{lem:solution-subspace}}
\label{sec:appendix:proof-solution-subspace}

\begin{lemSolutionSubspace}
  Problem \eqref{eq:local-trust-region-cmdp} admits an optimal solution
  $\Delta\theta^\star$ such that
  $\Delta\theta^\star \in \mathrm{span}\{\tilde g_r,\tilde g_{c_1},\dots,\tilde g_{c_K}\}$.
\end{lemSolutionSubspace}

\begin{proof}
  Let
  \[
    \mathcal{U}:=\mathrm{span}\{\tilde g_r,\tilde g_{c_1},\dots,\tilde g_{c_K}\}.
  \]
  Take any $\Delta\theta\in\mathbb{R}^n$ and decompose it as
  \[
    \Delta\theta=\Delta\theta_{\parallel}+\Delta\theta_{\perp},
    \qquad
    \Delta\theta_{\parallel}\in\mathcal{U},
    \qquad
    \Delta\theta_{\perp}\in\mathcal{U}^{\perp_F},
  \]
  where $\mathcal{U}^{\perp_F}:=\{v:\langle u,v\rangle_F=0,\ \forall u\in\mathcal{U}\}$.

  First, the linearized reward objective depends only on $\Delta\theta_{\parallel}$:
  \begin{align*}
    g_r^\top \Delta\theta
    &=
    (F\tilde g_r)^\top \Delta\theta
    =
    \tilde g_r^\top F \Delta\theta
    =
    \langle \tilde g_r,\Delta\theta_{\parallel}+\Delta\theta_{\perp}\rangle_F
    =
    \langle \tilde g_r,\Delta\theta_{\parallel}\rangle_F,
  \end{align*}
  because $\tilde g_r\in\mathcal{U}$ and $\Delta\theta_\perp\in\mathcal{U}^{\perp_F}$.

  Similarly, each constraint depends only on $\Delta\theta_{\parallel}$:
  for any $k$,
  \begin{align*}
    g_{c_k}^\top \Delta\theta
    &=
    (F\tilde g_{c_k})^\top \Delta\theta
    =
    \tilde g_{c_k}^\top F \Delta\theta
    =
    \langle \tilde g_{c_k},\Delta\theta_{\parallel}+\Delta\theta_{\perp}\rangle_F
    =
    \langle \tilde g_{c_k},\Delta\theta_{\parallel}\rangle_F,
  \end{align*}
  since $\tilde g_{c_k}\in\mathcal{U}$.

  Finally, the trust-region quadratic term satisfies the Pythagorean decomposition:
  \begin{align*}
    \|\Delta\theta\|_F^2
    &=
    \langle \Delta\theta_{\parallel}+\Delta\theta_{\perp},\Delta\theta_{\parallel}+\Delta\theta_{\perp}\rangle_F
    =
    \|\Delta\theta_{\parallel}\|_F^2+\|\Delta\theta_{\perp}\|_F^2
    \ge
    \|\Delta\theta_{\parallel}\|_F^2.
  \end{align*}
  Therefore, if $\Delta\theta$ is feasible for \eqref{eq:local-trust-region-cmdp}, then $\Delta\theta_{\parallel}$ is also feasible
  (it does not increase the KL quadratic term and keeps all linear terms unchanged) and achieves the same objective value.
  Hence an optimal solution exists in $\mathcal{U}$.
\end{proof}

\subsubsection{Lemma~\ref{lem:lowdin-orth}}
\label{sec:appendix:lowdin-proof}

\begin{lemLowdinOrth}
  If $O\succ 0$ (or if $O$ is regularized to be positive definite), then
  \[
    E^\top F E = I.
  \]
\end{lemLowdinOrth}

\begin{proof}
  Recall the definitions $E=SO^{-1/2}$ and $O=S^\top F S$.
  Then
  \begin{align*}
    E^\top F E
    &=
    \bigl(SO^{-1/2}\bigr)^\top F \bigl(SO^{-1/2}\bigr)
    =
    O^{-1/2}\,S^\top F S\,O^{-1/2}
    =
    O^{-1/2}\,O\,O^{-1/2}
    =
    I.
  \end{align*}
\end{proof}

\subsubsection{Proposition~\ref{prop:coeff-gram}}
\label{sec:appendix:proof-coeff-gram}

\begin{propCoeffGram}
  Let $e_j$ be the $j$th canonical basis vector in $\mathbb{R}^{m+1}$, and let the columns of $S$ be ordered as in
  \eqref{eq:method:S}. Then
  \[
    E^\top g_r = O^{1/2} e_1,
    \qquad
    E^\top g_{c_{i_j}} = O^{1/2} e_{j+1},
    \quad j=1,\dots,m.
  \]
\end{propCoeffGram}

\begin{proof}
  Let $e_1$ be the first canonical vector in $\mathbb{R}^{m+1}$.
  By construction of $S$ in \eqref{eq:method:S}, we have $Se_1=\tilde g_r$ and $g_r=F\tilde g_r$.
  Compute
  \begin{align*}
    S^\top g_r
    &=
    S^\top F \tilde g_r
    =
    S^\top F (Se_1)
    =
    (S^\top F S)e_1
    =
    Oe_1.
  \end{align*}
  Therefore,
  \begin{align*}
    E^\top g_r
    &=
    (SO^{-1/2})^\top g_r
    =
    O^{-1/2} S^\top g_r
    =
    O^{-1/2} (Oe_1)
    =
    O^{1/2} e_1.
  \end{align*}

  For the representative constraint $i_j$, the $(j+1)$th column of $S$ equals $\tilde g_{c_{i_j}}$, i.e.,
  $Se_{j+1}=\tilde g_{c_{i_j}}$, and $g_{c_{i_j}}=F\tilde g_{c_{i_j}}$.
  Repeating the same steps gives
  \[
    E^\top g_{c_{i_j}} = O^{1/2} e_{j+1}.
  \]
\end{proof}

\subsubsection{Theorem~\ref{thm:reduced-equivalence}}
\label{sec:appendix:proof-reduced-equivalence}

\begin{thmReducedEquivalence}
  Assume $O\succ 0$ (or $O$ is regularized so that $O\succ 0$).
  Fix any enforced constraint index set $\mathcal{K}$ and any representative set $\{i_1,\ldots,i_m\}$ used to form
  $S,O,E$ as in \eqref{eq:method:S}--\eqref{eq:method:E}.
  Then the change of variables $\Delta\theta=Ew$ gives a one-to-one correspondence between
  $\mathrm{span}(S)$ and $\mathbb{R}^{m+1}$ and transforms the trust-region CMDP subproblem
  \eqref{eq:local-trust-region-cmdp} restricted to $\Delta\theta\in\mathrm{span}(S)$ and constraints $k\in\mathcal{K}$
  into the reduced program \eqref{eq:method:reduced-qp}.
  In particular, any optimizer $w^\star$ of \eqref{eq:method:reduced-qp} yields a valid update
  $\Delta\theta^\star=Ew^\star$ in the original parameter space.
\end{thmReducedEquivalence}

\begin{proof}
  Since $O\succ 0$, the matrix $O^{-1/2}$ exists and is invertible.
  Hence $E=SO^{-1/2}$ has full column rank, and the map $w\mapsto \Delta\theta=Ew$ is a bijection between
  $\mathbb{R}^{m+1}$ and $\mathrm{span}(E)=\mathrm{span}(S)$.

  We now rewrite each term in \eqref{eq:local-trust-region-cmdp} under the substitution $\Delta\theta=Ew$.

  \paragraph{Objective.}
  We have
  \[
    g_r^\top \Delta\theta
    =
    g_r^\top Ew
    =
    (E^\top g_r)^\top w
    =
    u^\top w,
  \]
  with $u=E^\top g_r$ in the feasible mode.

  \paragraph{Constraints.}
  For any enforced constraint $k\in\mathcal{K}$,
  \[
    g_{c_k}^\top \Delta\theta
    =
    g_{c_k}^\top Ew
    =
    (E^\top g_{c_k})^\top w
    =
    C_k^\top w.
  \]
  Collecting all such constraints yields $Cw\le \Delta$.

  \paragraph{Trust region.}
  Using Lemma~\ref{lem:lowdin-orth},
  \begin{align*}
    \tfrac12 \Delta\theta^\top F \Delta\theta
    &=
    \tfrac12 (Ew)^\top F (Ew)
    =
    \tfrac12 w^\top (E^\top F E) w
    =
    \tfrac12 \|w\|_2^2.
  \end{align*}
  Thus $\tfrac12\Delta\theta^\top F\Delta\theta\le \delta$ is equivalent to $\|w\|_2\le R$.

  Combining these identities shows that \eqref{eq:local-trust-region-cmdp} restricted to $\Delta\theta\in\mathrm{span}(S)$
  is exactly \eqref{eq:method:reduced-qp} in the variable $w$.
  Optimal solutions correspond under the bijection $\Delta\theta=Ew$.
\end{proof}

\subsubsection{Lemma~\ref{lem:bound-gram}}
\label{sec:appendix:proof-bound-gram}

\begin{lemBoundGram}
  For any $\Delta\theta$ satisfying $\|\Delta\theta\|_F \le R$,
  \[
    g_{c_k}^\top \Delta\theta \le \nu_k R.
  \]
\end{lemBoundGram}

\begin{proof}
  Using $g_{c_k}=F\tilde g_{c_k}$,
  \begin{align*}
    g_{c_k}^\top \Delta\theta
    &=
    (F\tilde g_{c_k})^\top \Delta\theta
    =
    \tilde g_{c_k}^\top F\Delta\theta
    =
    \langle \tilde g_{c_k},\Delta\theta\rangle_F.
  \end{align*}
  By Cauchy--Schwarz in the Fisher inner product,
  \[
    \langle \tilde g_{c_k},\Delta\theta\rangle_F
    \le
    \|\tilde g_{c_k}\|_F \, \|\Delta\theta\|_F.
  \]
  If $\|\Delta\theta\|_F\le R$, then
  \[
    g_{c_k}^\top \Delta\theta
    \le
    \|\tilde g_{c_k}\|_F \, R.
  \]
  Finally,
  \[
    \|\tilde g_{c_k}\|_F^2
    =
    \tilde g_{c_k}^\top F\tilde g_{c_k}
    =
    (O_{\mathrm{full}})_{k+1,k+1}
    =
    \nu_k^2,
  \]
  so $\|\tilde g_{c_k}\|_F=\nu_k$.
\end{proof}

\subsubsection{Proposition~\ref{prop:nonbinding-gram}}
\label{sec:appendix:proof-nonbinding-gram}

\begin{propNonbindingGram}
  If $\Delta_k \ge \nu_k R$, then constraint $k$ cannot be active inside the trust region
  $\tfrac12\Delta\theta^\top F\Delta\theta\le\delta$ and can be excluded from $\mathcal{K}$ without changing feasibility
  within that trust region.
\end{propNonbindingGram}

\begin{proof}
  Inside $\tfrac12\Delta\theta^\top F\Delta\theta\le \delta$, we have $\|\Delta\theta\|_F\le \sqrt{2\delta}=R$.
  Lemma~\ref{lem:bound-gram} gives, for all such $\Delta\theta$,
  \[
    g_{c_k}^\top \Delta\theta
    \le
    \nu_k R.
  \]
  If $\Delta_k\ge \nu_k R$, then $g_{c_k}^\top \Delta\theta\le \Delta_k$ holds automatically for every point in the trust region.
  Hence constraint $k$ cannot cut the feasible set within the trust region and can be removed without changing feasibility there.
\end{proof}

\subsubsection{Lemma~\ref{lem:near-collinear}}
\label{sec:appendix:proof-near-collinear}

\begin{lemNearCollinear}
  Let $u_i,u_j$ be Fisher-unit vectors with $\langle u_i,u_j\rangle_F=\rho\in[-1,1]$.
  Then for any $\Delta\theta$ with $\|\Delta\theta\|_F\le R$,
  \[
    \min_{s\in\{+1,-1\}}
    \big|\langle u_i,\Delta\theta\rangle_F - s\,\langle u_j,\Delta\theta\rangle_F\big|
    \le
    \sqrt{2(1-|\rho|)}\,R.
  \]
  In particular, if $|\rho|\ge 1-\varepsilon_{\mathrm{dup}}$, the discrepancy is at most $\sqrt{2\varepsilon_{\mathrm{dup}}}\,R$.
  One may take $s=+1$ when $\rho\ge 0$ and $s=-1$ when $\rho<0$.
\end{lemNearCollinear}

\begin{proof}
  Let $u_i,u_j$ satisfy $\|u_i\|_F=\|u_j\|_F=1$ and $\langle u_i,u_j\rangle_F=\rho\in[-1,1]$.
  For any $s\in\{+1,-1\}$,
  \begin{align*}
    \|u_i - s u_j\|_F^2
    &=
    \langle u_i - s u_j,\, u_i - s u_j\rangle_F \\
    &=
    \|u_i\|_F^2 + \|u_j\|_F^2 - 2s\langle u_i,u_j\rangle_F \\
    &=
    2(1 - s\rho).
  \end{align*}
  Choose $s=+1$ when $\rho\ge 0$ and $s=-1$ when $\rho<0$, so that $s\rho=|\rho|$ and hence
  \[
    \|u_i - s u_j\|_F = \sqrt{2(1-|\rho|)}.
  \]
  Then for any $\Delta\theta$,
  \begin{align*}
    \big|\langle u_i,\Delta\theta\rangle_F - s\,\langle u_j,\Delta\theta\rangle_F\big|
    &=
    \big|\langle u_i - s u_j,\Delta\theta\rangle_F\big| \\
    &\le
    \|u_i - s u_j\|_F \, \|\Delta\theta\|_F
    \le
    \sqrt{2(1-|\rho|)}\,R,
  \end{align*}
  whenever $\|\Delta\theta\|_F\le R$.
  Taking the minimum over $s\in\{+1,-1\}$ yields the claim.
\end{proof}

\subsubsection{Lemma~\ref{lem:pareto-descent}}
\label{sec:appendix:proof-pareto-descent}

\begin{lemParetoDescent}
  If $\bar g_p \neq 0$ defined by \eqref{eq:method:pareto-qp}--\eqref{eq:method:pareto-dir}, then for every $i\in\mathcal{V}$,
  \[
    C_i^\top \bar g_p < 0.
  \]
\end{lemParetoDescent}

\begin{proof}
  Let the violated reduced normals be $\{C_i\}_{i\in\mathcal V}$ (rows of $C$ indexed by $\mathcal{V}$).
  Define the convex hull
  \[
    \mathcal{Q}
    :=
    \mathrm{conv}\{C_i: i\in\mathcal{V}\}
    =
    \left\{\sum_{i\in\mathcal{V}}\alpha_i C_i:\ \alpha\ge 0,\ \mathbf{1}^\top\alpha=1\right\}.
  \]
  Problem \eqref{eq:method:pareto-qp} computes
  \[
    g^\star
    =
    \arg\min_{g\in\mathcal{Q}} \|g\|_2^2,
    \qquad
    g^\star = C_{\mathcal{V}}^\top \alpha^\star,
    \qquad
    \bar g_p = -g^\star.
  \]

  Since $g^\star$ is the Euclidean projection of $0$ onto the closed convex set $\mathcal{Q}$,
  the projection optimality condition implies that for all $g\in\mathcal{Q}$,
  \[
    (g-g^\star)^\top(0-g^\star)\le 0.
  \]
  Rearranging gives
  \[
    g^\top g^\star \ge \|g^\star\|_2^2.
  \]
  Taking $g=C_i$ (note that each $C_i\in\mathcal{Q}$) yields, for every $i\in\mathcal{V}$,
  \[
    C_i^\top g^\star \ge \|g^\star\|_2^2.
  \]
  If $\bar g_p\neq 0$, then $g^\star\neq 0$ and $\|g^\star\|_2^2>0$. Therefore,
  \[
    C_i^\top \bar g_p
    =
    -C_i^\top g^\star
    \le
    -\|g^\star\|_2^2
    <
    0.
  \]
\end{proof}
\subsection{Reduced QP derivation and KKT solver}
\label{sec:appendix:kkt}

Here, we provide a detailed KKT-based solution procedure for the reduced program
\eqref{eq:method:reduced-qp}. The presentation parallels the style of the analytical LQCLP derivation in
CPO (\cite{achiam2017constrained}, Thm.~2), but here we treat the \emph{multi-constraint} case and exploit
the fact that LOCU's L\"owdin basis makes the trust region \emph{Euclidean} (i.e., the quadratic form is $I$ in reduced
coordinates). In particular, when $|\mathcal{J}|=1$, the formulas below reduce exactly to the CPO single-constraint
LQCLP solution specialized to $H=I$.

\paragraph{Problem statement.}
Recall the reduced step subproblem (with $p:=m+1$):
\begin{equation}
  \label{eq:app:reduced-prob}
  \begin{aligned}
    \max_{w\in\mathbb{R}^{p}}\quad
    & u^\top w\\
    \text{s.t.}\quad
    & \|w\|_2 \le R,\\
    & Cw \le \Delta,
  \end{aligned}
\end{equation}
where $C\in\mathbb{R}^{\bar K\times p}$ collects the reduced constraint normals and $\Delta\in\mathbb{R}^{\bar K}$
collects the corresponding slacks. In LOCU, $u$ is either $E^\top g_r$ (feasible mode) or the recovery
direction $\bar g_p$ (feasibility-recovery mode), but the solver below is identical.

\paragraph{Convexity and dual attainability.}
\eqref{eq:app:reduced-prob} is a linear maximization over a closed convex set (intersection of a Euclidean ball and
halfspaces), hence an optimum exists whenever the feasible set is nonempty. In LOCU, nonemptiness is ensured in the
recovery mode because we only constrain already-satisfied constraints and thus $w=0$ is feasible; in the reward mode,
$w=0$ is feasible whenever $\Delta\ge 0$.

\subsubsection{KKT conditions}
\label{sec:appendix:kkt-conditions}
\paragraph{Strong duality and KKT sufficiency.}
\eqref{eq:app:reduced-prob} is a convex SOCP (linear objective with a Euclidean ball and halfspace constraints).
Whenever the reduced feasible set has a nonempty relative interior (e.g., there exists $w$ such that
$\|w\|_2 < R$ and $Cw < \Delta$ for the effective inequality set), strong duality holds and KKT conditions are necessary and sufficient for optimality~\cite{boyd2004convex}. If the problem is (nearly) degenerate in practice, our implementation
handles it using regularization/tolerances in the active-set routine described below.

Write \eqref{eq:app:reduced-prob} as a minimization
\[
  \min_{w}\; -u^\top w
  \quad \text{s.t.}\quad
  Cw-\Delta \le 0,
  \qquad
  \|w\|_2^2-R^2 \le 0.
\]
The Lagrangian is
\[
  \mathcal{L}(w,\lambda,\beta)
  =
  - u^\top w
  +\lambda^\top(Cw-\Delta)
  +\beta(\|w\|_2^2-R^2),
\]
where $\lambda\in\mathbb{R}^{\bar K}$ and $\beta\in\mathbb{R}$ satisfy $\lambda\ge 0$, $\beta\ge 0$.

Stationarity:
\begin{align}
  \label{eq:app:kkt-stat}
  -u + C^\top\lambda + 2\beta w = 0.
\end{align}
Primal feasibility:
\begin{align}
  \label{eq:app:kkt-feas}
  Cw \le \Delta,
  \qquad
  \|w\|_2 \le R.
\end{align}
Dual feasibility:
\begin{align}
  \label{eq:app:kkt-dual}
  \lambda\ge 0,
  \qquad
  \beta\ge 0.
\end{align}
Complementary slackness:
\begin{align}
  \label{eq:app:kkt-comp}
  \lambda_i\bigl(C_i^\top w-\Delta_i\bigr)=0\ \ \forall i,
  \qquad
  \beta\bigl(\|w\|_2^2-R^2\bigr)=0.
\end{align}

\paragraph{Basic structure.}
If $u\neq 0$, the optimum typically lies on the ball boundary, i.e., $\|w^\star\|_2=R$ and thus $\beta^\star>0$.
(If the halfspaces alone already form a bounded polytope that contains the maximizer strictly inside the ball, then
$\beta^\star=0$ is possible; we include this case below for completeness.)

\begin{algorithm}[t]
  \caption{Reduced-space active-set KKT solver for \eqref{eq:app:reduced-prob}}
  \label{alg:app:reduced-active-set}
  \begin{algorithmic}[1]
    \REQUIRE $u\in\mathbb{R}^{p}$, $C\in\mathbb{R}^{\bar K\times p}$, $\Delta\in\mathbb{R}^{\bar K}$, radius $R>0$,
    tolerances $\tau_{\mathrm{feas}},\tau_{\mathrm{dual}}$, regularization $\varepsilon>0$
    \ENSURE $w^\star$
    \STATE If $\|u\|_2\le \tau_{\mathrm{dual}}$, return $w^\star\leftarrow 0$
    \STATE $w \leftarrow R\,u/\|u\|_2$ \hfill (unconstrained ball maximizer)
    \IF{$Cw \le \Delta + \tau_{\mathrm{feas}}\mathbf{1}$}
    \STATE return $w^\star\leftarrow w$
    \ENDIF
    \STATE Initialize active set $\mathcal{J}\leftarrow \{\arg\max_i (C_i^\top w-\Delta_i)\}$ \hfill (most violated)
    \FOR{$t=1,2,\dots$ until convergence}
    \STATE Form $C_{\mathcal{J}},\Delta_{\mathcal{J}}$, compute $M_{\mathcal{J}}=C_{\mathcal{J}}C_{\mathcal{J}}^\top$
    \STATE Solve $(M_{\mathcal{J}}+\varepsilon I)\,A_{\mathcal{J}}^\top = I$ (i.e., $A_{\mathcal{J}}=(M_{\mathcal{J}}+\varepsilon I)^{-1}$)
    \STATE $w_p \leftarrow C_{\mathcal{J}}^\top A_{\mathcal{J}}\Delta_{\mathcal{J}}$
    \STATE $d \leftarrow (I - C_{\mathcal{J}}^\top A_{\mathcal{J}} C_{\mathcal{J}})u$
    \IF{$\|d\|_2 \le \tau_{\mathrm{dual}}$}
    \STATE $w \leftarrow w_p$,\quad $\beta \leftarrow 0$
    \ELSE
    \STATE Compute $(a,b,c)$ via \eqref{eq:app:t-poly} and set $t^\star$ by \eqref{eq:app:t-star}
    \STATE $w \leftarrow w_p + t^\star d$,\quad $\beta \leftarrow 1/(2t^\star)$
    \ENDIF
    \STATE $\lambda_{\mathcal{J}} \leftarrow A_{\mathcal{J}}\big(C_{\mathcal{J}}u - 2\beta\Delta_{\mathcal{J}}\big)$,\quad $\lambda_{\bar{\mathcal{J}}}\leftarrow 0$
    \IF{$\min_{j\in\mathcal{J}}\lambda_j < -\tau_{\mathrm{dual}}$}
    \STATE Remove $j^\star\in\arg\min_{j\in\mathcal{J}}\lambda_j$ from $\mathcal{J}$ \hfill (restore dual feasibility)
    \STATE \textbf{continue}
    \ENDIF
    \IF{$\max_i (C_i^\top w-\Delta_i) > \tau_{\mathrm{feas}}$}
    \STATE Add $i^\star\in\arg\max_i (C_i^\top w-\Delta_i)$ to $\mathcal{J}$ \hfill (restore primal feasibility)
    \STATE \textbf{continue}
    \ENDIF
    \STATE \textbf{return} $w^\star\leftarrow w$
    \ENDFOR
  \end{algorithmic}
\end{algorithm}
\subsubsection{Closed form for a fixed active set}
\label{sec:appendix:kkt-fixed-active}

Let $\mathcal{J}\subseteq\{1,\dots,\bar K\}$ denote a candidate active set (constraints that are tight at optimum),
and write $C_{\mathcal{J}}\in\mathbb{R}^{|\mathcal{J}|\times p}$ and $\Delta_{\mathcal{J}}\in\mathbb{R}^{|\mathcal{J}|}$
for the corresponding rows/entries. In KKT form,
\[
  C_{\mathcal{J}} w = \Delta_{\mathcal{J}},
  \qquad
  \lambda_{\mathcal{J}}\ge 0,
  \qquad
  \lambda_{\bar{\mathcal{J}}}=0.
\]

\paragraph{Eliminating $w$ in terms of $(\beta,\lambda_{\mathcal{J}})$}
Assume first that $\beta>0$. From \eqref{eq:app:kkt-stat},
\begin{equation}
  \label{eq:app:w-beta}
  w(\beta,\lambda_\mathcal{J})
  =
  \frac{1}{2\beta}\Big(u - C_{\mathcal{J}}^\top \lambda_\mathcal{J}\Big).
\end{equation}
Imposing the active equalities $C_{\mathcal{J}}w=\Delta_{\mathcal{J}}$ yields the reduced linear system
\begin{equation}
  \label{eq:app:lambda-beta}
  \big(C_{\mathcal{J}} C_{\mathcal{J}}^\top\big)\lambda_\mathcal{J}
  =
  C_{\mathcal{J}} u - 2\beta\,\Delta_{\mathcal{J}}.
\end{equation}
This is exactly the multi-constraint analogue of the scalar equations in CPO's single-constraint LQCLP derivation:
there, $|\mathcal{J}|=1$ and $C_{\mathcal{J}}C_{\mathcal{J}}^\top$ is a scalar, leading to a fully analytical solution.

\paragraph{Rank and regularization.}
If $C_{\mathcal{J}}C_{\mathcal{J}}^\top$ is ill-conditioned, we solve
\[
  \big(C_{\mathcal{J}} C_{\mathcal{J}}^\top + \varepsilon I\big)\lambda_\mathcal{J}
  =
  C_{\mathcal{J}} u - 2\beta\,\Delta_{\mathcal{J}},
\]
with a small $\varepsilon>0$ (or use a pseudoinverse). In LOCU, near-collinear compression makes this case rare.

\paragraph{Explicit decomposition using a projector}

Define
\[
  M_{\mathcal{J}} := C_{\mathcal{J}} C_{\mathcal{J}}^\top \in \mathbb{R}^{|\mathcal{J}|\times|\mathcal{J}|},
  \qquad
  A_{\mathcal{J}} := M_{\mathcal{J}}^{-1},
  \qquad
  P_{\mathcal{J}} := I - C_{\mathcal{J}}^\top A_{\mathcal{J}} C_{\mathcal{J}} \in \mathbb{R}^{p\times p}.
\]
$P_{\mathcal{J}}$ is the (Euclidean) orthogonal projector onto the nullspace of $C_{\mathcal{J}}$.
Also define the minimum-norm particular solution of the active equalities:
\[
  w_p := C_{\mathcal{J}}^\top A_{\mathcal{J}}\,\Delta_{\mathcal{J}}
  \quad \text{so that}\quad
  C_{\mathcal{J}} w_p = \Delta_{\mathcal{J}}.
\]
Solving \eqref{eq:app:lambda-beta} and substituting into \eqref{eq:app:w-beta} yields the closed form
\begin{equation}
  \label{eq:app:w-projector-form}
  w(\beta)
  =
  w_p + \frac{1}{2\beta}\,P_{\mathcal{J}}u.
\end{equation}
Thus the solution for a fixed $\mathcal{J}$ lies on the affine space $C_{\mathcal{J}}w=\Delta_{\mathcal{J}}$ and moves
only along the nullspace direction $P_{\mathcal{J}}u$ as $\beta$ varies.

\paragraph{Determining $\beta$ from the ball constraint}

If $\beta>0$, complementary slackness enforces $\|w(\beta)\|_2 = R$. Let
\[
  d := P_{\mathcal{J}}u,\qquad t := \frac{1}{2\beta}\ (>0).
\]
Then \eqref{eq:app:w-projector-form} becomes
\begin{equation}
  \label{eq:app:w-affine}
  w(t)=w_p + t\,d,
\end{equation}
and the ball boundary condition is
\begin{equation}
  \label{eq:app:t-quadratic}
  \|w_p + t\,d\|_2^2 = R^2.
\end{equation}
Expanding gives a quadratic equation in $t$:
\begin{equation}
  \label{eq:app:t-poly}
  a\,t^2 + 2b\,t + c = 0,
  \qquad
  a:=\|d\|_2^2,\ \ b:=w_p^\top d,\ \ c:=\|w_p\|_2^2-R^2.
\end{equation}

\paragraph{Case 1: $d\neq 0$ (generic).}
Then $a>0$. Feasibility of the active-equality set inside the ball implies the discriminant is nonnegative:
\[
  b^2-a c \ge 0.
\]
Among the two roots,
\[
  t_{\pm} = \frac{-b \pm \sqrt{b^2-a c}}{a},
\]
the objective increases with $t$ because
\[
  u^\top w(t) = u^\top w_p + t\,u^\top d = u^\top w_p + t\,\|d\|_2^2,
\]
using $u^\top d = u^\top P_{\mathcal{J}}u = \|P_{\mathcal{J}}u\|_2^2=\|d\|_2^2$.
Hence the maximizer for the fixed active set uses the \emph{larger} root:
\begin{equation}
  \label{eq:app:t-star}
  t^\star
  =
  \frac{-b + \sqrt{b^2-a c}}{a},
  \qquad
  \beta^\star = \frac{1}{2t^\star},
  \qquad
  w^\star = w_p + t^\star d.
\end{equation}

\paragraph{Case 2: $d=0$ (degenerate/projection case).}
Then $P_{\mathcal{J}}u=0$, meaning $u\in\mathrm{range}(C_{\mathcal{J}}^\top)$ and the objective is constant over the
affine set $C_{\mathcal{J}}w=\Delta_{\mathcal{J}}$:
if $u=C_{\mathcal{J}}^\top y$, then $u^\top w = y^\top C_{\mathcal{J}}w = y^\top \Delta_{\mathcal{J}}$.
In this case, any feasible $w$ satisfying the active equalities is optimal; we pick the minimum-norm choice $w^\star=w_p$.
If $\|w_p\|_2<R$, one can set $\beta^\star=0$; if $\|w_p\|_2=R$, any $\beta^\star\ge 0$ satisfies the ball slackness.

\paragraph{Recovering multipliers $\lambda_{\mathcal{J}}$}

Once $t^\star$ (equivalently $\beta^\star$) is determined, the multipliers on the active constraints follow from
\eqref{eq:app:lambda-beta}:
\begin{equation}
  \label{eq:app:lambda-from-t}
  \lambda_{\mathcal{J}}^\star
  =
  A_{\mathcal{J}}\Big(C_{\mathcal{J}}u - 2\beta^\star \Delta_{\mathcal{J}}\Big)
  =
  A_{\mathcal{J}}\Big(C_{\mathcal{J}}u - \frac{1}{t^\star}\Delta_{\mathcal{J}}\Big),
  \qquad
  \lambda_{\bar{\mathcal{J}}}^\star = 0.
\end{equation}
By construction, $(w^\star,\lambda^\star,\beta^\star)$ satisfies stationarity, the active equalities, and the ball boundary.
The remaining KKT requirements are:
(i) $\lambda_{\mathcal{J}}^\star\ge 0$ and (ii) all inactive constraints are satisfied strictly (or at least not violated),
i.e., $C w^\star \le \Delta$.

\subsubsection{Active-set search over $\mathcal{J}$}
\label{sec:appendix:kkt-active-set}

The closed form above turns the multi-constraint problem into a small-dimensional active-set search:
for a candidate $\mathcal{J}$, we can compute a \emph{unique} KKT candidate $(w^\star,\lambda^\star,\beta^\star)$ (up to
degenerate cases) via \eqref{eq:app:t-star} and \eqref{eq:app:lambda-from-t}, and then verify feasibility and sign
conditions. Because the reduced dimension $p=m+1$ is small, this procedure is inexpensive.

We use a primal active-set loop:
start from the unconstrained ball maximizer, add violated constraints until feasibility is reached, and drop constraints with
negative multipliers until dual feasibility is restored.

\paragraph{Remarks.}
(i) The loop terminates when both primal feasibility ($Cw\le\Delta$) and dual feasibility
($\lambda\ge 0$) are satisfied; then KKT holds because stationarity and complementarity are enforced by construction.
(ii) The only matrix inverse is $(C_{\mathcal{J}}C_{\mathcal{J}}^\top + \varepsilon I)^{-1}$ with size $|\mathcal{J}|\le p$.
(iii) When $|\mathcal{J}|=1$, the computation of $t^\star$ and $\lambda_{\mathcal{J}}^\star$ reduces to the familiar
two-constraint (linear + quadratic) LQCLP algebra in CPO, with the additional simplification that $H=I$ in LOCU's
reduced coordinates.

\paragraph{Optional: one-dimensional search for $\beta$ (implementation alternative)}

Instead of solving \eqref{eq:app:t-poly} in closed form, one may determine $\beta$ by a monotone 1D search.
For fixed $\mathcal{J}$, \eqref{eq:app:w-projector-form} implies $\|w(\beta)\|_2$ is nonincreasing in $\beta$ whenever
$d=P_{\mathcal{J}}u\neq 0$. Therefore, one can bracket $\beta$ and apply bisection to solve $\|w(\beta)\|_2=R$.
This produces the same $w^\star$ as \eqref{eq:app:t-star} up to numerical tolerance.

\paragraph{Local approximation error.}
The preceding LQCLP derivation relies on the local linearization used in CPO~\cite[Section~6.1, Eq.~(11)]{achiam2017constrained}. To bound the resulting approximation error, hold the current Fisher matrix $F$ fixed and assume each reward or cost objective $J_i$, $i\in\{R,C_1,\ldots,C_K\}$, is twice continuously differentiable near the update segment and satisfies
\begin{equation}
\label{eq:app:local-smoothness}
\sup_{\tau\in[0,1]}\left\|F^{-1/2}\nabla_\theta^2J_i(\pi_{\theta+\tau\Delta\theta})F^{-1/2}\right\|_2\le L_i,\qquad L_i\ge0.
\end{equation}
For any update within the trust region of \eqref{eq:local-trust-region-cmdp}, Taylor's theorem and the fundamental theorem of calculus give
\begin{align}
\label{eq:app:local-objective-error}
\left|J_i(\pi_{\theta+\Delta\theta})-J_i(\pi_\theta)-\nabla_\theta J_i(\pi_\theta)^\top\Delta\theta\right|&\le L_i\delta,\\
\label{eq:app:local-gradient-error}
\left\|F^{-1/2}\bigl(\nabla_\theta J_i(\pi_{\theta+\Delta\theta})-\nabla_\theta J_i(\pi_\theta)\bigr)\right\|_2&\le L_i\sqrt{2\delta}.
\end{align}
Consequently, any constraint whose linearized inequality is satisfied obeys
\begin{equation}
\label{eq:app:local-cost-error}
g_{c_k}^\top\Delta\theta\le\Delta_k\quad\Longrightarrow\quad J_{C_k}(\pi_{\theta+\Delta\theta})\le b_k+L_{C_k}\delta.
\end{equation}
Let $\phi_k$ be the angle between $F^{-1/2}g_{c_k}$ and $F^{-1/2}\nabla_\theta J_{C_k}(\pi_{\theta+\Delta\theta})$, with $\nu_k$ defined in \eqref{eq:method:screen-metrics}. The gradient bound implies
\begin{equation}
\label{eq:app:local-angle-error}
\phi_k\le\arcsin\!\left(\frac{L_{C_k}\sqrt{2\delta}}{\nu_k}\right)\quad\text{if }L_{C_k}\sqrt{2\delta}<\nu_k.
\end{equation}
Higher local curvature or a smaller original gradient norm weakens this directional bound.
\subsection{Practical implementation}
\label{sec:appendix:practical_impl}
\paragraph{Fisher-vector products and natural gradients.}
Fisher-vector products are computed using Hessian--vector products of the batch-average KL divergence between the cached old policy and the current policy,
\begin{equation}
  \bar D_{\mathrm{KL}}(\pi_{\theta_{\mathrm{old}}}\,\|\,\pi_{\theta})
  :=
  \mathbb{E}_{s\sim\mathcal{B}}
  \Big[
    \mathrm{KL}\big(\pi_{\theta_{\mathrm{old}}}(\cdot\mid s)\,\|\,\pi_{\theta}(\cdot\mid s)\big)
  \Big],
  \qquad
  Fv \approx \nabla^2_{\theta}\,\bar D_{\mathrm{KL}}(\pi_{\theta_{\mathrm{old}}}\,\|\,\pi_{\theta})\,v .
\end{equation}
Natural-gradient directions for the reward and constraints are obtained by conjugate-gradient (CG) solves with Tikhonov damping.
We use a fixed number of CG iterations and add a diagonal damping term to the implicit Fisher operator to stabilize the solves under finite-sample noise.

\paragraph{Fisher Gram matrix and L\"owdin basis construction.}
Let $S$ stack the reward direction and the representative constraint directions used for basis construction.
We form the Fisher Gram matrix
\[
  O := S^\top F S \in \mathbb{R}^{(m+1)\times(m+1)} .
\]
Before computing the inverse square root, we symmetrize and regularize $O$ as
\[
  O \leftarrow \tfrac12 (O+O^\top) + \varepsilon_{\mathrm{reg}} I .
\]
We compute $O^{-1/2}$ via eigendecomposition,
\[
  O = U \Lambda U^\top,
  \qquad
  O^{-1/2} = U \Lambda^{-1/2} U^\top ,
\]
where eigenvalue flooring is applied by replacing each $\lambda_j$ with $\max\{\lambda_j,\lambda_{\min}\}$ prior to inversion.
This prevents numerical instability when the Gram matrix is ill-conditioned.
The L\"owdin basis is then constructed as
\[
  E := S\,O^{-1/2}.
\]

\paragraph{Reduced quantities and parameter update.}
Reduced-space quantities are obtained by projection onto the L\"owdin basis:
\[
  u := E^\top g_r,
  \qquad
  C_k^\top := E^\top g_{c_k}, \quad k\in\mathcal{K}.
\]
The trust-region radius is set to $R=\sqrt{2\delta}$.
After solving the reduced program for $w^\star$, the actor update in parameter space is given by
\[
  \Delta\theta^\star = E w^\star .
\]

\paragraph{Reward-aligned constraint filtering (optional).}
During basis construction, constraints whose Fisher cosine similarity with the reward direction exceeds a fixed threshold are optionally excluded from basis candidates.
This heuristic avoids introducing nearly collinear reward--constraint directions into the reduced basis.
In our benchmark environments, strong alignment between reward and constraint gradients is rarely observed in practice, and we found that disabling this filter does not measurably affect performance.
Accordingly, we retain this step as an optional numerical safeguard rather than a core assumption of the method.

\paragraph{Numerical tolerances.}
All linear solves in the reduced space use small ridge regularization.
Feasibility and dual checks in the active-set routine are evaluated with fixed tolerances.
If the reduced objective norm falls below a small threshold, the update defaults to the zero step.

\begin{algorithm}[t]
  \caption{LOCU one actor update}
  \label{alg:locu}
  \begin{algorithmic}[1]
    \REQUIRE policy $\pi_\theta$, trust region $\delta$, limits $\{b_k\}_{k=1}^K$,
    thresholds $\varepsilon_{\mathrm{nz}},\varepsilon_{\mathrm{dup}}$, regularization $\varepsilon_{\mathrm{reg}}$
    \STATE Estimate reward and cost gradients $g_r$, $\{g_{c_k}\}_{k=1}^K$
    \STATE Estimate Fisher-vector products for the FIM $F$
    \STATE Compute slacks $\Delta_k=b_k-J_{C_k}(\pi_\theta)$ and set $R\leftarrow \sqrt{2\delta}$
    \STATE Compute natural gradients $\tilde g_r=F^{-1}g_r$ and $\tilde g_{c_k}=F^{-1}g_{c_k}$ (with damping if needed)
    \STATE Form $S_{\mathrm{full}}=[\tilde g_r,\tilde g_{c_1},\dots,\tilde g_{c_K}]$ and $O_{\mathrm{full}}=S_{\mathrm{full}}^\top F S_{\mathrm{full}}$
    \STATE Compute $\nu_k,\mu_k,\varrho_{ij}$ via \eqref{eq:method:screen-metrics}--\eqref{eq:method:screen-cos}
    \STATE Build enforced set $\mathcal{K}$ by excluding
    (i) non-binding constraints (Prop.~\ref{prop:nonbinding-gram}) and
    (ii) near-zero constraints $\nu_k\le \varepsilon_{\mathrm{nz}}$
    \STATE Build near-collinear groups on $\mathcal{K}$ using edges $|\varrho_{ij}|\ge 1-\varepsilon_{\mathrm{dup}}$
    \STATE Set $i_\ell \in \arg\min_{k\in\mathcal{G}_\ell} \mu_k$ for each group $\mathcal{G}_\ell$
    \STATE Form $S=[\tilde g_r,\tilde g_{c_{i_1}},\dots,\tilde g_{c_{i_m}}]$ and $O=S^\top F S$; regularize $O$ if needed
    \STATE Compute L\"owdin basis $E=S O^{-1/2}$
    \STATE Build reduced quantities $u=E^\top g_r$ and $C_k^\top=E^\top g_{c_k}$ for all $k\in\mathcal{K}$
    \IF{all retained slacks satisfy $\Delta_k \ge 0$}
      \STATE Solve the reduced program \eqref{eq:method:reduced-qp} for $w^\star$
    \ELSE
      \STATE Compute Pareto descent direction $\bar g_p$ via \eqref{eq:method:pareto-qp}--\eqref{eq:method:pareto-dir}
      \STATE Set $u\leftarrow \bar g_p$ and solve the recovery program \eqref{eq:method:recovery-qp} for $w^\star$
    \ENDIF
    \STATE Set $\Delta\theta^\star \leftarrow E w^\star$
    \STATE Update $\theta \leftarrow \theta + \Delta\theta^\star$
  \end{algorithmic}
\end{algorithm}


\section{Experimental Details}
\label{sec:appendix:experimental-details}
\subsection{Task details}\label{sec:appendix:Task details}
\subsubsection{Circle Task from Bullet--safety--gym}
\label{sec:appendix:experimental-details:bullet-circle}

We use the Circle tasks from Bullet-Safety-Gym~\cite{Gronauer2022BulletSafetyGym} with Ball, Car, and Drone agents and the multi-constraint extension from prior work~\cite{yao2024gradient}. The reward of Circle task is defined as
\begin{equation}
  \label{eq:circle_reward}
  r_t
  =
  0.1\,\frac{p_t^\top (J v_t)}{1+\big|\|p_t\|_2-r_{\mathrm{circ}}\big|}
  +0.01\,r_t^{\mathrm{spec}},
  \qquad
  J=
  \begin{bmatrix}0&-1\\1&0
  \end{bmatrix}.
\end{equation}
This term captures the signed tangential velocity around the origin, yielding positive values for clockwise motion along the circular trajectory.
The normalization factor attenuates the reward as the agent deviates from the target radius $r_{\mathrm{circ}}$, so the reward is maximized near $\|p_t\|_2 \approx r_{\mathrm{circ}}$ under sustained clockwise rotation.

The agent-specific reward component $r_t^{\mathrm{spec}}$ follows
\begin{equation}
  \label{eq:circle_agent_reward}
  r_t^{\mathrm{spec}}=
  \begin{cases}
    -0.5\,\|a_t\|_2, & \text{Ball},\\
    0, & \text{Car},\\
    r_t^{\mathrm{alive}}-0.01\,\|\omega_t\|_2^2 - c_t^{\mathrm{ground}}, & \text{Drone}.
  \end{cases}
\end{equation}
Here $a_t$ denotes the clipped action and $\omega_t$ the angular velocity.
The term $c_t^{\mathrm{ground}}$ denotes the ground-contact penalty used by the environment.
This shaping term provides mild agent-dependent regularization without altering the task objective.
For the Ball agent, the $\ell_2$ action penalty regularizes control effort.
The Car agent uses no additional shaping beyond the task reward.
For the Drone agent, $r_t^{\mathrm{alive}}$ provides a positive reward when the agent remains in a valid flying state and a negative reward otherwise, while the remaining terms discourage aggressive rotations and unsafe ground contacts.

\begin{figure*}[t]
  \centering
  \begin{minipage}[t]{0.18\textwidth}
    \centering
    \includegraphics[width=\textwidth, trim=80 40 80 40, clip]{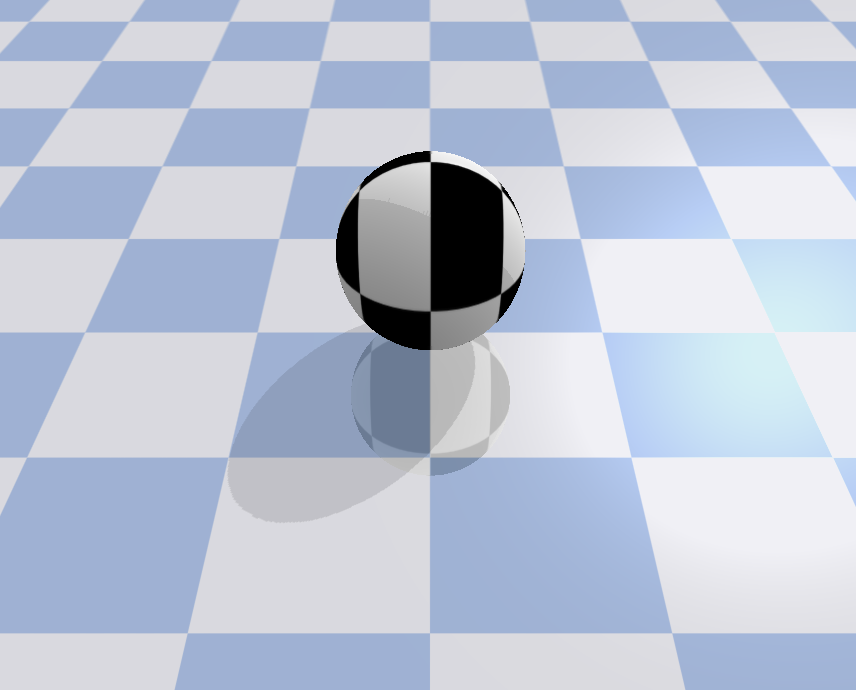}
    \\[-0.3em] Ball
  \end{minipage}
  \hfill
  \begin{minipage}[t]{0.18\textwidth}
    \centering
    \includegraphics[width=\textwidth, trim=80 40 80 40, clip]{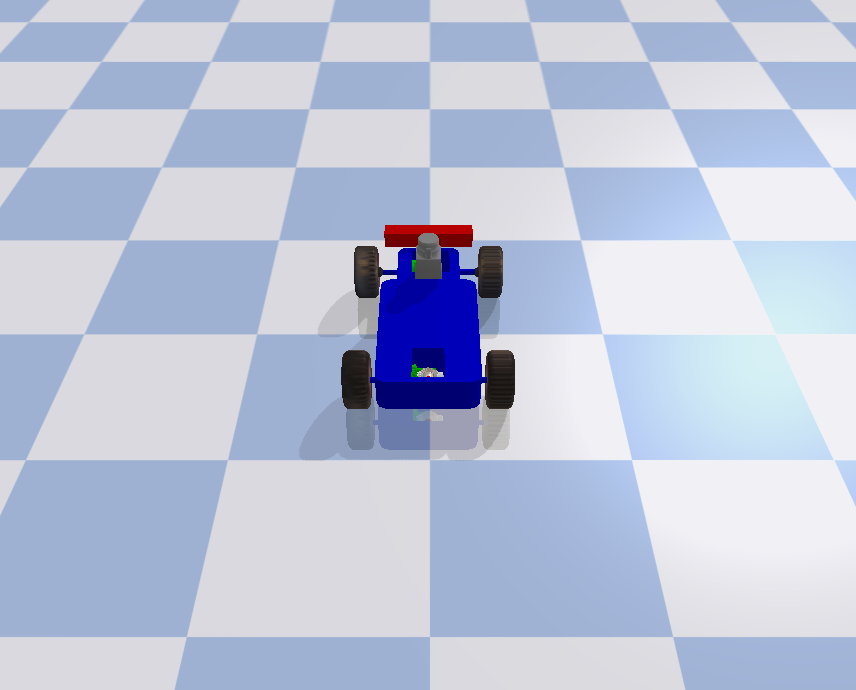}
    \\[-0.3em] Car
  \end{minipage}
  \hfill
  \begin{minipage}[t]{0.18\textwidth}
    \centering
    \includegraphics[width=\textwidth, trim=80 40 80 40, clip]{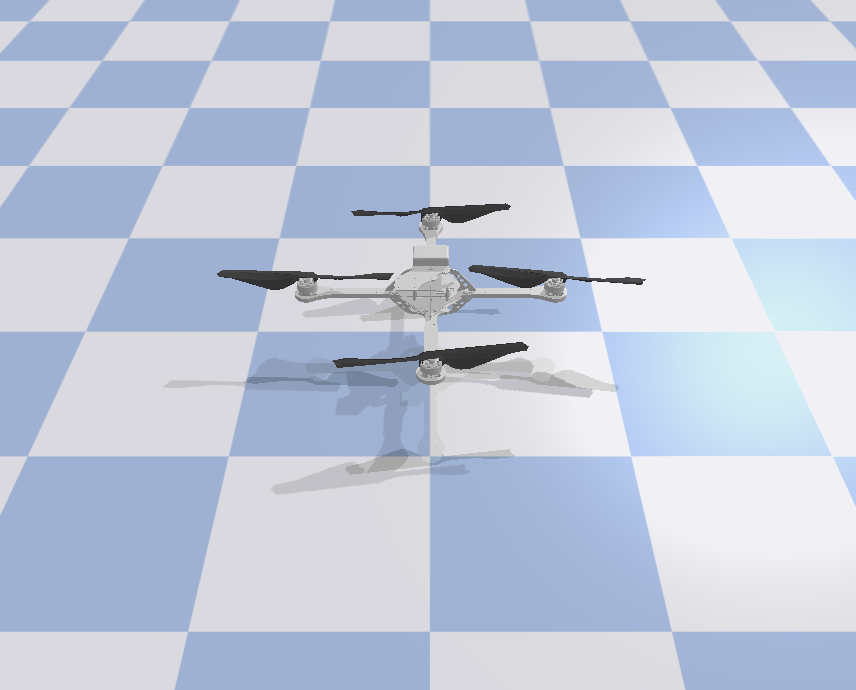}
    \\[-0.3em] Drone
  \end{minipage}
  \hfill
  \begin{minipage}[t]{0.18\textwidth}
    \centering
    \includegraphics[width=\textwidth]{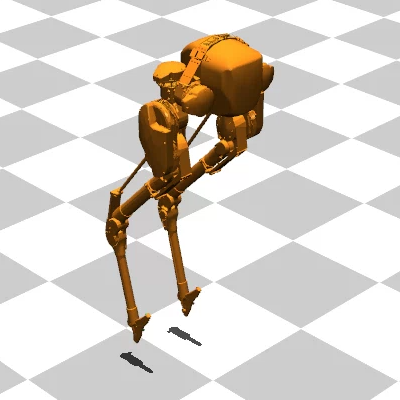}
    \\[-0.3em] Cassie
  \end{minipage}
  \hfill
  \begin{minipage}[t]{0.18\textwidth}
    \centering
    \includegraphics[width=\textwidth]{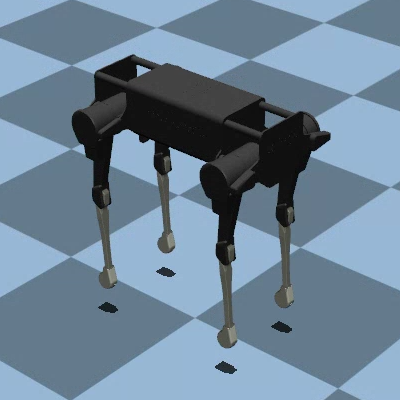}
    \\[-0.3em] Laikago
  \end{minipage}

  \caption{
    Agents used in the multi-constraint evaluation.
  Ball, Car, and Drone are BulletSafetyGym agents, while Cassie and Laikago are bipedal and quadrupedal robots, respectively.}
  \label{fig:benchmark_agents}
\end{figure*}

\paragraph{Constraints.}
We use three per-step binary costs: a boundary cost, a high-velocity cost, and a low-velocity cost.
Let $x_t$ denote the agent position along the $x$-axis and $v_t$ its planar speed.
With boundary threshold $x_{\mathrm{bnd}}>0$, base velocity limit $v_{\max}$, and scaling coefficients
$\kappa_{\mathrm{hi}}\in(0,1]$ and $\kappa_{\mathrm{lo}}\in(0,1)$, we define
\begin{align}
  c_{\mathrm{bdry},t} &= \mathbb{I}\{|x_t| > x_{\mathrm{bnd}}\}, \\
  c_{\mathrm{hi},t}   &= \mathbb{I}\{v_t > \kappa_{\mathrm{hi}}\,v_{\max}\}, \\
  c_{\mathrm{lo},t}   &= \mathbb{I}\{v_t < \kappa_{\mathrm{lo}}\,v_{\max}\}.
\end{align}
The two-constraint setting $\{c_{\mathrm{bdry}},c_{\mathrm{hi}}\}$ is referred to as \textbf{v2}, and the three-constraint setting $\{c_{\mathrm{bdry}},c_{\mathrm{hi}},c_{\mathrm{lo}}\}$ as \textbf{v3}.
Unless otherwise stated, the base v3 configuration uses $x_{\mathrm{bnd}}=6.0$, $\kappa_{\mathrm{hi}}=1.0$, and $\kappa_{\mathrm{lo}}=0.6$.
\begin{table}[t]
  \centering
  \caption{BC-v3 stress-test configurations used in Figure~\ref{fig:results:v3}.}
  \label{tab:bcv3_redun_thresholds}
  \setlength{\tabcolsep}{4pt}
  \renewcommand{\arraystretch}{1.08}
  \small
  \begin{fittable}
  \begin{tabular}{lccc}
    \toprule
    Setting
    & \makecell{Low-velocity\\coefficients $\kappa_{\mathrm{lo}}$}
    & \makecell{Boundary\\thresholds $x_{\mathrm{bnd}}$}
    & \makecell{High-velocity\\coefficients $\kappa_{\mathrm{hi}}$} \\
    \midrule
    Base (v3)
    & $\{\mathbf{0.6}^{(-a)}\}$
    & $\{\mathbf{6.0}^{(-a)}\}$
    & $\{\mathbf{1.0}^{(-a)}\}$ \\
    Redundant-2
    & $\{\mathbf{0.6}^{(-a)},\,0.65\}$
    & $\{\mathbf{6.0}^{(-a)},\,6.5\}$
    & $\{\mathbf{1.0}^{(-a)},\,0.95\}$ \\
    Redundant-3
    & $\{\mathbf{0.6}^{(-a)},\,0.65,\,0.7\}$
    & $\{\mathbf{6.0}^{(-a)},\,6.5,\,7.0\}$
    & $\{\mathbf{1.0}^{(-a)},\,0.95,\,0.9\}$ \\
    \bottomrule
  \end{tabular}
  \end{fittable}
\end{table}

\paragraph{BC-v3 Variants.}
Figure~\ref{fig:results:v3} reports results on several BC-v3 variants obtained by tightening and duplicating constraints on top of the base v3 setting.
All variants share the same task dynamics and reward, and only differ in the constraint coefficients summarized in Table~\ref{tab:bcv3_redun_thresholds}.
Boldfaced coefficients (marked with $^{(-a)}$) correspond to the constraints denoted as ``--a'' in the results and learning-curve figures.

As redundancy increases, the low-velocity coefficient $\kappa_{\mathrm{lo}}$ increases, raising the minimum admissible speed.
This tightens the low-velocity constraint and prolongs the recovery phase early in training, since the agent must first escape low-speed states before achieving sustained motion along the circle.
In parallel, the high-velocity coefficient $\kappa_{\mathrm{hi}}$ decreases, tightening the upper speed limit.
Because the task reward is proportional to tangential velocity, this directly reduces the achievable return by limiting the speed at which the agent can traverse the trajectory.

The tightest boundary coefficient remains fixed across all variants.
Therefore, additional boundary constraints do not further restrict the feasible set, but only introduce near-collinear constraint directions, isolating the effect of constraint redundancy on optimization behavior.

\subsubsection{Locomotion tasks}
\label{sec:appendix:experimental-details:locomotion}

In locomotion tasks the agent tracks velocity commands on flat terrain.
At the beginning of each episode, a command vector
\[
  (v^{\text{cmd}}_x, v^{\text{cmd}}_y, \omega^{\text{cmd}}_z)
\]
is sampled, where $v^{\text{cmd}}_y$ is fixed to $0$ to simplify the task.
In the implementation, $(v^{\text{cmd}}_x,\omega^{\text{cmd}}_z)$ are sampled uniformly from robot-specific ranges:
Cassie uses $v^{\text{cmd}}_x\!\sim\!\mathcal{U}[-1,1]$, while Laikago uses $v^{\text{cmd}}_x\!\sim\!\mathcal{U}[-1,2]$,
and both use $\omega^{\text{cmd}}_z\!\sim\!\mathcal{U}[-0.5,0.5]$.
The reward penalizes squared tracking error between commanded and realized base velocities and adds a small power regularization term.

\paragraph{Constraint Setting}
At time step $t$, let $\theta_t$ denote the angle between the robot base $z$-axis and the world $z$-axis, and let $a>0$ be a fixed balance threshold.
Let $h_t$ denote the center-of-mass (CoM) height with a robot-specific minimum threshold $h_{\min}>0$.
Let $\hat{\xi}_t \in \{-1,1\}^{n_{\mathrm{legs}}}$ be the measured foot contact vector and $\xi_t \in \{-1,1\}^{n_{\mathrm{legs}}}$ the desired periodic contact pattern derived from phase variables.

We define the following per-step costs:
\begin{align}
  c_{\mathrm{bal},t}  &= \mathbb{I}\{\theta_t > a\}, \\
  c_{\mathrm{hgt},t}  &= \mathbb{I}\{h_t < h_{\min}\}, \\
  c_{\mathrm{foot},t} &= \frac{1}{n_{\mathrm{legs}}} \sum_{i=1}^{n_{\mathrm{legs}}}
  \mathbb{I}\{\hat{\xi}_{t,i} \neq \xi_{t,i}\}.
\end{align}

For the balance constraint, we use $a = 15^\circ$.
For the height constraint, we set $h_{\min} = 0.35$ for Laikago and $h_{\min} = 0.7$ for Cassie.
The desired contact pattern $\xi_t$ is defined using per-leg phase variables with a fixed frequency, following the benchmark specification.

Constraint limits are computed from target per-step cost rates by scaling with $1/(1-\gamma)$, following the standard rule for per-step costs bounded in $[0,1]$.
The foot timing constraint is assigned a looser limit than the balance and height constraints, as it primarily serves as a motion regularizer rather than a strict safety constraint.

\paragraph{Episode length and termination handling.}
Episodes run for a fixed horizon of $T$ steps and are not terminated early upon falling.
When a fall is detected, the state and reward are held fixed until the horizon ends, preventing termination effects from confounding constraint evaluation.

\subsubsection{Bipedal Navigation task}
\label{sec:appendix:experimental-details:bipedal-navigation}

\paragraph{Objective.}
To increase task diversity beyond velocity-tracking locomotion, we construct a Cassie goal-navigation task with obstacle avoidance, denoted as CassieGoalHazard-v0. The task extends the Cassie locomotion environment by replacing the commanded-velocity reward with a planar goal-reaching objective and by adding hazard-avoidance constraints adapted from Safety-Gymnasium~\cite{ji2023safety}. This setting is designed to couple legged locomotion stability with state-dependent navigation constraints while preserving a feasible but nontrivial multi-constraint control problem.

At the beginning of each episode, eight circular hazards and one goal are sampled in a square map with extents $[-3.5,3.5]\times[-3.5,3.5]$. The goal radius is $r_{\mathrm{goal}}=0.5$, and each hazard has radius $r_{\mathrm{haz}}=0.35$. The observation augments the Cassie proprioceptive state with 16-bin pseudo-LiDAR measurements for the goal and hazards, together with the goal distance.

Let $d_t$ denote the planar distance between the robot base and the current goal at time step $t$. The reward is the distance reduction toward the goal plus a sparse goal bonus:
\begin{equation}
  r_t
  =
  d_{t-1}-d_t
  +0.2\,\mathbb{I}\{d_t\le r_{\mathrm{goal}}\}.
\end{equation}
When the goal is reached, a new goal is sampled while the hazard layout remains fixed. Thus, the task rewards repeated navigation progress rather than a single terminal success event.

\paragraph{Constraint setting.}
The task retains the three Cassie locomotion-stability costs from Section~\ref{sec:appendix:experimental-details:locomotion} and adds three navigation-specific costs. The full cost vector is ordered as$(c_{\mathrm{bal}},c_{\mathrm{hgt}},c_{\mathrm{foot}},c_{\mathrm{spd}},c_{\mathrm{haz}},c_{\mathrm{eng}}).$ The first three costs are the balance, height, and foot-timing costs used in the locomotion task, with the Cassie height threshold set to $h_{\min}=0.7$.

For the additional costs, let $v_{x,t}$ denote the forward base velocity, let $\mathcal{P}_t$ be the set of floor-contact points of the robot at time $t$, and let $\mathcal{H}$ be the set of hazard centers. We define
\begin{align}
  c_{\mathrm{spd},t}
  &=
  \mathbb{I}\{|v_{x,t}|>1.0\}, \\
  c_{\mathrm{haz},t}
  &=
  \mathbb{I}\{\exists p\in\mathcal{P}_t,\exists h\in\mathcal{H}:\|p-h\|_2\le r_{\mathrm{haz}}\}, \\
  c_{\mathrm{eng},t}
  &=
  \mathbb{I}\{e_t>13.0\},
\end{align}
where $e_t$ is the average actuator-power proxy computed from the absolute actuator force times joint velocity over simulation substeps.

Constraint limits are converted by the same $1/(1-\gamma)$ scaling used for the locomotion tasks. These thresholds were selected through pilot calibration to keep simultaneous satisfaction feasible while avoiding overly loose constraints.

This task introduces qualitatively heterogeneous constraints. Goal-reaching encourages faster progress, while hazard avoidance can require detours, speed regulation can restrict aggressive motion, and energy control penalizes high-actuation strategies. The resulting environment therefore tests whether an algorithm can recover feasibility from infeasible initialization while still incorporating reward-improving directions once the active constraints are controlled. All experiments use the same fixed horizon of $T=500$ steps.

\begin{table}[t]
  \centering
  \caption{Actor/critic architectures and training hyperparameters (shared across methods unless noted).}
  \label{tab:algo_hparams}
  \setlength{\tabcolsep}{3pt}
  \renewcommand{\arraystretch}{1.05}
  \small
  \begin{fittable}
  \begin{tabular}{lccccc}
    \toprule
    Parameter
    & LOCU
    & SDAC(A)
    & PCRPO / CRPO
    & TRPOL / GS
    & PPOL / GS \\
    \midrule
    Actor MLP (hidden, act.)
    & \multicolumn{4}{c}{(128, 128) ReLU / (512, 512) LeakyReLU}
    & (128, 128) ReLU / -- \\
    Policy distribution
    & \multicolumn{4}{c}{Diag.\ Gaussian / squashed Gaussian}
    & Diag.\ Gaussian / -- \\
    Action bounding
    & \multicolumn{4}{c}{clip / tanh-squash}
    & clip / -- \\
    Critic architecture
    & \multicolumn{4}{c}{scalar $V$-critic / distributional $Q$-critic}
    & scalar $V$-critic / -- \\
    Distributional atoms $(M/M')$
    & \multicolumn{4}{c}{-- / (25/50)}
    & -- / -- \\
    \midrule
    Target KL
    & \multicolumn{4}{c}{0.001}
    & 0.02 / -- \\
    Critic optimization iters
    & \multicolumn{4}{c}{20 / 40}
    & -- \\
    Critic learning rate
    & \multicolumn{4}{c}{1e-4 / 3e-4}
    & 1e-4 / -- \\
    L2 regularization
    & \multicolumn{4}{c}{0.001}
    & -- \\
    Damping coefficient
    & \multicolumn{4}{c}{0.1 / 0.01}
    & -- \\
    CG iters
    & \multicolumn{4}{c}{10}
    & -- \\
    Batch size (rollout)
    & \multicolumn{5}{c}{10000} \\
    Minibatch size
    & \multicolumn{4}{c}{--}
    & 256 / -- \\
    Discount factor $\gamma$
    & \multicolumn{5}{c}{0.99} \\
    GAE / trace-decay $\lambda$
    & \multicolumn{5}{c}{0.95 / 0.97} \\
    Advantage normalization
    & \multicolumn{5}{c}{True} \\
    PID gains $(k_p,k_i,k_d)$
    & --
    & --
    & --
    & (0.2, 0.002, 0.1)
    & (0.15, 0.005, 0.1) \\
    \midrule
    \multicolumn{6}{l}{\footnotesize Entries $a/b$ denote Bullet-Safety-Gym / Locomotion, respectively.} \\
    \multicolumn{6}{l}{\footnotesize ``--'' indicates not applicable or not evaluated; PPOL/GS was not run on locomotion tasks.} \\
    \bottomrule
  \end{tabular}
  \end{fittable}
\end{table}

\subsection{Implementation details}
\label{sec:appendix:experimental-details:architecture}

\subsubsection{Baseline Implementation}

\paragraph{SDAC, PID Lagrangian, and Gradient Shaping (GS).}
For SDAC, we use the official implementation released by the original authors.
For PID Lagrangian and Gradient Shaping (GS), we adopt the implementation built on top of the FSRL framework~\cite{liu2024offlinesaferl}, as released by the authors of GS.
Across these baselines, we align common experimental settings and remaining algorithm-specific components, including multiplier dynamics, and shaping rules, are kept identical to the respective authors' implementations without modification.

\paragraph{CRPO.}
We build CRPO on top of the CRPO implementation provided by OmniSafe~\cite{ji2024omnisafe} and extend it to the multi-constraint setting.
When multiple constraints are violated simultaneously, we sample a single constraint to update the policy at each iteration.
The sampling probability is proportional to the magnitude of constraint violation, so that more severely violated constraints are prioritized during training.

\paragraph{PCRPO.}
To the best of our knowledge, no official public implementation of PCRPO is available, although the original paper reports that PCRPO is implemented and evaluated using the OmniSafe framework.
We therefore implement PCRPO directly following the algorithmic description and the configuration settings reported in the paper.
In particular, the appendix provides a configuration table that lists shared hyperparameter settings for PCRPO, including the use of CG iterations.
Based on this specification, we implement PCRPO as a second-order primal method that computes natural-gradient style updates using a CG with its soft-switching method.

PCRPO uses two slack thresholds per constraint, denoted as an upper slack and a lower slack, to determine which objective is optimized at each update.
We set these thresholds by scaling the corresponding constraint limit $b_k$.
In Bullet-Safety-Gym, we use an upper slack of $1.25\,b_k$ and a lower slack of $0.75\,b_k$.
In locomotion tasks, we use an upper slack of $1.0\,b_k$ and a lower slack of $0.75\,b_k$, applied per constraint.

At each iteration, constraints are categorized into three groups based on their current costs relative to the slack thresholds:
(i) constraints with cost above the upper slack,
(ii) constraints with cost between the upper and lower slacks, and
(iii) constraints with cost below the lower slack.
If any constraint falls into group (i), one violated constraint is sampled from this group and a constraint-reduction update is performed.
Otherwise, if group (ii) is non-empty, one constraint is sampled and a reward update is applied together with the gradient-manipulation rule specified by PCRPO to mitigate interference between reward and constraint gradients.
If all constraints belong to group (iii), a feasible update is performed.

\paragraph{Clarabel.}
We use Clarabel to solve the local constrained policy-update problem in two forms: full QP and reduced QP. The full QP optimizes the coefficients of the natural reward and cost gradients, while the reduced QP optimizes the update in LOCU's L\"owdin coordinates, including the orthogonal reward direction. For constraint recovery, the full QP minimizes total constraint slack when the reward-maximization problem is infeasible. The reduced QP minimizes slack for currently violated constraints while enforcing the other constraints.

\subsubsection{Mean--std constraint and distributional cost critics}
\label{sec:appendix:constraint_critics}

SDAC adopts the mean--std safety measure
$F(Z;\alpha_{\mathrm{lo}})
=\mathbb{E}[Z]
+\frac{\phi(\Phi^{-1}(\alpha_{\mathrm{lo}}))}{\alpha_{\mathrm{lo}}}\mathrm{Std}[Z]$
motivated as a practical surrogate via the Gaussian--CVaR connection.
Following the discussion in SDAC, we treat mean--std as an uncertainty-weighted surrogate rather than an exact CVaR constraint \cite{kim2023trust}.

For locomotion tasks, we use a distributional cost critic and denote the predicted quantile samples for constraint $k$ at $(s,a)$ by $\{z_k^{(m)}(s,a)\}_{m=1}^M$.
In our locomotion experiments, we adopt the same task definitions and training protocol as SRCPO \cite{kim2024spectral}.
We keep the distributional critic training unchanged, including the quantile regression objective and the construction of bootstrapped target quantiles, and modify only the aggregation used to compute the mean--std surrogate from predicted quantiles.

The original SDAC aggregation estimates the mean--std surrogate by mixing quantile samples across the batch and then taking a single global standard deviation:
\begin{equation}
\label{eq:mixed-mean-std}
\begin{aligned}
  \widehat{F}_k^{\mathrm{mixed}}
  &=
  \mathrm{Mean}\!\Big(\{z_k^{(m)}(s,a)\}_{(s,a)\in\mathcal{B},\,m}\Big)
  + \sigma(\alpha_{\mathrm{lo}})
  \mathrm{Std}\!\Big(\{z_k^{(m)}(s,a)\}_{(s,a)\in\mathcal{B},\,m}\Big), \\
  \sigma(\alpha_{\mathrm{lo}})
  &=
  \frac{\phi(\Phi^{-1}(\alpha_{\mathrm{lo}}))}{\alpha_{\mathrm{lo}}}.
\end{aligned}
\end{equation}

In our locomotion experiments, this batch mixed aggregation can yield unstable updates and, in some runs, prevent convergence.
We therefore compute the statistics per state--action pair first and then average across the batch:
\begin{equation}
  \widehat{F}_k^{\text{cond}}
  =
  \mathrm{Mean}_{(s,a)\in\mathcal{B}}
  \Big(
    \mathrm{Mean}_m[z_k^{(m)}(s,a)]
    +
\sigma(\alpha_{\mathrm{lo}})\,\mathrm{Std}_m[z_k^{(m)}(s,a)]
  \Big).
\end{equation}
This modification reduces overly conservative behavior caused by batch heterogeneity by ensuring that the standard deviation term reflects within-sample distributional uncertainty rather than variation across states within the batch.

\section{Supplementary Analysis and Results}
\label{sec:appendix:supp-results}

\subsection{Multi-constraint update behaviors of existing Primal-dual algorithms}
\label{sec:appendix:projection-cases}

We summarize multi-constraint update behaviors in representative safe policy optimization algorithms by grouping methods by how they incorporate multiple constraints into a single actor step:
(i) \emph{sequential constraint projection} onto linearized halfspaces and
(ii) \emph{scalarized constraint handling} via a weighted aggregation of constraint signals. We additionally discuss practical \emph{hyperparameter tuning difficulties} that arise in multi-constraint settings.

\subsubsection{Sequential constraint projection}

A common projection-based approach is to first compute a reward-improving step and then enforce constraint satisfaction by projecting it onto a linearized constraint halfspace under a chosen metric. A representative instance is PCPO~\cite{yang2020projection}, which in the single-constraint case solves the projection
\begin{equation}
  \label{eq:app:pcpo-proj-prob}
  \Delta\theta^{+}
  =
  \arg\min_{\Delta\theta\in\mathbb{R}^n}
  \frac12 (\Delta\theta-\Delta\theta^{(0)})^\top L\,(\Delta\theta-\Delta\theta^{(0)})
  \quad\text{s.t.}\quad
  g_c^\top \Delta\theta \le \Delta,
\end{equation}
where $L$ is the projection metric.
PCPO considers $L=I$ for $\ell_2$ projection and $L=F$ for Fisher-metric projection.
The solution is the closed-form correction along $L^{-1}g_c$:
\begin{equation}
  \label{eq:app:pcpo-closed-form}
  \Delta\theta^{+}
  =
  \Delta\theta^{(0)} - \alpha\,L^{-1}g_c,
  \qquad
  \alpha
  =
  \max\!\left\{0,\;
    \frac{g_c^\top \Delta\theta^{(0)}-\Delta}{g_c^\top L^{-1}g_c}
  \right\}.
\end{equation}

For $L=F$, define $u_k=\tilde g_{c_k}/\normF{\tilde g_{c_k}}$ and $\mu_k=\Delta_k/\normF{\tilde g_{c_k}}$.
The Fisher-metric projection onto the halfspace $\ipF{u_k}{\Delta\theta}\le\mu_k$ is
\begin{equation}
  \label{eq:app:fisher-proj-operator}
  P_k(x)
  := \arg\min_{y:\,\ipF{u_k}{y}\le\mu_k}\frac12\normF{y-x}^2
  = x - \big[\ipF{u_k}{x}-\mu_k\big]_+\,u_k,
\end{equation}
where $[t]_+:=\max\{t,0\}$.

As discussed in PCPO~\cite{yang2020projection}, in the multi-constraint setting PCPO can apply these halfspace projections sequentially across constraints. However, this sequential projection procedure can exhibit order dependence and step shrink when constraint normals are correlated.
\Cref{fig:locu_vs_pcpo} visualizes the local model and contrasts LOCU with sequential-projection updates.
The left panel corresponds to LOCU's single reduced-space solve.
The right panel depicts sequential halfspace projections, where order dependence is shown in
\eqref{eq:app:order-ij}--\eqref{eq:app:order-ji} and step shrink is quantified by \eqref{eq:app:norm-shrink}.

\begin{figure*}[t]
  \centering
  \includegraphics[width=\textwidth]{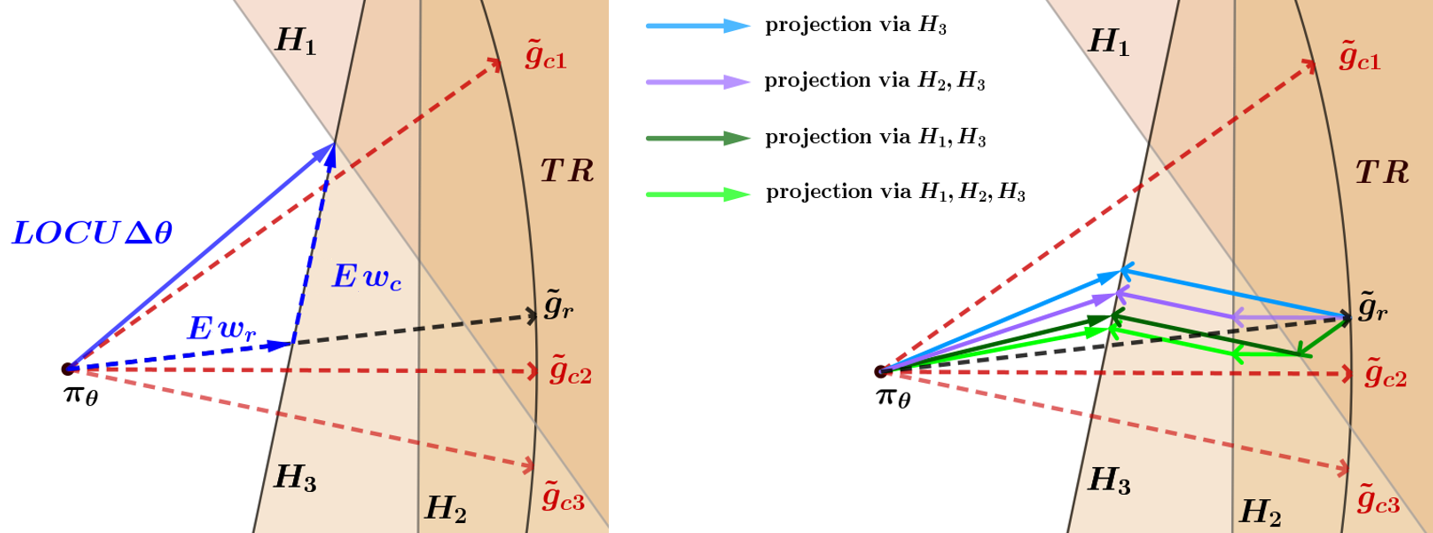}
  \caption{Geometric comparison between LOCU and PCPO updates.
    LOCU solves a reduced QP and produces an order-invariant update under the local model.
  PCPO applies sequential projections onto constraint halfspaces, which can produce order-dependent update directions.}
  \label{fig:locu_vs_pcpo}
\end{figure*}

Consider Fisher halfspaces $\mathcal{H}_k := \{\Delta\theta : \ipF{u_k}{\Delta\theta} \le \mu_k\}$.
Starting from the reward-only step $\Delta\theta^{(0)}$, a sequential projection update applies
\begin{equation}
  \label{eq:app:seq-proj}
  \Delta\theta^{(t)}
  =
  P_{\sigma_t}\big(\Delta\theta^{(t-1)}\big),
  \qquad
  t=1,2,\dots,m,
\end{equation}
for some ordering $\sigma$.

For two constraints $i$ and $j$, define
\[
  s_i := \ipF{u_i}{x},\quad s_j := \ipF{u_j}{x},\quad
  t_i := [s_i-\mu_i]_+,\quad t_j := [s_j-\mu_j]_+,
  \quad
  \rho_{ij}:=\ipF{u_i}{u_j}.
\]
Then
\begin{align}
  \label{eq:app:order-ij}
  P_j(P_i(x))
  &=
  x - t_i u_i - \Big[\,s_j - \rho_{ij}t_i - \mu_j\,\Big]_+ u_j,
  \\
  \label{eq:app:order-ji}
  P_i(P_j(x))
  &=
  x - t_j u_j - \Big[\,s_i - \rho_{ij}t_j - \mu_i\,\Big]_+ u_i.
\end{align}
If both constraints are active and $\rho_{ij}\neq 0$, the correction magnitudes differ and the final update depends on the order.

Sequential projections also reduce the Fisher step length.
Let $t_k := [\ipF{u_k}{x}-\mu_k]_+$.
Using $\normF{u_k}=1$,
\begin{equation}
  \label{eq:app:norm-shrink}
  \normF{P_k(x)}^2
  =
  \normF{x}^2 - t_k\,(t_k+2\mu_k).
\end{equation}
When $\mu_k \ge 0$, any active projection $t_k>0$ strictly decreases $\normF{\cdot}$, so $\normF{\Delta\theta^{(t)}}$ is non-increasing along \eqref{eq:app:seq-proj}.

The reward improvement in the local model is attenuated in proportion to Fisher correlations.
Let $u_r=\tilde g_r/\normF{\tilde g_r}$.
Then
\begin{equation}
  \label{eq:app:reward-shrink}
  \ipF{u_r}{P_k(x)}
  =
  \ipF{u_r}{x} - t_k\,\ipF{u_r}{u_k}.
\end{equation}
If $\ipF{u_r}{u_k}$ is not small across multiple active constraints, sequential projections repeatedly subtract reward-aligned components and can produce overly conservative updates.

\subsubsection{Scalarized constraint handling}
Another prevalent design is to reduce a vector-valued constraint problem to a \emph{single scalar} surrogate by aggregating multiple constraint signals through a weighted sum (typically via dual variables).
This avoids explicit geometric reasoning about how constraint normals interact; instead, interactions appear only through the aggregate scalar and the resulting single update direction.

\paragraph{Constrained Update Projection (CUP)} CUP~\cite{yang2022constrained} explicitly separates an intermediate performance-improvement
iterate $\pi_{\theta_{k+\frac12}}$ from a \emph{projection-to-safety} step.
In the projection stage, CUP solves a constrained KL projection in policy space:
\begin{equation}
  \label{eq:app:cup-projection}
  \pi_{\theta_{k+1}}
  =
  \arg\min_{\pi_\theta\in\Pi_\theta}\;
 \bar D_{\mathrm{KL}}\!\left(\pi_{\theta_{k+\frac12}},\pi_\theta\right)
  \quad\text{s.t.}\quad
  \hat C_{\pi_{\theta_k}}(\pi_\theta,\beta_k)\le b,
\end{equation}
where $\hat C$ denotes a surrogate cost and $\beta_k$ is a penalty coefficient ensuring the theoretical safety bound.
The optimization is implemented via a primal--dual routine with Lagrange multipliers associated with constraints. In the multi-constraint case, multipliers are updated per-constraint based on their individual violations, and the resulting policy update is driven by the aggregated penalty direction induced by the current multiplier vector.

\paragraph{First Order Constraint Optimization on Policy Space (FOCOPS)}

FOCOPS first solves a constrained optimization in policy space with a KL trust region, obtaining a non-parametric optimal policy (shown in~\cite{zhang2020first}, Thm.~1):
\begin{equation}
  \label{eq:app:focops-nonparam}
  \pi^*(a\mid s)
  \propto
  \pi_{\theta_k}(a\mid s)\,
  \exp\!\left(\frac{1}{\lambda}\Big(A_R^{\pi_{\theta_k}}(s,a)-\sum_{k} \nu_k A_{C_k}^{\pi_{\theta_k}}(s,a)\Big)\right),
\end{equation}
where $\lambda>0$ and $\{\nu_k\ge 0\}$ are dual variables.
FOCOPS then projects back to the parametric class by minimizing KL:
\begin{equation}
  \label{eq:app:focops-param-proj}
  \theta_{k+1}
  =
  \arg\min_{\theta}\;
  \mathbb{E}_{s\sim d^{\pi_{\theta_k}}}\!\left[
    \mathrm{KL}\!\left(\pi^*(\cdot\mid s)\,\|\,\pi_\theta(\cdot\mid s)\right)
  \right].
\end{equation}

\paragraph{Overlap and cancellation under scalarization.}
Both CUP and FOCOPS rely on a scalarized aggregate of multiple constraints, e.g.,
the aggregated advantage in \eqref{eq:app:focops-nonparam}
\[
  \hat A(s,a)
  := A_R(s,a) - \sum_{k=1}^K \nu_k A_{C_k}(s,a),
\]
or, in parameter space, an aggregated penalty direction of the form
\[
  \tilde g_{\mathrm{eff}}
  =
  \tilde g_r - \sum_{k=1}^{K}\lambda_k \tilde g_{c_k}.
\]
In multi-constraint regimes, the behavior of the aggregate is governed by pairwise correlations among constraint
directions. In Fisher geometry, the penalty direction satisfies
\begin{equation}
  \label{eq:app:scalarized-interference}
  \normF{\sum_{k=1}^K \lambda_k \tilde g_{c_k}}^2
  =
  \sum_{i=1}^K \lambda_i^2 \normF{\tilde g_{c_i}}^2
  \;+\;
  2\sum_{1\le i<j\le K}\lambda_i\lambda_j\,\ipF{\tilde g_{c_i}}{\tilde g_{c_j}}.
\end{equation}
The cross terms encode \emph{interference}:
if $\ipF{\tilde g_{c_i}}{\tilde g_{c_j}}<0$ (opposing directions), the aggregate can exhibit partial cancellation,
weakening constraint correction even when individual constraints are violated;
if $\ipF{\tilde g_{c_i}}{\tilde g_{c_j}}>0$ (overlapping directions), the aggregate can over-emphasize a redundant
direction, effectively reducing the rank of meaningful constraint information and making the update sensitive to the
instantaneous multiplier vector.
Since scalarization does not explicitly model the constraint geometry, trade-offs among conflicting constraints are resolved
implicitly through multiplier weights rather than through an explicit reduced-space analysis of active constraints.

\subsubsection{Hyperparameter tuning difficulties} \label{sec:appendix:lagrangian}
\begin{table}[t]
  \centering
  \caption{PID coefficient candidate values used in our experiments.}
  \label{tab:pid_coeffs}
  \begin{tabular}{l l}
    \toprule
    Coefficient & Candidate values \\
    \midrule
    $K_P$ & $\{0.25,\,1\}$ \\
    $K_I$ & $\{10^{-4},\,10^{-3},\,10^{-2},\,10^{-1},\,1\}$ \\
    $K_D$ & $\{0.01,\,0.1,\,1\}$ \\
    \bottomrule
  \end{tabular}
\end{table}
Primal--dual Lagrangian methods handle constraints by introducing multipliers and optimizing
\begin{equation}
  \label{eq:app:lagrangian_form}
  \max_{\theta}\;\min_{\lambda\ge 0}\;
  J_R(\pi_\theta)\;-\;\sum_{k=1}^{K}\lambda_k\big(J_{C_k}(\pi_\theta)-b_k\big),
\end{equation}
so the actor update direction is $g_r-\sum_k \lambda_k g_{c_k}$.
In practice, the training behavior can be dominated by the dual feedback loop and becomes sensitive to the
initialization of $\lambda$ and the relative step sizes of primal and dual updates, which often leads to substantial hyperparameter tuning and unstable learning~\cite{paternain2019constrained,Kamran2022modern}.

PID Lagrangian replaces simple dual ascent with a PID-style controller over the constraint error
$e_k := J_{C_k}(\pi_\theta)-b_k$~\cite{stooke2020responsive},
\begin{equation}
  \label{eq:app:pid_update}
  \lambda_k \leftarrow
  \Big[\lambda_k + K_P e_k + K_I \sum_{\tau \le t} e_{k,\tau} + K_D (e_k - e_{k-1})\Big]_+,
\end{equation}
However, selecting  $K_P$, $K_I$, and $K_D$ remains a practical bottleneck.

\begin{figure}[t]
  \centering
  \includegraphics[width=\linewidth]{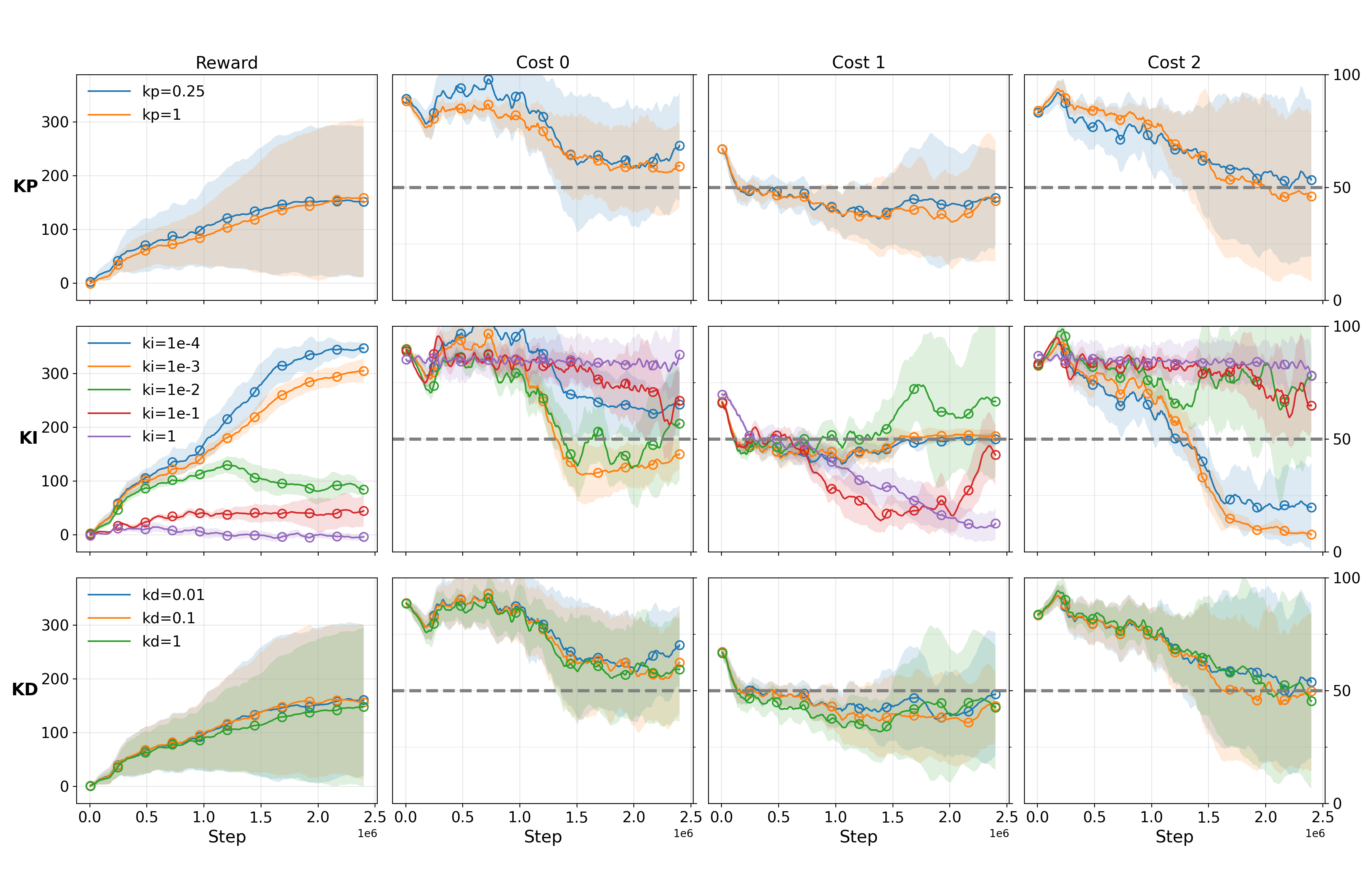}
  \caption{Learning-curve comparison of PID PPO-Lagrangian under P I D coefficient variations on the Bullet Safety Gym BC-v3 task. Dashed lines in the cost plots indicate the corresponding limits.}
  \label{fig:pid_gain_sweep}
\end{figure}

We evaluated PID PPO-Lagrangian with gradient shaping~\cite{yao2024gradient} under different PID coefficient choices. We used similar PID coefficient sets used for the coefficient learning-curve comparisons in Appendix~C of the PID-Lagrangian paper~\cite{stooke2020responsive}, summarized in Table~\ref{tab:pid_coeffs}. All combinations were evaluated with five random seeds. In Figure~\ref{fig:pid_gain_sweep}, each curve was obtained by averaging all runs that shared the coefficient value of interest over the remaining two PID coefficients and seeds, and the shaded region denotes the standard deviation over the same set of runs.

Gradient shaping mitigates cancellation and overlap among constraint gradients in the actor update~\cite{yao2024gradient}, yet sensitivity of the multiplier dynamics from \eqref{eq:app:pid_update} still remains a challenge. Across the tested settings, only a narrow regime such as $K_I=10^{-3}$ yields relatively stable boundary tracking in this task, suggesting that the robustness reported in single-constraint PID Lagrangian
analyses~\cite{stooke2020responsive} can be weakened when multiple constraints are coupled through a policy update.

\subsection{Perturbation sensitivity of Löwdin and Gram-Schmidt coefficients}
We compare the perturbation sensitivity of L\"owdin and Gram--Schmidt orthogonalization through theoretical analysis and empirical evaluation, focusing on how stochastic noise propagates through their respective orthogonalization procedures.
\subsubsection{Theoretical analysis}
\label{sec:appendix:lowdin-gs-noise}

Using the notation in Sections~\ref{sec:policy-gradients}, \ref{sec:fisher-geometry}, and
\ref{sec:method:lowdin}, let
$S=[\tilde g_r,\tilde g_{c_{i_1}},\ldots,\tilde g_{c_{i_m}}]$ and $O=S^\top F S$.
Suppose $O$ can have small eigenvalues due to near-collinear directions. Under stochastic gradient noise $\Xi$,
where $\widehat S:=S+\Xi$, the perturbed Fisher Gram matrix is $\widehat O:=\widehat S^\top F\widehat S$. Then
\begin{equation}
  \label{eq:app:deltaO}
  \Delta O
  :=
  \widehat O-O
  =
  S^\top F\Xi+\Xi^\top FS+\Xi^\top F\Xi .
\end{equation}
Its norm is bounded by
\begin{equation}
  \label{eq:app:deltaO-bound}
  \|\Delta O\|_2
  \le
  2\|F^{1/2}S\|_2\|F^{1/2}\Xi\|_2
  +
  \|F^{1/2}\Xi\|_2^2 .
\end{equation}
For LOCU, since Proposition~\ref{prop:coeff-gram} gives
$E^\top g_r=O^{1/2}e_1$ and $E^\top g_{c_{i_j}}=O^{1/2}e_{j+1}$, the sensitivity depends on the noise of the square root.

\paragraph{For L\"owdin Orthogonalization.}
Let $O=Q\Lambda Q^\top$, with $Q$ the eigenvector matrix so that $Q^\top(\cdot)Q$ represents a change to the
eigenbasis, and $\Lambda=\mathrm{diag}(\lambda_1,\ldots,\lambda_{m+1})$. Then, defining
$\Delta_L:=\widehat O^{1/2}-O^{1/2}$, the first-order term satisfies the Sylvester equation:
\begin{equation}
  \label{eq:app:lowdin-sylvester}
  O^{1/2}\Delta_L+\Delta_LO^{1/2}
  =
  \Delta O+\mathcal{O}(\|\Delta O\|_2^2).
\end{equation}
So in the eigenbasis of $O$,
\begin{equation}
  \label{eq:app:lowdin-eigenbasis}
  Q^\top\Delta_LQ
  =
  \left[
  \frac{(Q^\top\Delta OQ)_{ab}}{\sqrt{\lambda_a}+\sqrt{\lambda_b}}
  \right]_{a,b}
  +
  \mathcal{O}(\|\Delta O\|_2^2),
  \qquad
  a,b\in\{1,\ldots,m+1\}.
\end{equation}
where $a,b\in\{1,\ldots,m+1\}$ index the coordinates in that eigenbasis.

This shows the symmetric structure of L\"owdin. The sensitivity to near-collinearity depends on the denominator in
\eqref{eq:app:lowdin-eigenbasis}, so the effect of the noise is distributed across the spectrum of the whole retained block.

\paragraph{For Gram--Schmidt Orthogonalization.}
Let $O=T^\top T$ and $T=\mathrm{chol}(O)$, the Cholesky factorization of $O$; this is equivalent to Fisher
Gram--Schmidt applied to the columns of $S$, where $T$ collects the resulting projection and normalization
coefficients, so the corresponding reduced coefficients are encoded by the upper-triangular factor $T$.
If $o_{ab}$ and $t_{ab}$ denote the entries of $O$ and $T$, then the Cholesky/Gram--Schmidt recursion is
\begin{equation}
  \label{eq:app:gs-recursion}
  t_{kk}^2
  =
  o_{kk}
  -
  \sum_{s<k}t_{sk}^2,
  \qquad
  t_{kj}
  =
  \frac{
    o_{kj}
    -
    \sum_{s<k}t_{sk}t_{sj}
  }{
    t_{kk}
  },
  \qquad
  j>k .
\end{equation}
Using the first-order Taylor expansion of these relations gives
\begin{equation}
  \label{eq:app:gs-pert-diag}
  \Delta t_{kk}
  =
  \frac{
    \Delta o_{kk}
    -
    2\sum_{s<k}t_{sk}\Delta t_{sk}
  }{
    2t_{kk}
  }
  +
  \mathcal{O}(\|\Delta O\|_2^2),
\end{equation}
and
\begin{equation}
  \label{eq:app:gs-pert-offdiag}
  \Delta t_{kj}
  =
  \frac{
    \Delta o_{kj}
    -
    \sum_{s<k}
    \left(
      \Delta t_{sk}t_{sj}
      +
      t_{sk}\Delta t_{sj}
    \right)
    -
    t_{kj}\Delta t_{kk}
  }{
    t_{kk}
  }
  +
  \mathcal{O}(\|\Delta O\|_2^2),
  \qquad
  j>k .
\end{equation}
This shows the sequential noise structure of Gram--Schmidt: the noise at stage $k$ is divided by the current pivot
$t_{kk}$, and later quantities inherit earlier noise through the recursion.

\paragraph{Comparison in update.}
Fundamentally, both orthogonalization methods suffer under severe collinearity. LOCU addresses this via pairwise
screening and compression, $|\varrho_{ij}|\ge 1-\varepsilon_{\mathrm{dup}}$. However, this cannot fully prevent
span-level near-collinearity, where a retained vector can align with the span of others much more tightly than the
threshold $\varepsilon_{\mathrm{dup}}$.

In worst-case scenarios, the Gram matrix $O$ becomes near-singular, destabilizing both methods. For L\"owdin, this
forces the minimum eigenvalue $\lambda_{\min}(O)\approx0$, driving the denominator in
\eqref{eq:app:lowdin-eigenbasis} to approach zero. For Gram--Schmidt, the instability manifests in the sequential
pivots. Rewriting \eqref{eq:app:gs-recursion} in Schur-complement form yields
\begin{equation}
  \label{eq:app:gs-schur-pivot}
  t_{kk}^2
  =
  o_{kk}
  -
  O_{1:k-1,k}^\top
  O_{1:k-1,1:k-1}^{-1}
  O_{1:k-1,k}
  =
  o_{kk}\bigl(1-\rho_{k\mid 1:k-1}^2\bigr).
\end{equation}
Thus, span-level near-collinearity, $\rho_{k\mid 1:k-1}^2\approx1$, drives the normalized residual
$t_{kk}^2/o_{kk}$ to near zero.

In such cases, both methods require regularization/damping to avoid division by near zero. However, their noise and
bias propagation differ fundamentally as shown earlier. In Gram--Schmidt, the sequential structure causes both
stochastic noise and the damping bias introduced at step $k$ to be inherited and compounded by all subsequent steps
$j>k$, leading to cumulative distortion in the resulting basis geometry. Conversely, L\"owdin applies symmetric
matrix-level regularization, e.g., $O+\varepsilon I$, which distributes both perturbation and damping effects across
the entire eigenbasis without order-dependent accumulation. By treating all vectors simultaneously, L\"owdin preserves
the mutual symmetry and geometric integrity of the basis far more effectively than the inherently hierarchical
Gram--Schmidt approach.

\subsubsection{Empirical comparison}

We evaluated LOCU on the BC-v3 Redun : 3 task by replacing Löwdin orthogonalization with Gram--Schmidt orthogonalization under three gradient orderings (original, random, and reverse), while injecting observation noise into each training batch following \cite{liu2024offlinesaferl}.

As shown in Figure~\ref{fig:lowdin-gs-noise}, L\"owdin orthogonalization generally performs below Gram--Schmidt index but above random and reverse, attributed to its symmetric distribution of perturbation-induced errors across gradient directions. In this task, the last three constraints in the indexed order are the loosest low-velocity constraints, and the differences among the Gram--Schmidt variants reflect their inherent ordering dependence. Under the indexed order, accumulated errors primarily propagate to these loose constraints, whereas under the reverse order, they affect the reward gradient and the tighter high-velocity constraints. The noise implementation of~\cite{liu2024offlinesaferl} adds perturbations of a fixed magnitude without accounting the relative scale of each observation. At the highest noise level of 0.25, all methods degrade substantially, while the pronounced failure of Gram--Schmidt-random suggests that Gram--Schmidt orthogonalization can be more vulnerable than L\"owdin when large stochastic perturbations distort the gradient geometry.

\begin{figure}[!t]
    \centering
    \includegraphics[width=\linewidth]{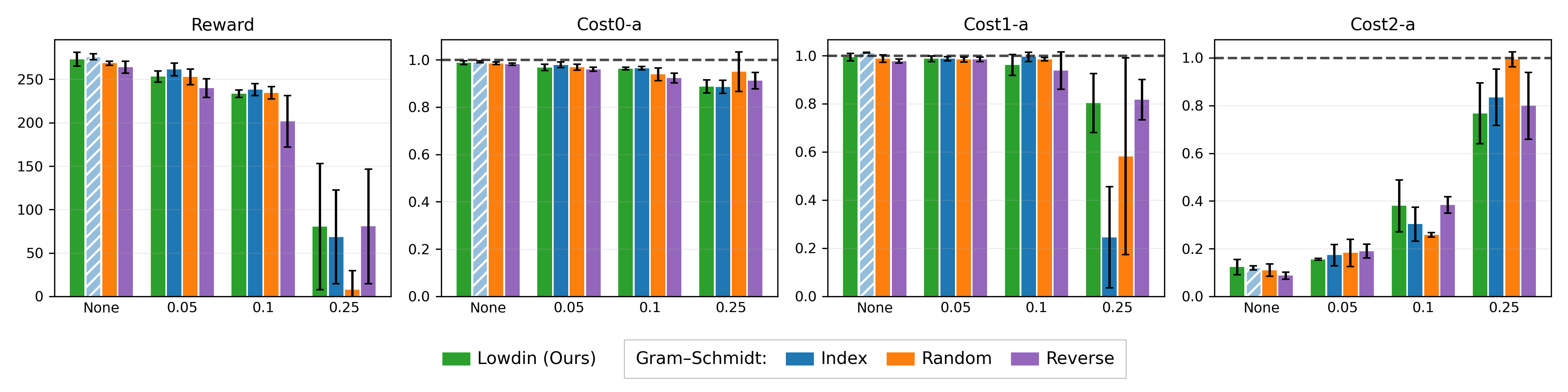}
    \caption{Performance comparison of L\"owdin and Gram--Schmidt orthogonalization under observation noise on BC-v3 Redun 3.}
    \label{fig:lowdin-gs-noise}
\end{figure}

\subsection{Computational Complexity}\label{app:computational}

\subsubsection{Wall clock time results}\label{app:computational cost}
Table~\ref{tab:wallclock} reports wall clock training time on BC-v2 and redun 3 variants of BC-v3, measured on a single machine with an NVIDIA RTX 4070 GPU, an Intel i7-13700K CPU, and 64GB RAM.
PPOL-V and PPOL-GS are excluded from the redundancy results in the main text due to unstable learning in earlier stages, but we additionally trained them solely to compare wall clock time and include their times in the table. Wall clock time is logged with Weights and Biases~\cite{wandb}.

\begin{table}[H]
  \centering
  \caption{Wall-clock time on the main circle task.}
  \label{tab:wallclock}
  \setlength{\tabcolsep}{4pt}
  \renewcommand{\arraystretch}{1.08}
  \small
  \begin{tabularx}{\textwidth}{lYYYYYYYY}
    \toprule
    Task
    & LOCU
    & SDAC(A)
    & PCRPO
    & CRPO
    & \makecell{TRPOL\\-GS}
    & \makecell{TRPOL\\-V}
    & \makecell{PPOL\\-V}
    & \makecell{PPOL\\-GS} \\
    \midrule
    BC v2
    & 17:08 & 18:20 & 17:30 & 17:23 & 16:08 & 16:27 & 17:11 & 18:29 \\
    BC v3 Redun 3
    & 25:02 & 25:40 & 23:00 & 22:53 & 17:02 & 17:05 & 22:28 & 23:40 \\
    \bottomrule
  \end{tabularx}
\end{table}

\begin{table}[H]
  \centering
  \caption{Wall-clock time on Circle tasks with parallel processing.}
  \label{tab:wallclock_circle2}
  \setlength{\tabcolsep}{4pt}
  \renewcommand{\arraystretch}{1.08}
  \small
  \begin{tabularx}{\textwidth}{llYYYYYY}
    \toprule
    Task
    & Redun
    & LOCU
    & SDAC(A)
    & CRPO
    & TRPOL
    & \makecell{PPOL\\(bs=64)}
    & \makecell{PPOL\\(bs=128)} \\
    \midrule
    \multirow{3}{*}{PointCircle}
    & 2 & 10:19 & 10:55 & 10:18 & 10:22 & 21:08 & 15:37 \\
    & 3 & 10:31 & 11:08 & 10:19 & 10:25 & 21:13 & 15:40 \\
    & 5 & 10:44 & 11:19 & 10:22 & 10:29 & 21:20 & 15:45 \\
    \bottomrule
  \end{tabularx}
\end{table}

With implementation for main text experiment, LOCU computes natural-gradient directions for all constraints through CG then solves a QP, so its cost can grow with the number of constraints. However, this overhead is not always dominant. In BC-v2, where only a small number of constraints are used, PPO-style baselines often spend more time on repeated mini-batch actor optimization and multiple constraint-critic updates, which can outweigh CG solves. 

In BC-v3 Redun 3, where the number of constraints increases to nine, wall-clock time increases for most methods, with the largest increases observed for LOCU and SDAC(A). This reflects the cumulative cost of computing multiple constraint natural-gradient directions before the constrained update. The comparison should also be interpreted with implementation details: SDAC(A) uses OSQP and SciPy backends for its QP~\cite{osqp,2020SciPy-NMeth}, whereas LOCU implementation are in Python-level.

\subsubsection{Parallel processing}
To examine whether this bottleneck can be reduced, we construct a MuJoCo-JAX (MJX)~\cite{todorov2012mujoco} version of the CircleRun-v3 Redun task and implement LOCU, SDAC(A), CRPO, PPOLag, and TRPOLag with vectorized constraint processing and end-to-end XLA compilation. This targets a common limitation of Python-based safe RL implementations, where most computations are executed sequentially and therefore use only part of the available hardware throughput. Table~\ref{tab:wallclock_circle2} reports the wall-clock time of the MJX implementation on PointCircle with redundancy levels 2, 3, and 5. The results show that full-batch second-order actor updates can be faster than minibatch-based first-order updates with parallel processing. LOCU, SDAC(A), CRPO, and TRPOL complete each run in roughly 10--11 minutes, whereas PPOLag takes 15--21 minutes depending on the minibatch size, denoted as bs in the table. 

As expected, LOCU and SDAC(A), which compute natural-gradient directions for all constraints, show a larger redundancy-dependent increase than CRPO and TRPOL, which only compute single natural-gradient. Under the MJX implementation, however, this increase is small: from redundancy level 2 to 5, LOCU increases from 10:19 to 10:44 and SDAC(A) from 10:55 to 11:19, i.e., less than 30 seconds while the number of constraints grows from 6 to 15. By contrast, in the Python-based implementation in Table~\ref{tab:wallclock}, increasing from BC-v2 with two constraints to BC-v3 Redun 3 with nine constraints increases LOCU and SDAC(A) by roughly 7--8 minutes. Compared with the substantially larger cost of PPOLag, the remaining difference between LOCU and the single-natural-gradient methods can be considered minor.

Overall, the results indicate that parallelized constraint processing substantially reduces wall-clock sensitivity to the number of constraints compared with sequential Python-based implementations. The same implementation strategy may therefore be useful for other multi-constraint safe RL methods or multi-objective settings. More broadly, the results weaken the view that second-order safe RL methods are impractical solely due to computational burden. When multiple constraints must be handled and constraint satisfaction is central to the task, LOCU or other trust-region-based second-order methods can remain practical options under parallel processing.

\subsubsection{Inverse fisher preconditioner}
\label{app:naturalgradient cost}

Natural gradients are typically computed using conjugate gradient (CG), which needs to be applied separately to each gradient, the computational cost increases linearly with the number of constraints.
An alternative option can be using a shared approximate Fisher inverse or preconditioner across all gradients. Since the FIM is defined solely by the outer product of policy score functions, independent of advantage estimates, layer-wise curvature approximations such as K-FAC and its variants can provide a single approximate inverse applicable to all reward and constraint gradients. Although these approaches can be more computationally efficient than using CG iterations depending on the network architecture and the number of constraints, further investigation is required to assess its approximation accuracy, numerical stability, and overall effectiveness in multi-constraint policy optimization.

\subsection{Additional Experiments}\label{app:additional experiments}
We conduct additional experiments to evaluate LOCU from three perspectives. First, we examine the contribution of its algorithmic components through ablation studies. Second, we investigate hyperparameter sensitivity by varying the clustering threshold and KL limit. Finally, we evaluate robustness to redundant constraints by increasing the number of constraints while preserving the feasible region.
\subsubsection{Component Analysis}\label{app:ablation}
We compare LOCU against two ablated variants and a CPO baseline.
The two ablated variants in Figure~\ref{fig:ablation:components} are labeled as \texttt{w/o Lowdin} and \texttt{w/o Lowdin, Screening \& compression}.
\texttt{w/o Lowdin} removes L\"owdin orthogonalization and only normalizes each direction by its Fisher norm before applying the same reduced update.
\texttt{w/o Lowdin, Screening \& compression} further removes screening and clustering and uses all normalized constraint directions in the same update procedure.
For the CPO baseline, we follow the original implementation where infeasibility and degenerate cases use a recovery direction formed by summing the gradients of currently violated constraints, and the resulting QP is solved with OSQP~\cite{osqp}.
Removing L\"owdin slows convergence relative to CPO and only recovers return late in training.
When both L\"owdin and screening with compression are removed, the early stage becomes highly seed dependent under simultaneous violations, and some runs finish with less than half of LOCU's final return.
This is consistent with redundant and correlated constraint directions producing ill conditioned reduced systems and noisy recovery when they are treated as distinct correction directions without orthogonalization or reduction.

\begin{figure*}[!t]
  \centering
  \includegraphics[width=\textwidth]{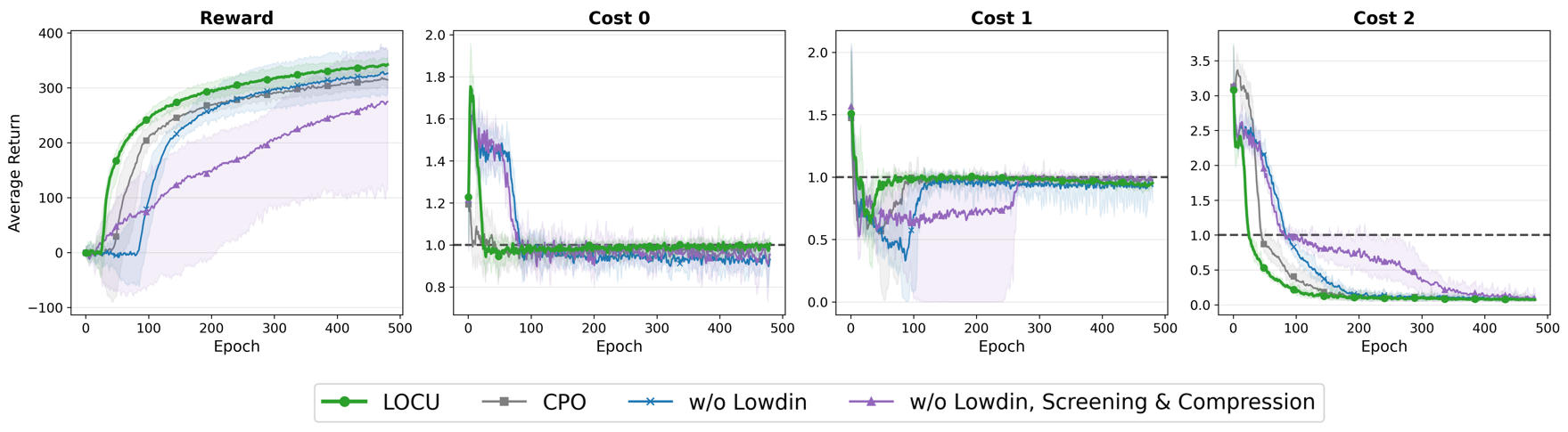}
  \caption{
    Learning curves for \texttt{w/o Lowdin} and \texttt{w/o Lowdin, Screening \& compression}, compared with the CPO baseline.
  }
  \label{fig:ablation:components}
\end{figure*}

\subsubsection{Hyperparameter Sensitivity}
We test the cosine similarity threshold $\varepsilon_{\mathrm{dup}}$ that controls near-collinear clustering in Section~\ref{sec:method:screening}.
Since $\varepsilon_{\mathrm{dup}}$ is used to decide whether constraint gradients are treated as redundant, we evaluate it on BC-v3 Redun : 3, where redundancy is strongest. Figure~\ref{fig:ablation:epsdup} compares threshold values 0.05, 0.1, and 0.2. Differences appear mainly in the early recovery stage, while final returns remain very similar across all tested values. These gaps are small relative to typical across-seed variation and may narrow with additional seeds. Normalized costs converge below the limits in all cases, indicating that LOCU is not particularly sensitive to the exact choice of $\varepsilon_{\mathrm{dup}}$ within this range and does not require careful angle tuning to avoid ill-conditioned updates.
\begin{figure*}[!t]
  \centering
  \includegraphics[width=\textwidth]{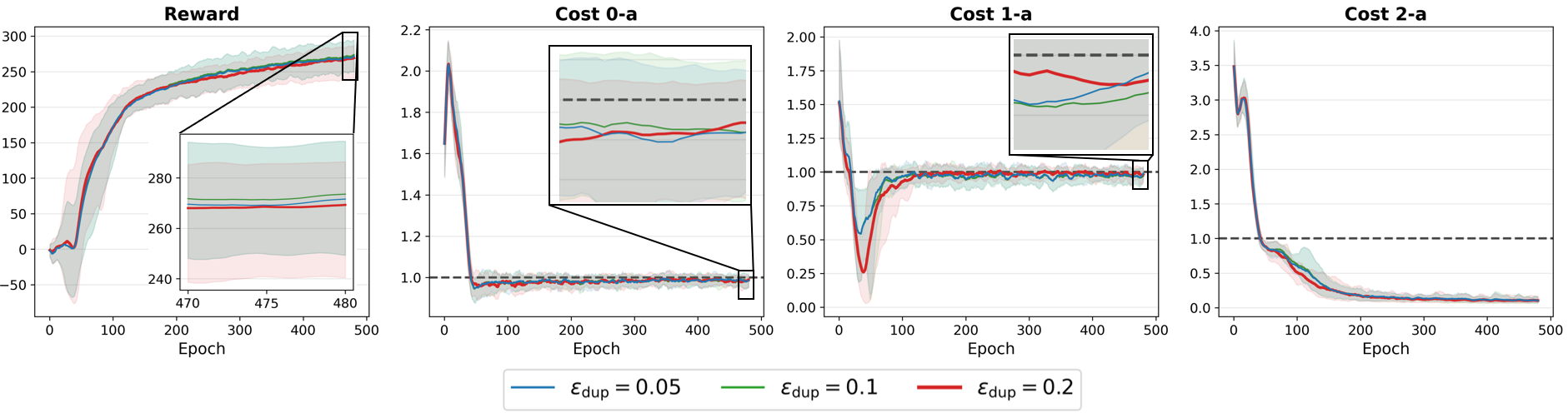}
  \caption{
    Ablation on the clustering threshold $\varepsilon_{\mathrm{dup}}$ for near-collinear constraint grouping on BC-v3 Redun : 3.
  }
  \label{fig:ablation:epsdup}
\end{figure*}

Next, we evaluate sensitivity to the trust region radius against SDAC(A) on the BC-v3 task. As the limit size decreases, a smaller KL limit generally reduces approximation error but yields less policy improvement per update. Since our task starts from an infeasible state, a smaller radius requires more updates to reach the feasible region, leaving fewer remaining updates for reward optimization. Both algorithms exhibit similar performance degradation in reward, while showing no further sensitivity in cost satisfaction, as shown in Figure~\ref{fig:kl:sensitivity}.

\begin{figure*}[!t]
\centering
\includegraphics[width=\textwidth]{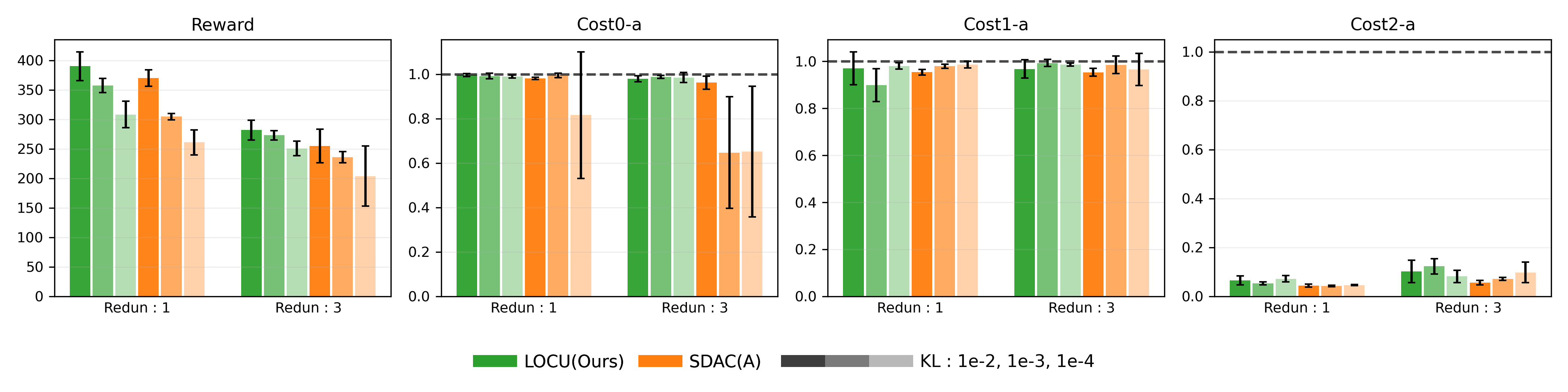}
\caption{
Sensitivity analysis on the KL divergence limit compared against SDAC(A) on BC-v3.
}
\label{fig:kl:sensitivity}
\end{figure*}

Finally, we examine numerical hyperparameter sensitivity on the BC-v3 Redun : 3 task by simultaneously assigning $10^{-4}$ and $10^{-6}$ to the near-zero threshold, Fisher damping, and CG tolerance. The baseline setup uses default values $10^{-2}$ for the near-zero threshold, $10^{-1}$ for Fisher damping, and $10^{-8}$ for CG tolerance. As illustrated in Figure~\ref{fig:numerical:sensitivity}, smaller sweep values yield a small increase in reward while causing larger variance across runs. This behavior is attributed to reduced Fisher damping, which produces a closer approximation to the true natural gradient but increases sensitivity to near-zero eigenvalues of the Fisher matrix, making individual updates more accurate yet overall training less stable.

\begin{figure*}[!t]
\centering
\includegraphics[width=\textwidth]{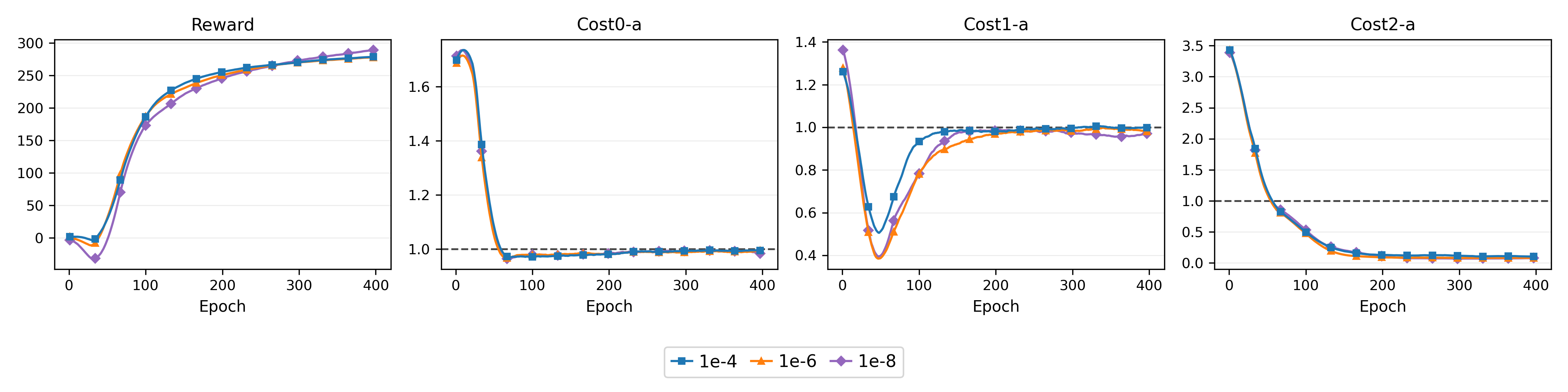}
\caption{
Ablation on threshold, damping, and CG tolerance values on BC-v3 Redun : 3.
}
\label{fig:numerical:sensitivity}
\end{figure*}

\subsubsection{Robustness to Redundant Constraints}
\label{sec:appendix:robustness-test}

To evaluate robustness to a larger number of constraints without changing the feasible region, we introduce redundant constraint variants to the goal-object Cassie task.
For each base constraint, we keep the original tight threshold and add two looser variants, increasing the total number of constraints from 6 to 18.
For example, we use thresholds of \(15^\circ/20^\circ/25^\circ\) for body tilt, \(13/15/18\,\mathrm{W}\) for energy, and \(0.0/0.05/0.1\,\mathrm{m}\) for the hazard contact margin.
This creates a setting with substantial redundancy and heterogeneity across constraint gradients.

Figure~\ref{fig:robustness-redun-vs-base-eval} compares the final evaluation results between the original 6-constraint Cassie task and the redundant 18-constraint version.
LOCU preserves reward under \(3\times\) constraint scaling while satisfying all constraints at convergence.
Figure~\ref{fig:robustness-redun-vs-base-learning} further shows that the learning dynamics remain stable after adding redundant constraints, indicating that the additional constraints do not substantially degrade optimization.

\begin{figure}[!t]
  \centering
  \includegraphics[width=0.92\textwidth]{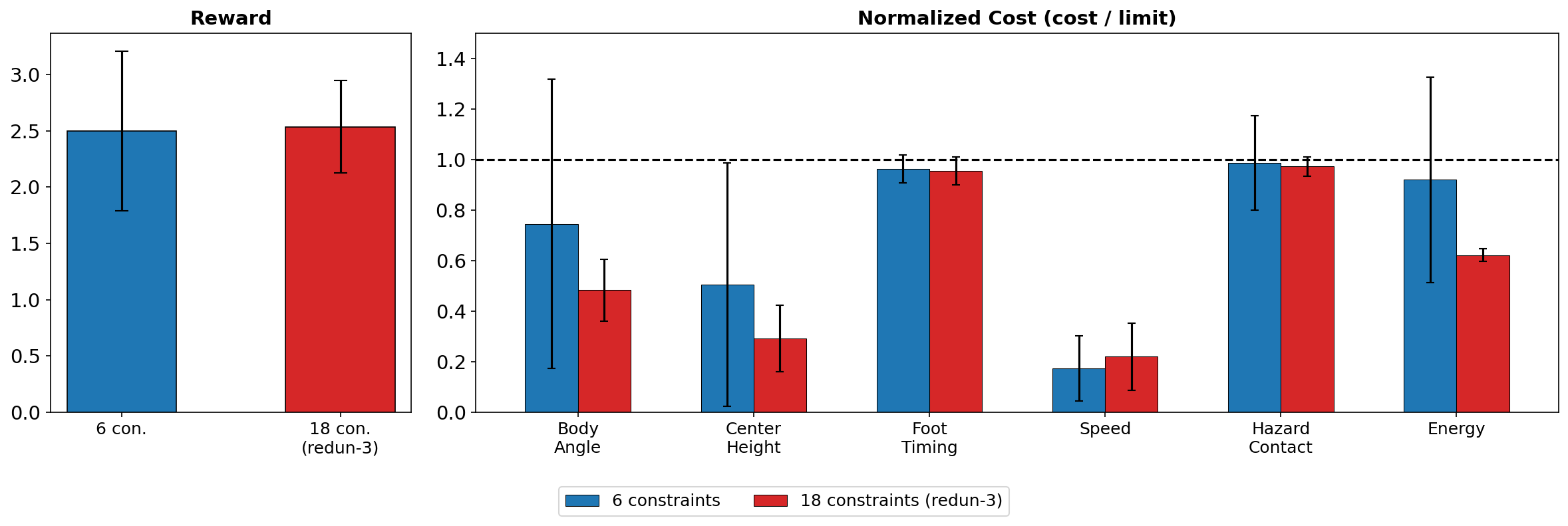}
  \caption{Final evaluation comparison between the original 6-constraint Cassie task and the redundant 18-constraint version.}
  \label{fig:robustness-redun-vs-base-eval}
\end{figure}

Table~\ref{tab:robustness-screening-statistics} reports the screening statistics during training in the 18-constraint setting.
Even with 18 nominal constraints, only a small subset is active at each update, and redundancy compression further reduces the basis dimension.
By the end of training, the average active set size is 4.9 and the average basis dimension is 2.8, so the practical optimization remains compact.

\begin{table}[!t]
  \centering
  \caption{Screening statistics during training in the 18-constraint redundant Cassie task.}
  \label{tab:robustness-screening-statistics}
  \setlength{\tabcolsep}{4pt}
  \renewcommand{\arraystretch}{1.08}
  \small
  \begin{tabularx}{\textwidth}{lYYYYY}
    \toprule
    Phase & Basis dim. & Active set & Redun. & Inactive & Zero-grad \\
    \midrule
    \(0\text{--}20\%\) & \(5.7 \pm 0.9\) & \(9.5 \pm 1.6\) & \(3.8 \pm 1.0\) & \(8.5 \pm 1.6\) & \(2.5 \pm 1.2\) \\
    \(20\text{--}40\%\) & \(3.8 \pm 1.5\) & \(6.5 \pm 3.1\) & \(2.7 \pm 1.6\) & \(11.5 \pm 3.1\) & \(3.2 \pm 1.7\) \\
    \(40\text{--}60\%\) & \(3.9 \pm 1.7\) & \(6.6 \pm 4.0\) & \(2.7 \pm 1.8\) & \(11.4 \pm 4.0\) & \(1.6 \pm 1.1\) \\
    \(60\text{--}80\%\) & \(3.1 \pm 1.5\) & \(5.5 \pm 4.2\) & \(2.4 \pm 1.7\) & \(12.5 \pm 4.2\) & \(1.2 \pm 0.5\) \\
    \(80\text{--}100\%\) & \(2.8 \pm 1.4\) & \(4.9 \pm 4.0\) & \(2.1 \pm 1.5\) & \(13.1 \pm 4.0\) & \(0.9 \pm 0.4\) \\
    \bottomrule
  \end{tabularx}
\end{table}

\subsection{Learning curves}
\label{app:learning_curves}
We report learning curves for the experiments in the main text and Appendix.
Costs are normalized by their corresponding limits, so values above 1 indicate constraint violations.
\begin{figure}[H]
  \centering
  \includegraphics[width=0.7\textwidth]{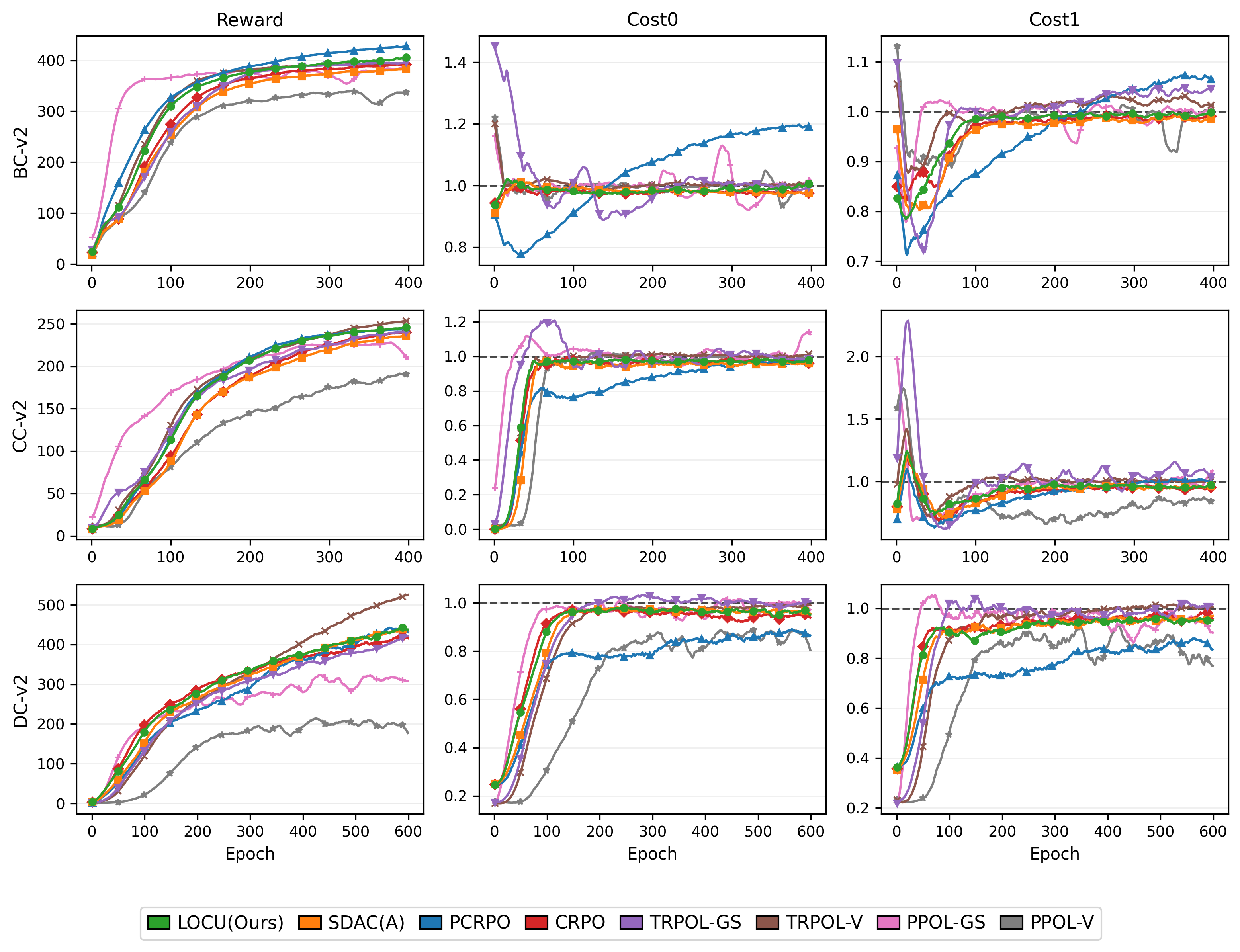}
  \caption{
    Learning curves on the Bullet Safety Gym v2 suite for BC, CC, and DC.
  }
  \label{fig:lc:v2}
\end{figure}
\begin{figure}[H]
  \centering
  \includegraphics[width=\textwidth]{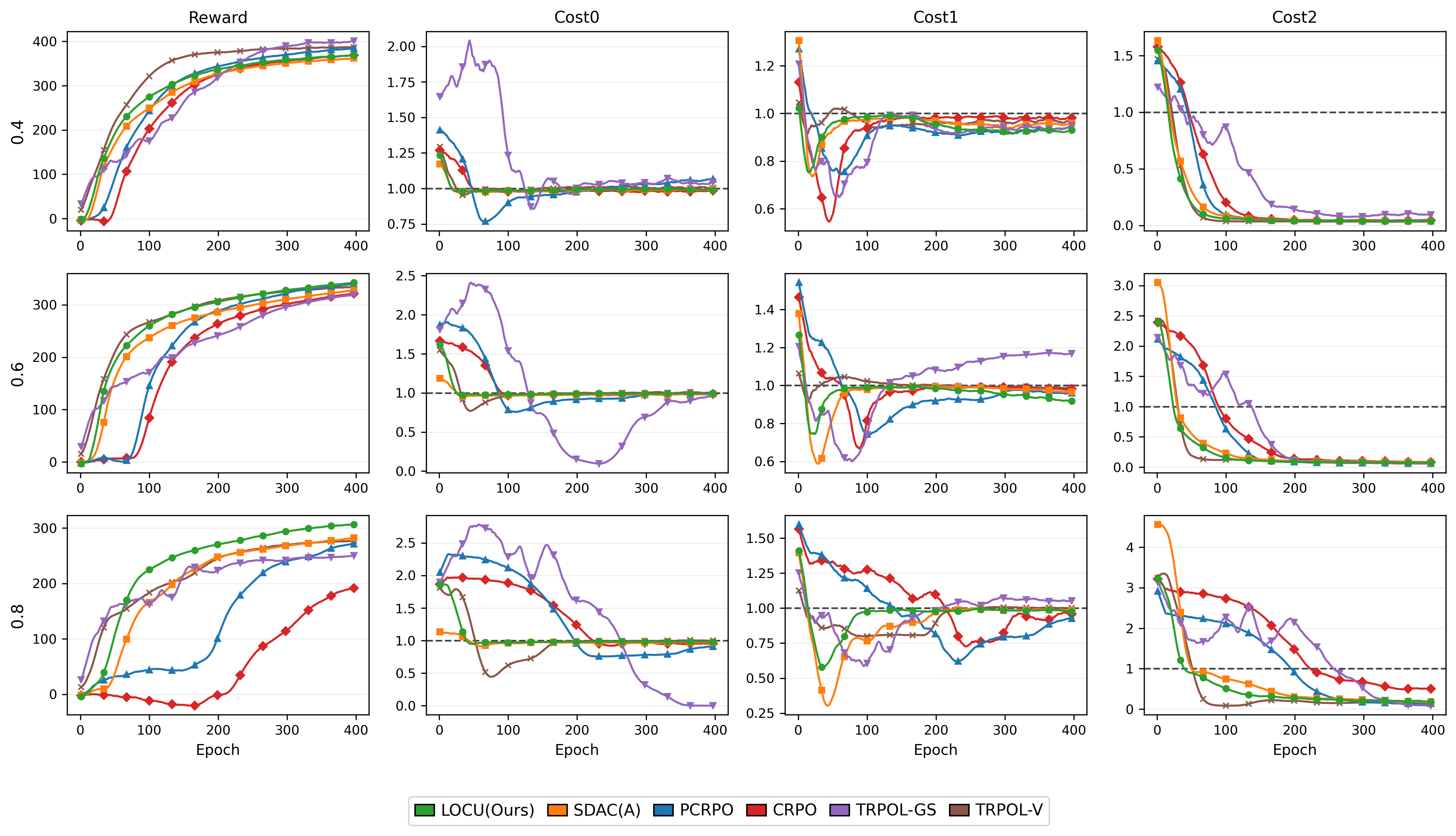}
  \caption{
    Learning curves on BC-v3 under the low velocity sweep from the v2 case to $\kappa_{\mathrm{lo}}\in\{0.4,0.6,0.8\}$.
  }
  \label{fig:lc:bcv3-low}

\end{figure}
\begin{figure}[H]
  \centering
  \includegraphics[width=\textwidth]{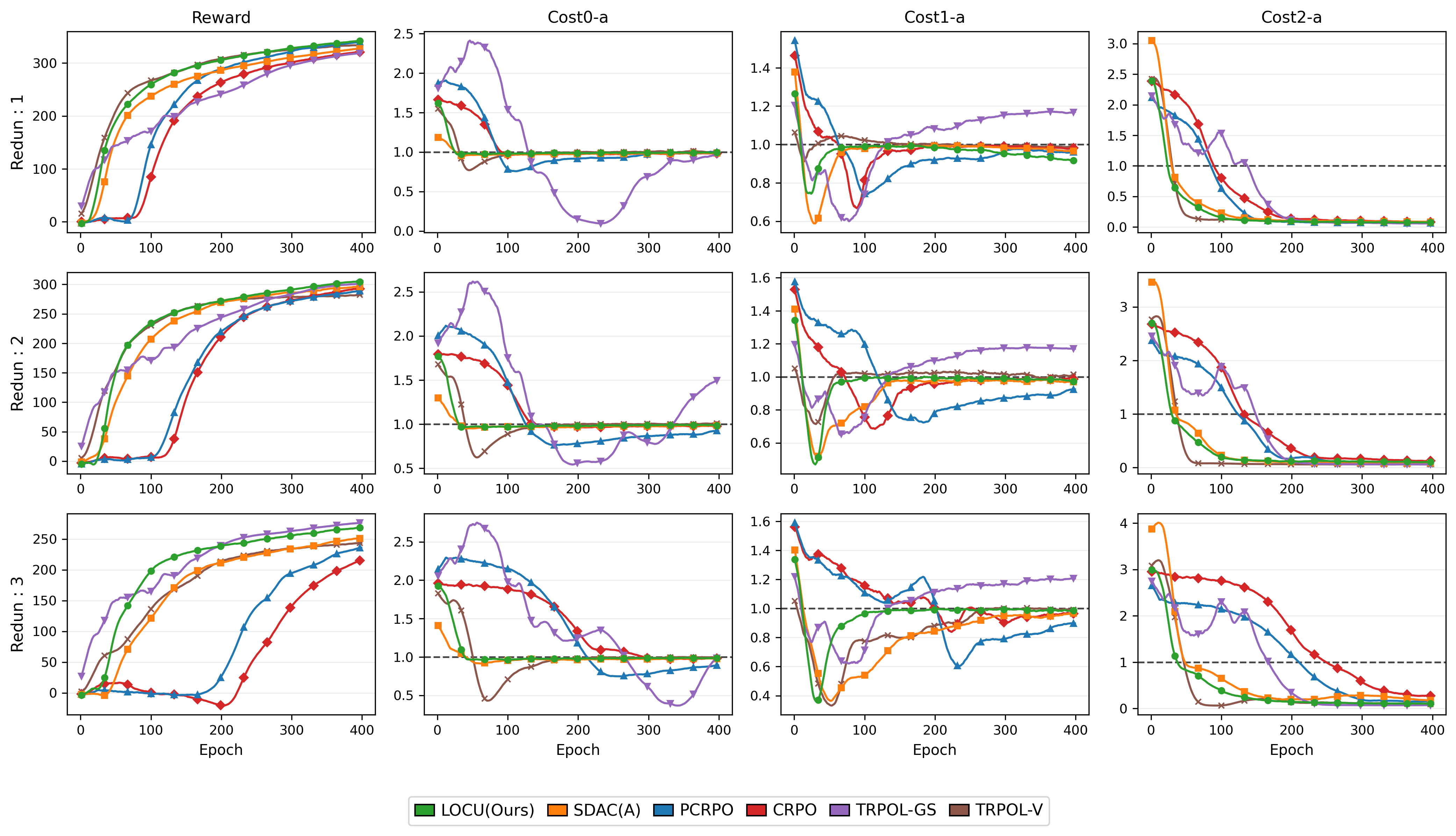}
  \caption{
    Learning curves on BC-v3 under redundancy variants redun : 1, redun : 2, and redun : 3.
    The "-a" denotes the tightest redundant constraint.
  }
  \label{fig:lc:bcv3-redun}
\end{figure}
\begin{figure}[H]
  \centering
  \includegraphics[width=\textwidth]{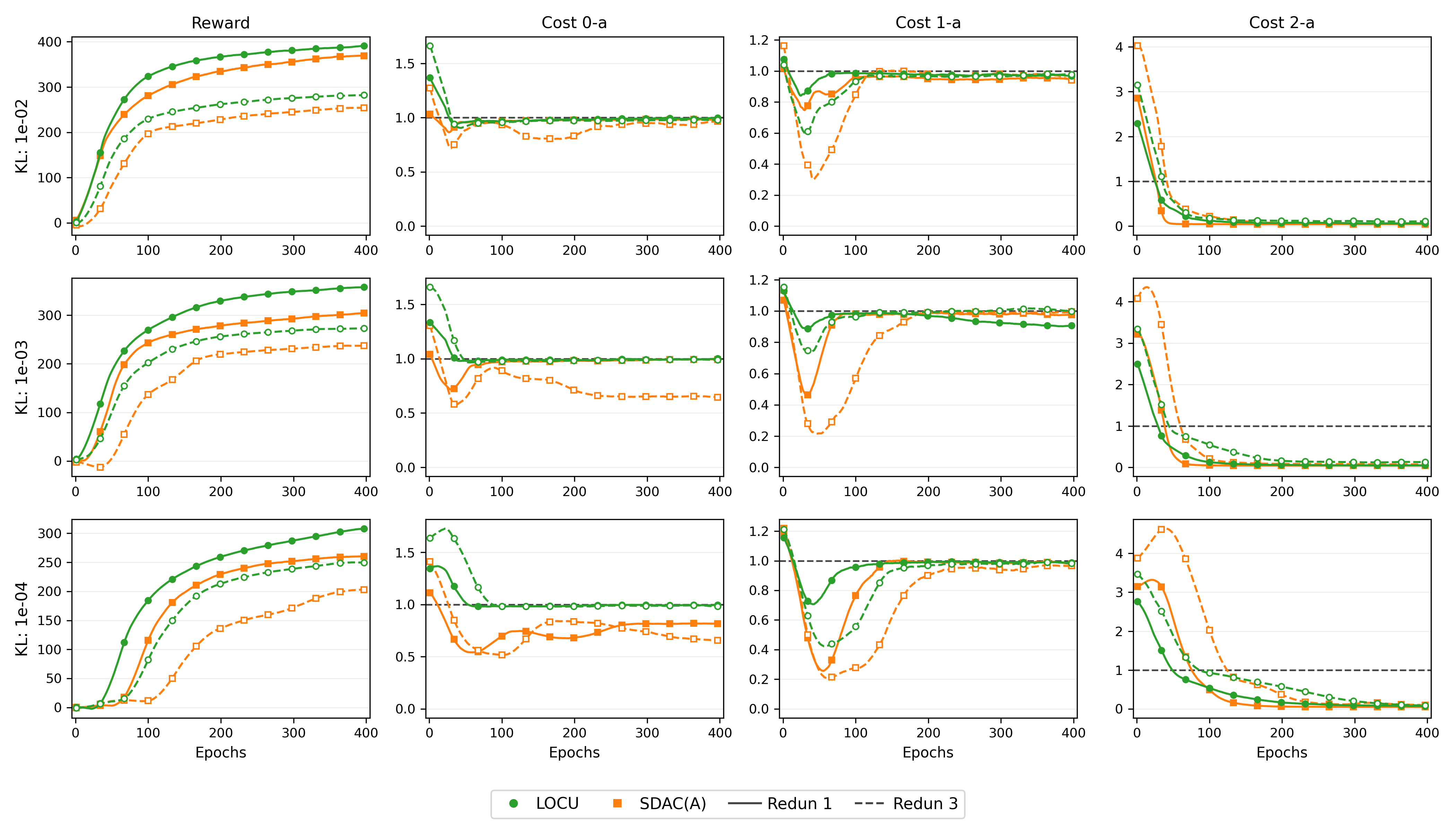}
  \caption{
    Learning curves of LOCU with different hyperparameters on BC-v3 under redundancy variants redun : 1 and redun : 3. 
  }
  \label{fig:lc:bcv3-redun}
\end{figure}

\begin{figure}[H]
  \centering
  \includegraphics[width=\textwidth]{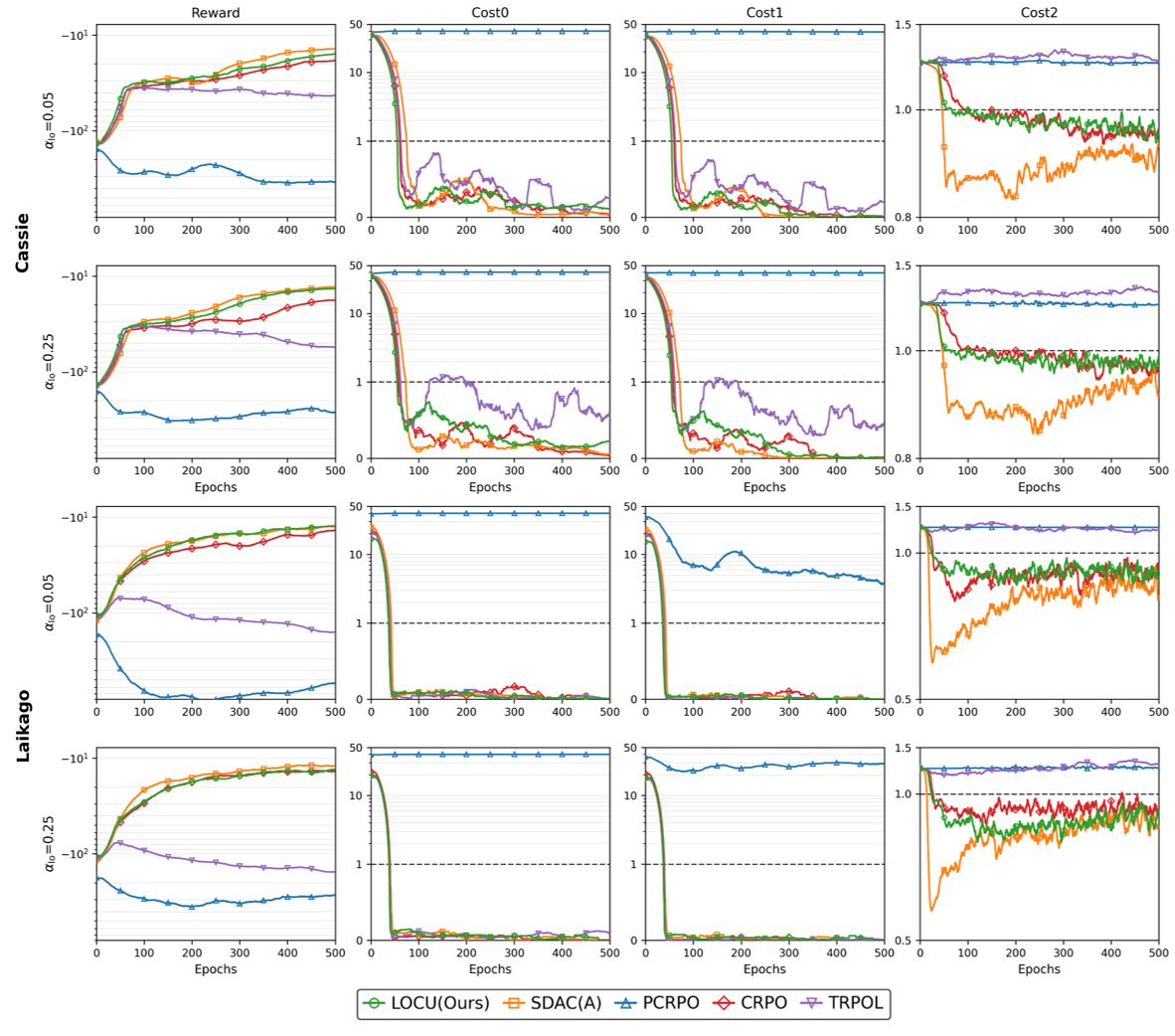}
  \caption{
    Learning curves on the locomotion tasks with Cassie and Laikago under three simultaneous constraints.
  }
  \label{fig:lc:locomotion}
\end{figure}
\begin{figure}[H]
  \centering
  \includegraphics[width=\textwidth]{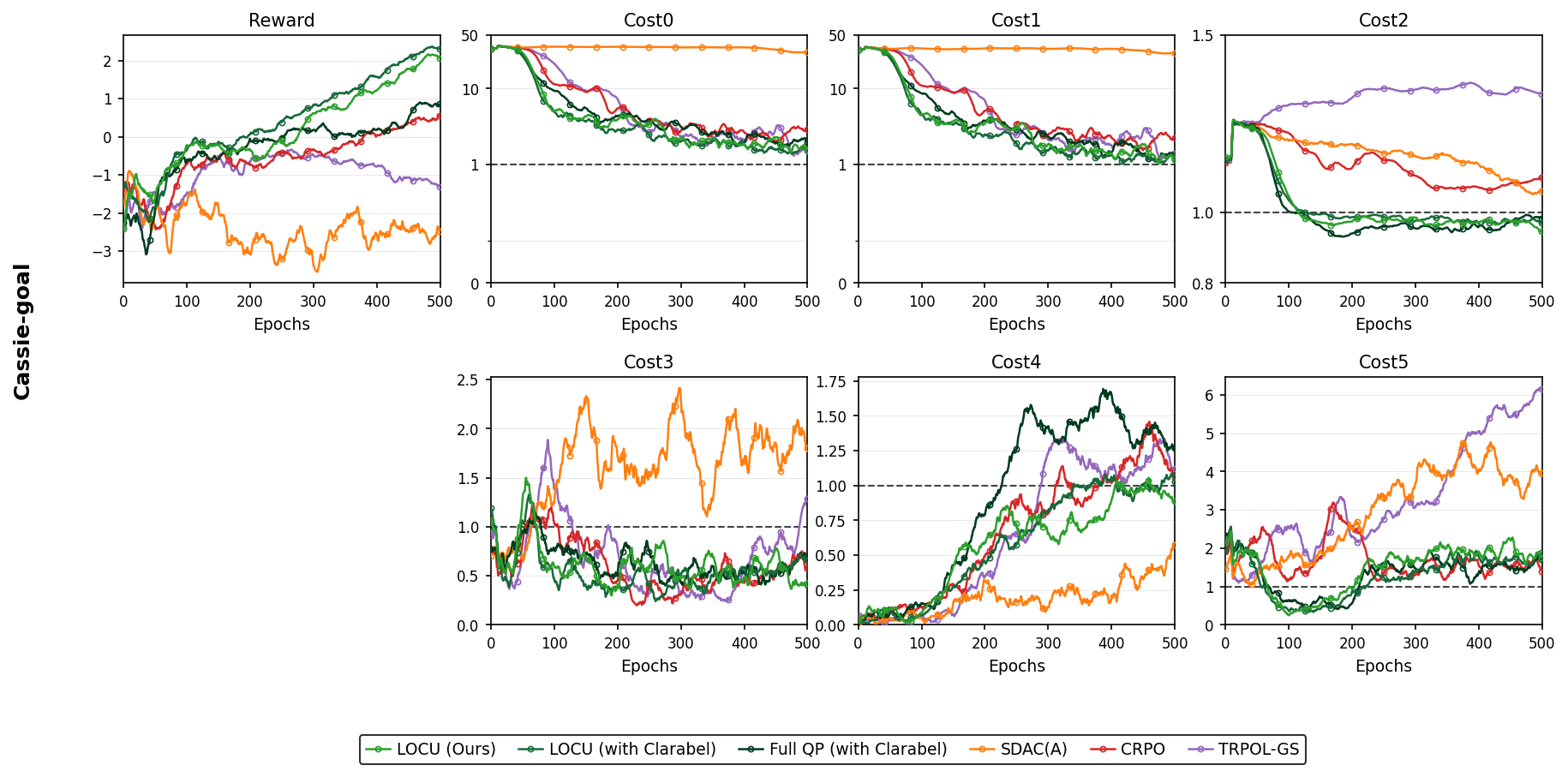}
  \caption{
    Learning curves on the goal-object task with Cassie under additional three simultaneous constraints.
  }
  \label{fig:lc:cassie_goal}
\end{figure}

\begin{figure}[H]
  \centering
  \includegraphics[width=\textwidth]{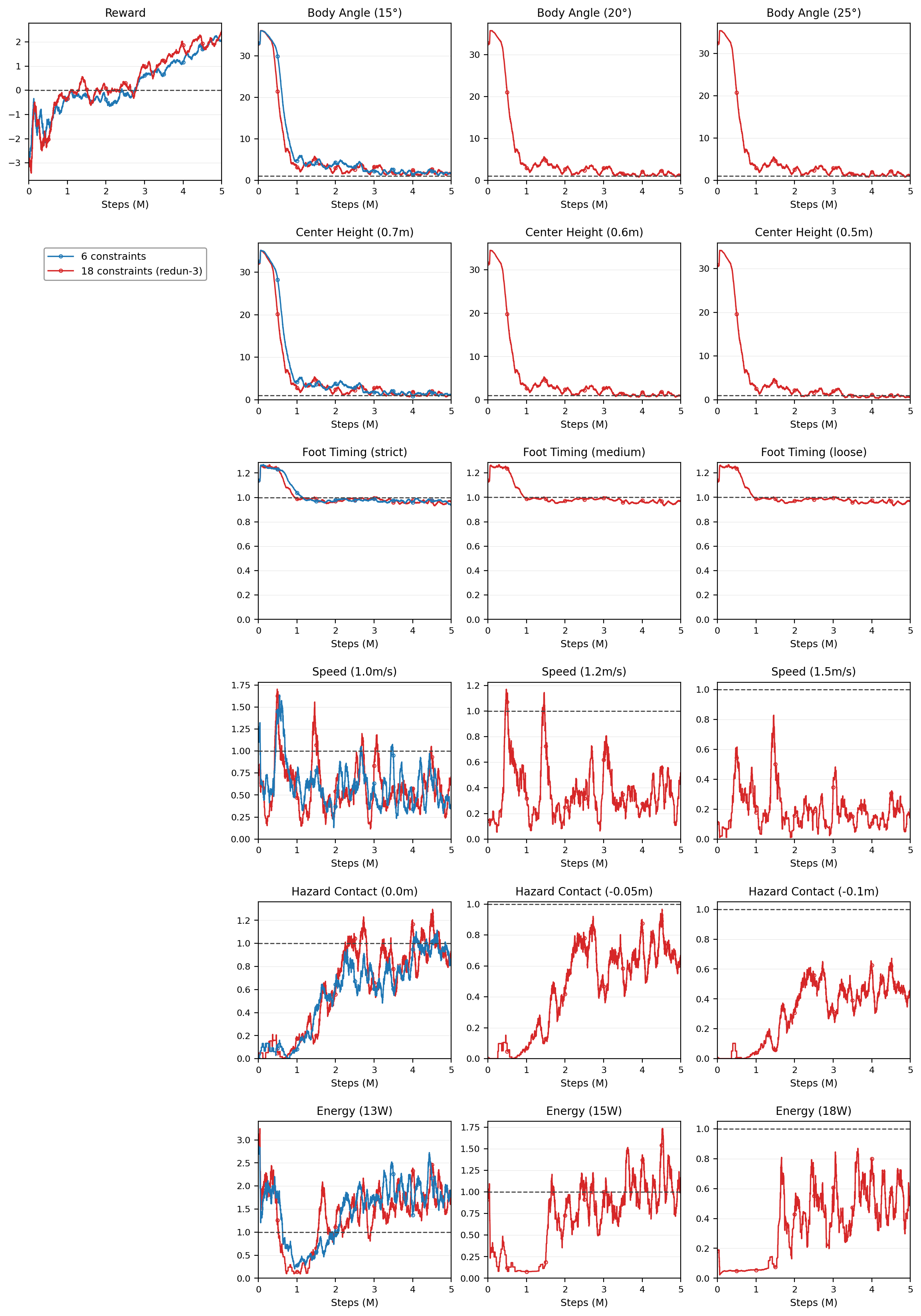}
  \caption{Learning curves for the goal-objective task with Cassie and the redundant 18-constraint version.}
  \label{fig:robustness-redun-vs-base-learning}
\end{figure}

\begin{figure}[H]
  \centering
  \includegraphics[width=\textwidth]{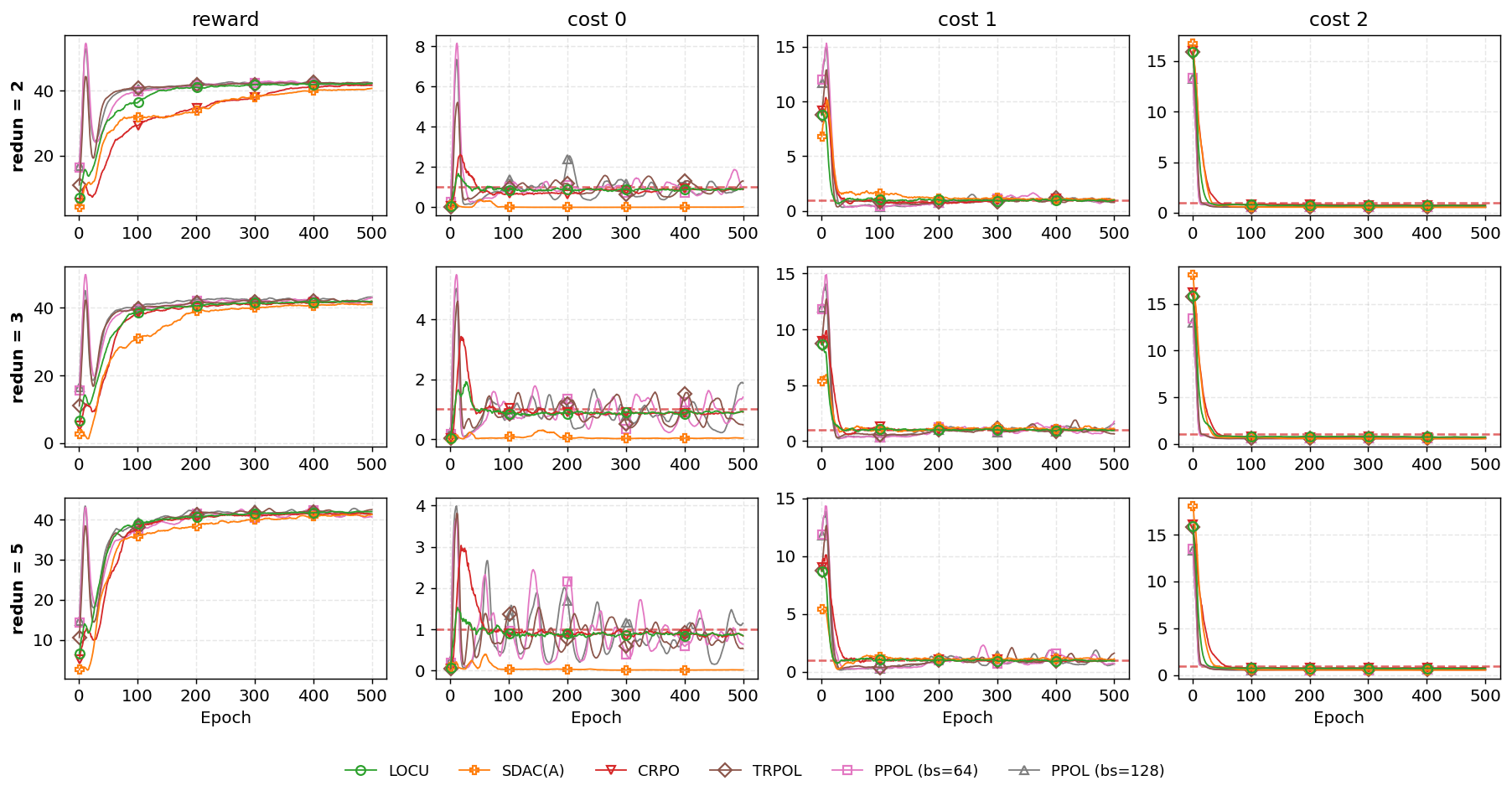}
  \caption{Learning curves for the mjx implementation of circle tasks.}
  \label{fig:mjx-circle-redun}
\end{figure}

\begin{figure}[H]
  \centering
  \includegraphics[width=\textwidth]{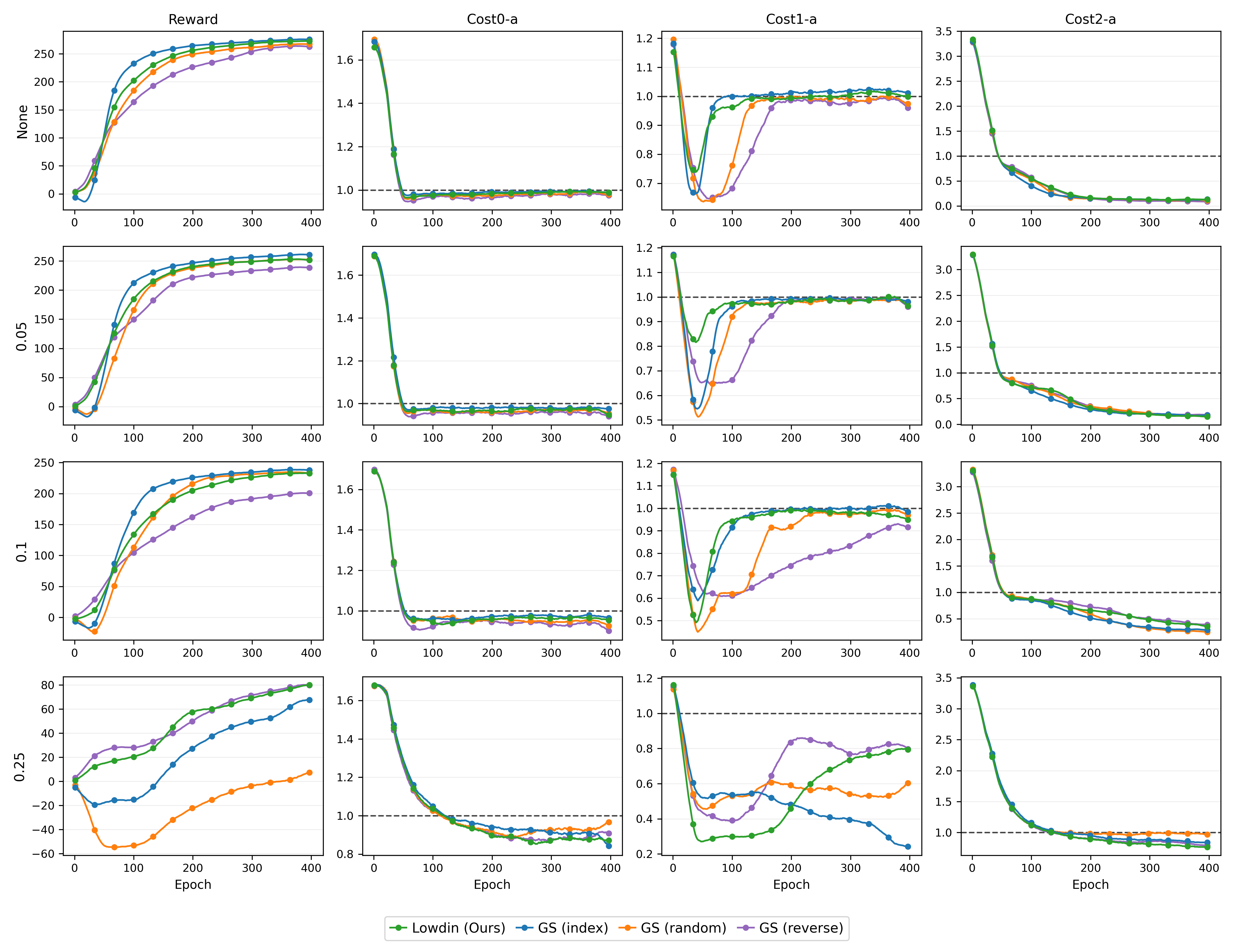}
  \caption{Learning curves for the orthogonalization comparison.}
  \label{fig:mjx-circle-redun}
\end{figure}

\clearpage

\end{document}